\documentclass[letterpaper]{article}

\usepackage{natbib}
\usepackage{uai2019}
\usepackage[margin=1in]{geometry}

\usepackage{times}

\usepackage{booktabs}
\usepackage{multirow}

\usepackage{amsfonts} % for \mathbb
\usepackage{bm} % for bm
\usepackage{mathtools} % for coloneqq
\usepackage[vlined,linesnumbered,ruled,resetcount]{algorithm2e} % for algorithm

\usepackage{amsthm}

\usepackage{subcaption}

\usepackage{tabularx}
\usepackage{stfloats} % table page bottom

\usepackage{color} % for highlight text

\usepackage{hyperref}

\newtheorem{post}{Postulate}
\newtheorem{thrm}{Theorem}
\newtheorem{assp}{Assumption}

\theoremstyle{definition}
\newtheorem{dfnt}{Definition}

\usepackage{amsmath}
\DeclareMathOperator*{\argmax}{arg\,max}
\DeclareMathOperator{\tr}{tr}
\newcommand\boldred[1]{\textit{\textbf{#1}}}
\title{Domain Generalization via Multidomain Discriminant Analysis}

\author{ 
	Shoubo Hu\thanks{~~Correspondence: Shoubo Hu $<$\href{mailto:sbhu@cse.cuhk.edu.hk}{sbhu@cse.cuhk.edu.hk}$>$ }~~, Kun Zhang$^{\dagger}$, Zhitang Chen{\small $^{\ddagger}$}, Laiwan Chan{$^{\ast}$} \\
	$^{\ast}$Department of Computer Science and Engineering, The Chinese University of Hong Kong \\ 
	$^{\dagger}$Department of Philosophy, Carnegie Mellon University~
	$^{\ddagger}$Huawei Noah's Ark Lab \\
}

\begin{document}

\maketitle

\begin{abstract}
Domain generalization (DG) aims to incorporate knowledge from multiple source domains into a single model that could generalize well on unseen target domains. This problem is ubiquitous in practice since the distributions of the target data may rarely be identical to those of the source data. In this paper, we propose Multidomain Discriminant Analysis (MDA) to address DG of classification tasks in general situations. MDA learns a domain-invariant feature transformation that aims to achieve appealing properties, including a minimal divergence among domains within each class, a maximal separability among classes, and overall maximal compactness of all classes. Furthermore, we provide the bounds on excess risk and generalization error by learning theory analysis. Comprehensive experiments on synthetic and real benchmark datasets demonstrate the effectiveness of MDA.
\end{abstract}

\section{INTRODUCTION}
Supervised learning has made considerable progress in tasks such as image classification \citep{krizhevsky2012imagenet}, object recognition \citep{DBLP:journals/corr/SimonyanZ14a}, and object detection \citep{girshick2014rich}. In standard setting, a model is trained on training or source data and then applied on test or target data for prediction, where one implicitly assumes that both source and target data follow the same distribution. However, this assumption is very likely to be violated in real problems. For example, in image classification, images from different sources may be collected under different conditions (e.g., viewpoints, illumination, backgrounds, etc), which makes classifiers trained on one domain perform poorly on instances of \textit{previously unseen} domains. These problems of transferring knowledge to unseen domains are known as domain generalization (DG; \citep{blanchard2011generalizing}). Note that no data from target domain is available in DG, whereas unlabeled data from the target domain is usually available in domain adaptation, for which a much richer literature exists (e.g., see \cite{patel2015visual}).

Denote the space of feature $X$ by $\mathcal{X}$ and the space of label $Y$ by $\mathcal{Y}$. A domain is defined as a joint distribution $\mathbb{P}(X, Y)$ over $\mathcal{X} \times \mathcal{Y}$. In DG of classification tasks, one is given $m$ sample sets, which were generated from $m$ source domains, for model training. The goal is to incorporate the knowledge from source domains to improve the model generalization ability on an \textit{unseen} target domain. An example of DG is shown in Figure \ref{fig:illustration}. % For other lines of research in this area, we refer to readers to, e.g. \cite{xu2014exploiting, ghifary2015domain}.%Most existing DG methods resort to learn a feature transformation $h(X)$ such that it has fixed (conditional) distribution across all domains \cite{muandet2013domain,xu2014exploiting,ghifary2015domain,erfani2016robust,ghifary2017scatter, AAAI1816595}.

\begin{figure}[htb]
    \centering
    \includegraphics[width=0.45 \textwidth]{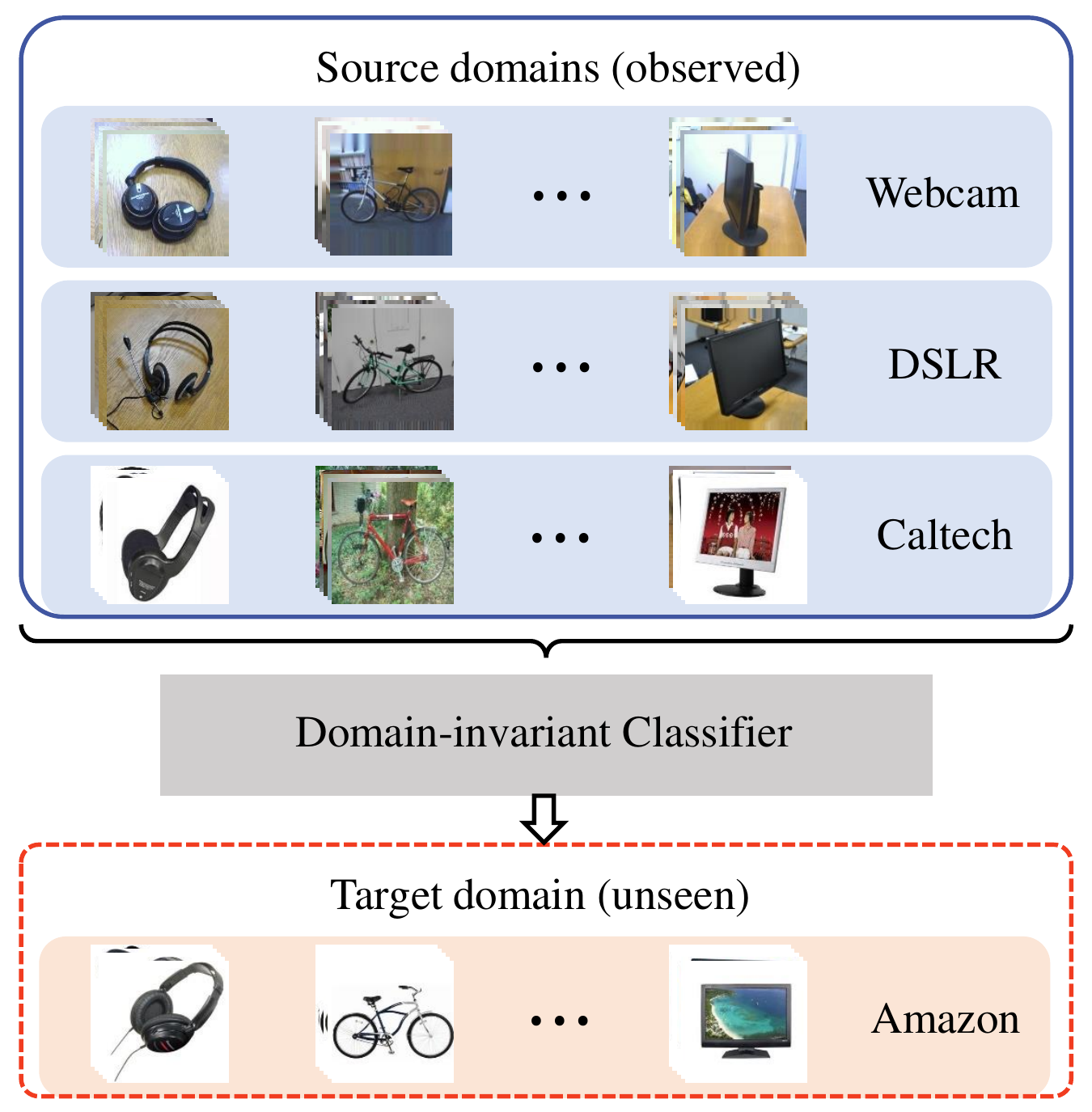}
    \caption{Illustration of DG on Office+Caltech Dataset. One is given source domains: Webcam, DSLR, Caltech, and aims to train a classifier generalizes well on target domain Amazon, which is unavailable in training.}
    \label{fig:illustration}
\end{figure}

Although various techniques such as kernel methods \citep{muandet2013domain, ghifary2017scatter, AAAI1816595}, support vector machine (SVM) \citep{khosla2012undoing, xu2014exploiting}, and deep neural network \citep{ghifary2015domain, Motiian_2017_ICCV, 8237853, lihl2018domain, Li_2018_ECCV}, have been adopted to solve DG problem, the general idea, which is learning a domain-invariant representation with stable (conditional) distribution in all domains, is shared in most works. Among previous works, kernel-based methods interpret the domain-invariant representation as a feature transformation from the original input space to a transformed space $\mathbb{R}^{q}$, in which the (conditional) distribution shift across domains is minimized.

Unlike previous kernel-based methods, which assume that $\mathbb{P}(Y|X)$ keeps stable and only $\mathbb{P}(X)$ changes across domains (i.e., the covariate shift situation \citep{shimodaira2000improving}), the problem of DG or domain adaptation has also been investigated from a causal perspective \citep{zhang2015multi}. In particular, \cite{zhang2013domain} pointed out that for many learning problems, especially for classification tasks, $Y$ is usually the cause of $X$, and proposed the setting of target shift ($\mathbb{P}(Y)$ changes while $\mathbb{P}(X|Y)$ stays the same across domains), conditional shift ($\mathbb{P}(Y)$ stays the same and $\mathbb{P}(X|Y)$ changes across domains), and their combination accordingly. \cite{pmlr-v48-gong16} proposed to do domain adaptation with conditionally invariant components of $X$, i.e., the transformations of $X$ that have invariant conditional distribution given $Y$ across domains. \cite{AAAI1816595} then used this idea for DG, under the assumption of conditional shift. 
%Specifically, \cite{AAAI1816595} factorized $\mathbb{P}(X,Y)$ in the causal direction as $\mathbb{P}(X|Y)$ and $\mathbb{P}(Y)$, and assume that $\mathbb{P}(X|Y)$ changes across domains and $\mathbb{P}(Y)$ keeps stable (conditional shift situation \citep{zhang2013domain}) instead of making assumptions on the factorization in anti-causal direction, $\mathbb{P}(Y|X)$ and $\mathbb{P}(X)$. 
Their assumptions stem from the following postulate of causal independence \citep{janzing2010causal,daniuvsis2010inferring}:
\begin{post}[Independence of cause and mechanism]\label{post:1}
If $Y$ causes $X$ ($Y \to X$), then the marginal distribution of the cause, $\mathbb{P}(Y)$, and the conditional distribution of the effect given cause, $\mathbb{P}(X|Y)$, are ``independent'' in the sense that $\mathbb{P}(X|Y)$ contains no information about $\mathbb{P}(Y)$.
\end{post}
According to postulate \ref{post:1}, $\mathbb{P}(X|Y)$ and $\mathbb{P}(Y)$ would behave independently across domains. However, this independence typically does not hold in the anti-causal direction \citep{ScholkopfJPSZMJ2012}, so $\mathbb{P}(Y|X)$ and $\mathbb{P}(X)$ tends to vary in a coupled manner across domains. Under assumptions that $\mathbb{P}(X|Y)$ changes while $\mathbb{P}(Y)$ keeps stable, generally speaking, both $\mathbb{P}(Y|X)$ and $\mathbb{P}(X)$ change across domains in the anti-causal direction, which is clearly different from the covariate shift situation.

In this paper, we further relax the causally motivated assumptions in \cite{AAAI1816595} and propose a novel DG method, which is applicable when both $\mathbb{P}(X|Y)$ and $\mathbb{P}(Y)$ change across domains. %As existing kernel-based methods, we first map features to a reproducing kernel Hilbert space $\mathcal{H}$, and then find a transformation from $\mathcal{H}$ to a subspace $\mathbb{R}^{q}$, in which the target domain is expected to be as similar as possible to source domains. % in which the distribution shift over domains within each class is minimized while maintaining the separability of different classes. %The basic idea is still about learning domain-invariant feature transformation $h(X)$. 
Our method focuses on separability between classes and does not enforce the transformed marginal distribution of features to be stable, which allows us to relax the assumption of stable $\mathbb{P}(Y)$. To improve the separability, a novel measure named average class discrepancy, which measures the class discriminative power of source domains, is proposed. Average class discrepancy and other three measures are unified in one objective for feature transformation learning to improve its generalization ability on the target domain.
As the second contribution, we derive the bound on excess risk and generalization error\footnote{\cite{blanchard2011generalizing} proved the generalization error bound of DG in general settings} for kernel-based domain-invariant feature transformation methods. To the best of our knowledge, this is one of the first works to give theoretical learning guarantees on excess risk of DG. %The generalization error bound of DG in general settings is proved in \cite{blanchard2011generalizing} and we derive it particularly for . 
Lastly, experimental results on synthetic and real datasets demonstrate the efficacy of our method in handling varying class prior distributions $\mathbb{P}(Y)$ and complex high-dimensional distributions, respectively.

This paper is organized as follows. Section 2 gives the background on kernel mean embedding. Section 3 introduces our method in detail. %We define the DG problem, propose key measures, formulate the DG objective, and solve for the feature transformation in section 3. 
Section 4 gives the bounds on excess risk and generalization error for kernel-based methods. Section 5 gives experimental settings and analyzes the results. Section 6 concludes this work.

\section{PRELIMINARY ON KERNEL MEAN EMBEDDING}
Kernel mean embedding is the main technique to characterize probability distributions in this paper. Kernel mean embedding represents probability distributions as elements in a reproducing kernel Hilbert space (RKHS). More precisely, an RKHS $\mathcal{H}$ over domain $\mathcal{X}$ with a kernel $k$ is a Hilbert space of functions $f:\mathcal{X} \rightarrow \mathbb{R}$. Denoting its inner product by $\langle \cdot, \cdot \rangle_{\mathcal{H}}$, RKHS $\mathcal{H}$ fulfills the reproducing property $\langle f(\cdot), k(\bm{x},\cdot) \rangle_{\mathcal{H}} = f(\bm{x})$, where $\phi(\bm{x}) \coloneqq k(\bm{x},\cdot)$ is the canonical feature map of $\bm{x}$. The kernel mean embedding of a distribution $\mathbb{P}(X)$ is defined as \citep{smola2007hilbert,gretton2007kernel}:
\begin{equation}
	\mu_{X} \coloneqq \mathbb{E}_{X}[\phi(X)] = \int _{ \mathcal{X}  }^{  }{ \phi(\bm{x})d \mathbb{P} (\bm{x}) },
	\label{eq:ebd_def}
\end{equation}
where $\mathbb{E}_{X}[\phi(X)]$ is the expectation of $\phi(X)$ with respect to $\mathbb{P}(X)$. It was shown that $\mu_{X}$ is guaranteed to be an element in the RKHS if $\mathbb{E}_{X}[k(\bm{x},\bm{x})]<\infty$ is satisfied \citep{smola2007hilbert}. In practice, given a finite sample of size $n$, the kernel mean embedding of $\mathbb{P}(X)$ is empirically estimated as $\hat{\mu}_{X} = \frac{1}{n} \sum^{n}_{i=1} \phi(\bm{x}_{i})$, where $\lbrace \bm{x}_{i} \rbrace^{n}_{i=1}$ are independently drawn from $\mathbb{P}(X)$. When $k$ is a characteristic kernel \citep{973}, %(e.g., the radial basis function (RBF) kernel $k(\bm{x}, \bm{x}^{\prime}) = \exp\left( - \frac{\Vert \bm{x} - \bm{x}^{\prime} \Vert^{2}}{2\sigma^{2}} \right) $, where $\sigma$ is a hyper-parameter called kernel width), 
$\mu_{X}$ captures all information about $\mathbb{P}(X)$ \citep{sriperumbudur2008injective}, which means that $\Vert \mu_{X} - \mu_{X^{\prime}} \Vert_{\mathcal{H}} = 0$ if and only if $\mathbb{P}(X)$ and $\mathbb{P}(X^{\prime})$ are the same distribution.

% Using the kernel mean embedding to characterize distributions was preferred in this work for the following reasons: 1) if a \textit{characteristic kernel}, such as the radial basis function (RBF) kernel, is used, the mapping from a distribution to the corresponding mean embedding is injective, i.e., the mean embedding captures all information of the distribution \cite{fukumizu2004dimensionality,5122}; 2) inheriting from kernel methods, the kernel mean embedding provides a powerful, compact, and nonparametric tool to represent probability distributions of high dimensional data with unstructured or structured attributes \cite{muandet2017kernel}.

\section{MULTIDOMAIN DISCRIMINANT ANALYSIS}
\subsection{PROBLEM DEFINITION}
DG of classification tasks is studied in this paper. Let $\mathcal{X}$ be the feature space, $\mathcal{Y}$ be the space of class labels, and $c$ be the number of classes. A domain is defined to be a joint distribution $\mathbb{P}(X,Y)$ on $\mathcal{X} \times \mathcal{Y}$.
%and let $\mathfrak{P}_{\mathcal{X} \times \mathcal{Y}}$ denote the set of all domains. Let $\mathfrak{P}_{\mathcal{Y}}$ and $\mathfrak{P}_{\mathcal{X} | \mathcal{Y}}$ denote the set of probability distributions $\mathbb{P}_{Y}$ on $\mathcal{Y}$ and $\mathbb{P}_{X|Y}$ on $\mathcal{X}$ given $\mathcal{Y}$, respectively.
Let $\mathfrak{P}_{\mathcal{X} \times \mathcal{Y}}$ denote the set of domains $\mathbb{P}(X,Y)$ and $\mathfrak{P}_\mathcal{X}$ denote the set of distributions $\mathbb{P}(X)$ on $\mathcal{X}$. We assume that there is an underlying finite-variance unimodal distribution over $\mathfrak{P}_{\mathcal{X} \times \mathcal{Y}}$. In practice, domains are not observed directly, but given in the form of finite sample sets.

\begin{assp}[Data-generating process]\label{assp:data}
Each sample set is assumed to be generated in two separate steps: 1) a domain $\mathbb{P}^{s}(X,Y)$ is sampled from $\mathfrak{P}_{\mathcal{X} \times \mathcal{Y}}$, where $s$ is the domain index; 2) $n^{s}$ independent and identically distributed (i.i.d.) instances are then drawn from $\mathbb{P}^{s}(X,Y)$.
\end{assp}

Suppose there are $m$ domains sampled from $\mathfrak{P}_{\mathcal{X} \times \mathcal{Y}}$, the set of $m$ observed sample sets is denoted by $\mathcal{D}=\lbrace \mathcal{D}^{s} \rbrace^{m}_{s=1}$, where each $\mathcal{D}^{s} = \lbrace ( \bm{x}^{s}_{i}, y^{s}_{i}) \rbrace^{n^{s}}_{i=1}$ consists of $n^{s}$ i.i.d. instances from $\mathbb{P}^{s}(X,Y)$. Since in general $\mathbb{P}^{s}(X,Y) \neq \mathbb{P}^{s^\prime}(X,Y)$, instances in $\mathcal{D}$ are not i.i.d.

In DG of classification tasks, one aims to incorporate the knowledge in $\mathcal{D}$ into a model which could generalize well on a previously unseen target domain. In this work, features $X$ are first mapped to an RKHS $\mathcal{H}$.
% Formally, one is given $m$ sample sets $D^{s}$, where $s = 1, 2, \dots, m$, associated with $m$ related but not identical source domains $\mathbb{P}^{1}, \mathbb{P}^{2}, \dots, \mathbb{P}^{m}$. To tackle the distribution shift in DG problem, 
Then we resort to learning a transformation from the RKHS $\mathcal{H}$ to a $q$-dimensional transformed space $\mathbb{R}^{q}$, in which instances of the same class are close and instances of different classes are distant from each other. 1-nearest neighbor is adopted to conduct classification in $\mathbb{R}^{q}$.
%conditional distributions of features given class labels from different domains $\mathbb{P}^{s}(h(X)|Y)$ are close to each other within each class and distinguishable between different classes.

\begin{table}[htbp] 
	\centering
	\caption{Notations used in the paper}
	\vspace{-3mm}
	\label{tb:notation}
	\resizebox{\linewidth}{!}{%
	\begin{tabular}{cccc}
		\toprule
		Notation & Description & Notation & Description \\
		\midrule
		$X$, $Y$ & feature/label variable & $\bm{x}$, $y$ & feature/label instance \\
		$m$, $c$ & \# domains/classes & $s$, $j$ & domain/class index \\
		$\mathfrak{P}_{\mathcal{X} \times \mathcal{Y}}$, $\mathfrak{P}_{\mathcal{X}}$  & the set of $\mathbb{P}(X,Y)$ / $\mathbb{P}(X)$ & $\mathcal{D}^{s}$ & sample set of domain $s$ \\
		$\mathbb{P}^{s}_{j}$ & class-conditional distribution & $\mu^{s}_{j}$ & kernel mean embedding of $\mathbb{P}^{s}_{j}$ \\
		$u_{j}$ & mean representation of class $j$ & $\bar{u}$ & mean representation of $\mathcal{D}$ \\
		$k$ & kernel & $\mathcal{H}_{k}$ & RKHS associated with $k$ \\
		\bottomrule
	\end{tabular} }
\end{table}

\subsection{REGULARIZATION MEASURES}

\subsubsection{Average Domain Discrepancy}
To achieve the goal that instances of the same class are close to each other, we first consider minimizing the discrepancy of the class-conditional distributions, $\mathbb{P}^{s}(X|Y=j)$, within each class over all source domains. 

For ease of notation, the class-conditional distribution of class $j$ in domain $s$, $\mathbb{P}^{s}(X|Y=j)$, is denoted by $\mathbb{P}^{s}_{j}$. Denoting the kernel mean embedding \eqref{eq:ebd_def} of $\mathbb{P}^{s}_{j}$ by $\mu^{s}_{j}$, the average domain discrepancy is defined below.

\begin{dfnt}[Average domain discrepancy]
Given the set of all class-conditional distributions $\mathcal{P} = \lbrace \mathbb{P}^{s}_{j} \rbrace$ for $s\in \lbrace 1, \dots, m \rbrace$ and $j \in \lbrace 1,\dots,c \rbrace$, the average domain discrepancy, $\Psi^{\textit{add}}(\mathcal{P})$, is defined as
\begin{align}
\Psi^{\textit{add}}(\mathcal{P}) \coloneqq \frac{1}{c\binom{m}{2}} \sum^{c}_{j=1} \sum_{1 \le s < s^{\prime} \le m} \Vert \mu^{s}_{j} - \mu^{s^{\prime}}_{j} \Vert^{2}_{\mathcal{H}},
\label{eq:add_def}
\end{align}
where $\binom{m}{2}$ is the number of 2-combinations from a set of $m$ elements, $\Vert \cdot \Vert^{2}_{\mathcal{H}}$ denotes the squared norm in RKHS $\mathcal{H}$, and $\Vert \mu^{s}_{j} - \mu^{s^{\prime}}_{j} \Vert_{\mathcal{H}}$ is thus the Maximum Mean Discrepancy (MMD; \citep{gretton2007kernel}) between $\mathbb{P}^{s}_{j}$ and $\mathbb{P}^{s^{\prime}}_{j}$.
\end{dfnt}

The following theorem shows that $\Psi^{\textit{add}}(\mathcal{P})$ is suitable for measuring the discrepancy between class-conditional distributions of the same class from multiple domains.

\begin{thrm}
Let $\mathcal{P}$ denote the set of all class-conditional distributions. If $k$ is a characteristic kernel \citep{973}, $\Psi^{\text{add}}(\mathcal{P})=0$ if and only if $\mathbb{P}^{1}_{j} = \mathbb{P}^{2}_{j}=\dots=\mathbb{P}^{m}_{j}$, for $j=1,\dots, c$.
\end{thrm}

\begin{proof}
Since $k$ is a characteristic kernel, $\Vert \mu_{\mathbb{P}} - \mu_{\mathbb{Q}}\Vert_{\mathcal{H}}$ is a metric and attains 0 if and only if $\mathbb{P} = \mathbb{Q}$ for any distributions $\mathbb{P}$ and $\mathbb{Q}$ \citep{sriperumbudur2008injective}. Therefore, $\Vert \mu^{s}_{j} - \mu^{s^{\prime}}_{j}\Vert_{\mathcal{H}} = 0$ if and only if $\mathbb{P}^{s}_{j} = \mathbb{P}^{s^{\prime}}_{j}$ for all $s$ and $s^{\prime}$ given $j$, which means $\mathbb{P}^{1}_{j} = \mathbb{P}^{2}_{j}=\dots=\mathbb{P}^{m}_{j}$ within each class $j$. Conversely, if $\mathbb{P}^{1}_{j} = \mathbb{P}^{2}_{j}=\dots=\mathbb{P}^{m}_{j}$ for $j=1,\dots, c$, then each term $\Vert \mu^{s}_{j} - \mu^{s^{\prime}}_{j} \Vert=0$ and $\Psi^{\text{add}}(\mathcal{P})$ is thus 0. 
\end{proof}

\subsubsection{Average Class Discrepancy}
Minimizing average domain discrepancy $\Psi^{\textit{add}}$ \eqref{eq:add_def} would make the means of class-conditional distributions of the same class close in $\mathcal{H}$. However, it is possible that the means of class-conditional distributions of different classes are also close, which is a major source of performance reduction of existing kernel-based DG methods. To this end, average class discrepancy is proposed. 

\begin{dfnt}[Average class discrepancy]
Let $\mathcal{P}$ denote the set of all class-conditional distributions. The average class discrepancy is defined as
\begin{align}
\Psi^{\textit{acd}}(\mathcal{P}) \coloneqq \frac{1}{\binom{c}{2}} \sum_{1 \le j < j^{\prime} \le c} \Vert u_{j} - u_{j^{\prime}} \Vert^{2}_{\mathcal{H}},
\label{eq:acd_def}
\end{align}
where $u_{j} = \sum^{m}_{s=1} \mathbb{P}(S=s|Y=j) \mu^{s}_{j}$ is the mean representation of class $j$ in RKHS $\mathcal{H}$. 
\end{dfnt}

%We denote the MMD between the mean distribution on features of class $j$, $\mathbb{P}_{X|j}$, and class $j^{\prime}$, $\mathbb{P}_{X|j^{\prime}}$, i.e. $\Vert u_{j} - u_{j^{\prime}} \Vert^{2}_{\mathcal{H}}$ in \eqref{eq:def_diff}, by $\text{MMD}[ \mathbb{P}_{X|j}, \mathbb{P}_{X|j^{\prime}} ] $.

It was shown in \cite{sriperumbudur2010hilbert} that the MMD between two distributions $\mathbb{P}$ and $\mathbb{Q}$, $\text{MMD}[ \mathbb{P}, \mathbb{Q} ] \le \sqrt{C} W_{1} ( \mathbb{P}, \mathbb{Q} )$ for some constant $C$ satisfying $\sup_{\bm{x} \in \mathcal{X}} k(\bm{x}, \bm{x}) \le C < \infty$, where $W_{1} ( \mathbb{P}, \mathbb{Q} )$ denotes the first Wasserstein distance \citep{del1999tests} between distributions $\mathbb{P}$ and $\mathbb{Q}$. In other words, if $\mathbb{P}$ and $\mathbb{Q}$ are distant in MMD metric, they are also distant in the first Wasserstein distance. Therefore, distributions of different classes tend to be distinguishable by maximizing average class discrepancy, $\Psi^{\textit{acd}}(\mathcal{P})$. 

\subsubsection{Incorporating Instance-level Information}
Both average domain discrepancy $\Psi^{\textit{add}}(\mathcal{P})$ \eqref{eq:add_def} and average class discrepancy $\Psi^{\textit{acd}}(\mathcal{P})$ \eqref{eq:acd_def} are defined based on the kernel mean embedding of class-conditional distributions $\mathbb{P}^{s}_{j}$. By simultaneously minimizing $\Psi^{\textit{add}}$ \eqref{eq:add_def} and maximizing $\Psi^{\textit{acd}}$ \eqref{eq:acd_def}, one would make class-conditional kernel mean embeddings within each class close and the those of different classes distant in $\mathcal{H}$. However, certain subtle information, such as the compactness of the distribution, is not captured in $\Psi^{\textit{add}}$ and $\Psi^{\textit{acd}}$. As a result, although all mean embeddings satisfy the desired condition, there may still be a high chance of misclassification for some instances. To incorporate such information conveyed in each instance, we propose two extra measures based on kernel Fisher discriminant analysis \citep{mika1999fisher}. The first is multidomain between-class scatter.

\begin{dfnt}[Multidomain between-class scatter] Let $\mathcal{D}$ denote the set of $n$ instances from $m$ domains, each of which consists of $c$ classes. The multidomain between-class scatter is
\begin{align}
    \Psi^{\textit{mbs}}(\mathcal{D}) \coloneqq \frac{1}{n} \sum^{c}_{j=1} n_{j} \Vert u_{j} - \bar{ u } \Vert^{2}_{\mathcal{H}},
    \label{eq:mbs_def}
\end{align}
where $n_{j}$ is the total number of instances in class $j$, and $\bar{ u } = \sum^{c}_{j=1} P(Y=j) u_{j}$ is the the mean representation of the entire set $\mathcal{D}$ in $\mathcal{H}$.
\end{dfnt}

Both $\Psi^{\textit{mbs}}(\mathcal{D})$ and $\Psi^{\textit{acd}}(\mathcal{P})$ measure the discrepancy between the distributions of different classes. The difference stems from the weight $n_{j}$ in $\Psi^{\textit{mbs}}(\mathcal{D})$ \eqref{eq:mbs_def}. By adding $n_{j}$, each term in $\Psi^{\textit{mbs}}(\mathcal{D})$ is equivalent to pooling all instances of the same class together and summing up their distance to $\bar{u}$. In other words, $\Psi^{\textit{mbs}}(\mathcal{D})$ corresponds to a simple pooling scheme. Note that when the proportion of instances of each class is the same across all domains (i.e., $n^{s}_{j} / n^{s} = n^{s'}_{j} / n^{s'}, \forall s, s'$ for $j = 1, \dots, c$, where $n^{s}_{j}$ is the number of instances of class $j$ in domain $s$), $\Psi^{\textit{mbs}}(\mathcal{D})$ is consistent with the between-class scatter in \cite{mika1999fisher}.

Multidomain within-class scatter, as a straightforward counterpart of multidomain between-class scatter \eqref{eq:mbs_def}, is defined as follows.
\begin{dfnt}[Multidomain within-class scatter]
    Let $\mathcal{D}$ denote the set of $n$ instances from $m$ domains, each of which consists of $c$ classes. The multidomain within-class scatter is
    \begin{align}
        \Psi^{\textit{mws}}(\mathcal{D}) \coloneqq \frac{1}{n} \sum^{c}_{j=1} \sum^{m}_{s=1} \sum^{n^{s}_{j}}_{i=1} \Vert \phi(\bm{x}^{s}_{i \in j}) - u_{j} \Vert^{2}_{\mathcal{H}},
        \label{eq:mws_def}
	\end{align}
	where $\bm{x}^{s}_{i \in j}$ denotes the feature vector of $i$th instance of class $j$ in domain $s$.
\end{dfnt}

The definition above indicates that multidomain within-class scatter measures the sum of the distance between the canonical feature map of each instance and the mean representation in RKHS $\mathcal{H}$ of the class it belongs to. It differs from average domain discrepancy in that the information of every instance is considered in multidomain within-class scatter. As a result, by minimizing $\Psi^{\textit{mws}}(\mathcal{D})$, one increases the overall compactness of the distributions across classes. Similar to $\Psi^{\textit{mbs}}(\mathcal{D})$, when the proportion of instances of each class is the same across all domains (i.e., $n^{s}_{j} / n^{s} = n^{s^{\prime}}_{j} / n^{s'}, \forall s, s'$ for $j = 1, \dots, c$), $\Psi^{\textit{mws}}(\mathcal{D})$ is consistent with the within-class scatter in \cite{mika1999fisher}.

We note that each of the measures has its unique contribution and that ignoring any of them may lead to sub-optimal solutions, as demonstrated by the empirical results and illustrated in Appendix A.

\subsection{FEATURE TRANSFORMATION}

Our method resorts to finding a suitable transformation from RKHS $\mathcal{H}$ to a $q$-dimensional transformed space $\mathbb{R}^{q}$, i.e., $\mathbf{W}: \mathcal{H} \mapsto \mathbb{R}^{q}$. We elaborate how the proposed measures are transformed to $\mathbb{R}^{q}$ in this section.

According to the property of norm in RKHS, $\Psi^{\textit{add}}(\mathcal{P})$ can be equivalently computed as
\begin{align}
% & \Psi^{\textit{add}}(\mathcal{P}) \nonumber \\
%= & \frac{1}{C} \sum^{C}_{j=1} \sum_{1 \le s < s^{\prime} \le m} \tr \left( (\mu^{s}_{j} - \mu^{s^{\prime}}_{j}) (\mu^{s}_{j} - \mu^{s^{\prime}}_{j})^{T} \right) \nonumber \\
\tr \left( \frac{1}{c\binom{m}{2}} \sum^{c}_{j=1} \sum_{1 \le s < s^{\prime} \le m} (\mu^{s}_{j} - \mu^{s^{\prime}}_{j}) (\mu^{s}_{j} - \mu^{s^{\prime}}_{j})^{T} \right), \label{eq:add_tr} 
\end{align}
where $\tr(\cdot)$ denotes the trace operator.

Let the data matrix $\mathbf{X}=\left[ \bm{x}_{1}, \dots, \bm{x}_{n} \right]^{T} \in \mathbb{R}^{n \times d}$, where $d$ is the dimension of input features $X$ and $n = \sum^{m}_{s=1}n^{s}$, and the feature matrix $\bm{\Phi}=\left[ \phi(\bm{x}_{1}), \dots, \phi(\bm{x}_{n}) \right]^{T}$, where $\phi: \mathbb{R}^{d} \mapsto \mathcal{H}$ denotes the canonical feature map. Then $\mathbf{W}$ can be expressed as a linear combination of all canonical feature maps in $\bm{\Phi}$ \citep{scholkopf1998nonlinear}, i.e., $\mathbf{W} = \bm{\Phi}^{T}\mathbf{B}$, where $\mathbf{B}$ is a matrix collecting coefficients of canonical feature maps. Then by applying the transformation $\mathbf{W}$, $\Psi^{\textit{add}}(\mathcal{P})$ in trace formulation \eqref{eq:add_tr} becomes\
\begin{small}
\begin{align}
& \Psi^{\textit{add}}_{\mathbf{B}}
%= & \tr \left( \frac{1}{C} \sum^{C}_{j=1} \sum_{1 \le s < s^{\prime} \le m} \mathbf{B}^{T} \bm{\Phi} (\mu^{s}_{j} - \mu^{s^{\prime}}_{j} ) (\mu^{s}_{j} - \mu^{s^{\prime}}_{j} )^{T} \bm{\Phi}^{T}\mathbf{B} \right) \nonumber \\
= \tr \left( \mathbf{B}^{T} \mathbf{GB} \right),
\label{eq:add_B}
\end{align}
\end{small}
where 
\begin{small}
\begin{align}
\mathbf{G} = \frac{1}{c\binom{m}{2}} \sum^{c}_{j=1} \sum_{1 \le s < s^{\prime} \le m} \bm{\Phi} (\mu^{s}_{j} - \mu^{s^{\prime}}_{j} ) (\mu^{s}_{j} - \mu^{s^{\prime}}_{j} )^{T} \bm{\Phi}^{T}. \label{eq:M}
\end{align}
\end{small}

Similarly, after applying the transformation $\mathbf{W}$, average class discrepancy $\Psi^{\textit{acd}}(\mathcal{P})$ \eqref{eq:acd_def}, multidomain between-class scatter $\Psi^{\textit{mbs}}(\mathcal{D})$ \eqref{eq:mbs_def}, and multidomain within-class scatter $\Psi^{\textit{mws}}(\mathcal{D})$ \eqref{eq:mws_def} are given by:
\begin{small}
\begin{align}
& \Psi^{\textit{acd}}_{\mathbf{B}} = \tr \left( \mathbf{B}^{T} \mathbf{FB} \right), \Psi^{\textit{mbs}}_{\mathbf{B}} = \tr \left( \mathbf{B}^{T} \mathbf{PB} \right), \nonumber \\
& \Psi^{\textit{mws}}_{\mathbf{B}} = \tr \left( \mathbf{B}^{T} \mathbf{QB} \right),
\label{eq:acd_B}
\end{align}
\end{small}
where
\begin{small}
\begin{align}
&\mathbf{F} = \frac{1}{\binom{c}{2}} \sum_{1 \le j < j^{\prime} \le c} \bm{\Phi} (u_{j} - u_{j^{\prime}}) ( u_{j} - u_{j^{\prime}})^{T} \bm{\Phi}^{T}, \label{eq:F} \\
& \mathbf{P} = \frac{1}{n} \sum^{c}_{j=1} n_{j} \bm{\Phi} ( u_{j} - \bar{ u } ) ( u_{j} - \bar{ u } )^{T} \bm{\Phi}^{T}, \label{eq:P} \\
& \mathbf{Q} = \frac{1}{n} \sum^{c}_{j=1} \sum^{m}_{s=1} \sum^{n^{s}_{j}}_{i=1} \bm{\Phi} ( \phi(\bm{x}^{s}_{i \in j}) - u_{j}) ( \phi(\bm{x}^{s}_{i \in j}) - u_{j} )^{T} \bm{\Phi}^{T}. \label{eq:Q}
\end{align}
\end{small}

\subsection{EMPIRICAL ESTIMATION}

In practice, one exploits a finite number of instances from $m$ source domains to estimate the transformed measures in $\mathbb{R}^{q}$. Since all measures depend on $\mu^{s}_{j}$ and $u_{j}$, the estimation of measures reduces to the estimation of $\mu^{s}_{j}$ and $u_{j}$ ($s = 1,\dots,m, j = 1, \dots, c$) using the source data. Let $\bm{x}^{s}_{i \in j}$ denote the feature vector of $i$th instance of class $j$ in domain $s$ and $n^{s}_{j}$ denote the total number of instances of class $j$ in domain $s$, each $\mu^{s}_{j}$ can be empirically estimated as
\begin{align}
\hat{ \mu }^{s}_{j} & = \frac{1}{n^{s}_{j}} \sum^{n^{s}_{j}}_{i=1} \phi(\bm{x}^{s}_{i \in j}). \label{eq:emp_mu_s_j}
\end{align}

The empirical estimation of $u_{j}$ requires $\mathbb{P}(S=s|Y=j)$, which can be estimated using Bayes rule as $\mathbb{P}(S=s|Y=j) = \frac{\text{Pr}(Y=j|S=s) \text{Pr}(S=s) }{\text{Pr}(Y=j)}$. Since it is usually hard to model the underlying distribution over $\mathfrak{P}_{\mathcal{X} \times \mathcal{Y}}$, we assume that the probabilities of sampling all source domains are equal, i.e., $\text{Pr}(S=s) = \frac{1}{m}$ for $s=1,\dots, m$ given $\mathcal{D}$. As a result, $\mathbb{P}(S=s|Y=j) = \frac{n^{s}_{j} / n^{s}}{ \sum^{m}_{s'=1}(n^{s'}_{j}/ n^{s'}) }$. Then the empirical estimation of the mean representation of class $j$ in RKHS $\mathcal{H}$ is given by
%${= \frac{n^{s}_{j} / n^{s}}{ \sum^{m}_{s=1}(n^{s}_{j}/ n^{s}) }$, where $n^{s}_{j}$ denote the number of points of class $j$ in domain $s$. $s=1,\dots, m$, $j=1,\dots,C$.    note that there is according to assumption \ref{assp:data} \textcolor{red}{$\mathbb{P}^{s}$ is sampled uniformly from $\mathfrak{P}_{\mathcal{X} \times \mathcal{Y}}$, there is $\text{Pr}(\mathbb{P}^{s}) = \frac{1}{m}$ for $s=1,\dots, m$ given $\mathcal{D}$. $\mathbb{P}(S=s|Y=j)$  
\begin{align}
\hat{u}_{j} & = \sum^{ m }_{s=1} \frac{n^{s}_{j} / n^{s}}{ \sum^{m}_{s'=1}(n^{s'}_{j}/ n^{s'}) } \hat{ \mu }^{s}_{j}. \label{eq:emp_u_j}
\end{align}

By substituting the empirical class-conditional kernel mean embedding \eqref{eq:emp_mu_s_j} and empirical mean representation of each class \eqref{eq:emp_u_j} into formulation \eqref{eq:M}, \eqref{eq:F}, \eqref{eq:P}, and \eqref{eq:Q}, these matrices can be estimated from $m$ observed sample sets using the \emph{kernel trick} \citep{Theodoridis:2008:PRF:1457541}.

\subsection{THE OPTIMIZATION PROBLEM}
Following the solution in \cite{ghifary2017scatter} and in the spirit of Fisher's discriminant analysis \citep{mika1999fisher}, we unify measures introduced in previous sections and solve the matrix $\mathbf{B}$ as
\begin{align}
\argmax_{\mathbf{B}} \frac{ \Psi^{\textit{acd}}_{\mathbf{B}}+ \Psi^{\textit{mbs}}_{\mathbf{B}} }{ \Psi^{\textit{add}}_{\mathbf{B}} + \Psi^{\textit{mws}}_{\mathbf{B}} }.
\label{eq:obj_ori}
\end{align}
It can be seen that through maximizing the numerator, the objective \eqref{eq:obj_ori} preserves the separability among different classes. Through minimizing the denominator, \eqref{eq:obj_ori} tries to find a domain-invariant transformation which improves the overall compactness of distributions of all classes and make the class-conditional distributions of the same class as close as possible.

By substituting the transformed average domain discrepancy \eqref{eq:add_B}, average class discrepancy, multidomain between-class scatter, and multidomain within-class scatter \eqref{eq:acd_B}, adding $\mathbf{W}^{T}\mathbf{W} = \mathbf{B}^{T}\mathbf{KB}$ for regularization, and introducing a trade-off between the measures for further flexibility into the objective \eqref{eq:obj_ori}, we aim to achieve
\begin{align}
\argmax_{\mathbf{B}} = \frac{ \tr \left( \mathbf{B}^{T} \left( \beta \mathbf{F} + (1 - \beta) \mathbf{P} \right) \mathbf{B} \right) }{ \tr \left( \mathbf{B}^{T}( \gamma \mathbf{G} + \alpha \mathbf{Q} + \mathbf{K})\mathbf{B} \right) },
\label{eq:obj_params}
\end{align}
where $\alpha$, $\beta$, and $\gamma$ are trade-off parameters controlling the significance of corresponding measures. Since the objective \eqref{eq:obj_params} is invariant to re-scaling of $\mathbf{B}$, rewriting \eqref{eq:obj_params} as a constrained optimization problem and setting the derivative of its Lagrangian to zero (see Appendix B) yields the following generalized eigenvalue problem:
\begin{align}
\left( \beta \mathbf{F} + (1 - \beta) \mathbf{P} \right) \mathbf{B} = \left( \gamma \mathbf{G} + \alpha \mathbf{Q} + \mathbf{K} \right) \mathbf{B} \bm{\Gamma},
\label{eq:obj_last}
\end{align}
where $\bm{\Gamma}=\text{diag}(\lambda_{1}, \dots, \lambda_{q})$ is the diagonal matrix collecting $q$ leading eigenvalues, $\mathbf{B}$ is the matrix collecting corresponding eigenvectors.\footnote{In practice, $\gamma \mathbf{G} + \alpha \mathbf{Q} + \mathbf{K}$ is replaced by $\gamma \mathbf{G} + \alpha \mathbf{Q} + \mathbf{K} + \epsilon \mathbf{I}$ for numerical stability, where $\epsilon$ is a small constant and set to be $1\mathrm{e}{-5}$ for kernel-based DG methods in all experiments.} 

Computing the matrices $\mathbf{G}$, $\mathbf{F}$, $\mathbf{P}$, and $\mathbf{Q}$ takes $\mathcal{O}(n^2)$. Solving the generalized eigenvalue problem \eqref{eq:obj_last} takes $\mathcal{O}(qn^2)$. In sum, the overall computational complexity is $\mathcal{O}(n^2+qn^2)$, which is the same as existing kernel-based methods. After the transformation learning, unseen target instances can then be transformed into $\mathbb{R}^{q}$ using $\mathbf{B}$ and $\bm{\Gamma}$. 
% According to \eqref{eq:obj_last}, $\mathbf{B}$ can be solved for transforming instances into the subspace. 
We term the proposed method Multidomain Discriminant Analysis (MDA) and summarize the algorithm in Algorithm \ref{alg:mda}.

\LinesNotNumbered
\begin{algorithm}[ht]                %<-- Remove float environment
	\SetKwData{Left}{left}
	\SetKwData{This}{this}
	\SetKwData{Up}{up}
	\SetKwFunction{Union}{Union}
	\SetKwFunction{FindCompress}{FindCompress}
	\SetKwInOut{Input}{input}
	\SetKwInOut{Output}{output}
	\Input{$\mathcal{D} = \{ \mathcal{D}^{s} \}_{s=1}^{m}$ - the set of instances from $m$ domains;\\ $\alpha$, $\beta$, $\gamma$ - trade-off parameters.}
	\Output{Optimal projection $\mathbf{B}_{n \times q}$; \\ corresponding eigenvalues $\bm{\Gamma}$.}
	\BlankLine
	
	%\nextnr
	Construct kernel matrix $\mathbf{K}$ from $\mathcal{D}$, whose entry on $i$th row and $i^{\prime}$th column $[ \mathbf{K} ]_{ii^{\prime}} = k( \bm{x}_{i}, \bm{x}_{i^{\prime}} )$\;
	%\nextnr
	Compute matrices $\mathbf{G}$, $\mathbf{F}$, $\mathbf{P}$, $\mathbf{Q}$ from \eqref{eq:M}, \eqref{eq:F}, \eqref{eq:P}, \eqref{eq:Q}, respectively\;
	Center the kernel matrix as $\mathbf{K} \leftarrow \mathbf{K} - \mathbf{1}_{n} \mathbf{K} - \mathbf{K} \mathbf{1}_{n}$ $+ \mathbf{1}_{n} \mathbf{K} \mathbf{1}_{n}$, where $\mathbf{1}_{n}\in \mathbb{R}^{n\times n}$ denotes a matrix with all entries equal to $\frac{1}{n}$; \\
	%\nextnr
	Solve \eqref{eq:obj_last} for the projection $\mathbf{B}$ and corresponding eigenvalues $\bm{\Gamma}$, then select $q$ leading components.\
	\textbf{Target domain transformation}\ \\
	%\nextnr
	Denote the set of instances from the target domain by $\mathcal{D}^{t}$, one first constructs the kernel matrix $\mathbf{K}^{t}$, where $[ \mathbf{K}^{t} ]_{i^{\prime}i}=k( \bm{x}^{t}_{i'}, \bm{x}_{i} )$, $\forall \bm{x}^{t}_{i'} \in \mathcal{D}^{t}$, $\forall \bm{x}_{i} \in \mathcal{D}$; \\
	%\nextnr
	Center the kernel matrix as $\mathbf{K}^{t} \leftarrow \mathbf{K}^{t} - \mathbf{1}_{n^{t}} \mathbf{K}^{t} - $ $ \mathbf{K}^{t} \mathbf{1}_{n} + \mathbf{1}_{n^{t}} \mathbf{K}^{t} \mathbf{1}_{n}$, where $n^{t}$ is the number of instances in $\mathcal{D}^{t}$; \\
	%\nextnr
	Then the transformed features of the target domain are given by $\mathbf{X}^{t} = \mathbf{K}^{t} \mathbf{B} \bm{ \Gamma }^{-\frac{1}{2}} $.
	\caption{Multidomain discriminant analysis}\label{alg:mda}
\end{algorithm}

\section{LEARNING THEORY ANALYSIS}
We analyze the the excess risk and generalization error bound after applying feature transformation $\mathbf{W}$.

In standard setting of learning theory analysis, the decision functions of interest are $f: \mathcal{X} \mapsto \mathcal{Y}$. However, our DG problem setting is much more general in the sense that not only $\mathbb{P}(X)$ changes (as in the covariate shift setting), but $\mathbb{P}(Y|X)$, which corresponds to $f$ in learning theory, also changes across domains. As a result, the decision functions of interest in our analysis are $f: \mathfrak{P}_{\mathcal{X}} \times \mathcal{X} \mapsto \mathcal{Y}$. $\mathbb{P}^{s}$ and $\mathbb{P}^{s}_{X}$ are used interchangeably to denote the marginal distribution of $X$ in domain $s$.

Let $\bar{k}$ be a kernel on $\mathfrak{P}_{\mathcal{X}} \times \mathcal{X}$ and $\mathcal{H}_{\bar{k}}$ be the associated RKHS. As in \cite{blanchard2011generalizing}, we consider kernel $\bar{k} = k_{\mathbb{P}} (\mathbb{P}^{1}, \mathbb{P}^{2}) k_{x}(x_{1}, x_{2})$, where $k_{\mathbb{P}}$ and $k_{x}$ are kernels on $\mathfrak{P}_{\mathcal{X}}$ and $\mathcal{X}$, respectively. To ensure that $\bar{k}$ is universal, we consider a particular form for $k_{\mathbb{P}}$. Let $k^{\prime}_{x}$ be another kernel on $\mathcal{X}$ and $\mathcal{H}_{ k^{\prime}_{x} }$ be its associated RKHS, $\gamma$ be a mapping $\gamma: \mathfrak{P}_{\mathcal{X}} \mapsto \mathcal{H}_{k^{\prime}_{x}}$. Then $k_{\mathbb{P}}$ defined as a kernel $k_{\gamma}$ on $\mathcal{H}_{ k^{\prime}_{x} }$, i.e. $k_{\mathbb{P}} (\mathbb{P}^{1}, \mathbb{P}^{2}) = k_{\gamma}( \gamma( \mathbb{P}^{1} ), \gamma( \mathbb{P}^{2} ) )$ would lead $\bar{k}$ to be universal \citep{blanchard2011generalizing}. We consider following assumptions regarding the kernels and loss function in our analysis:
\begin{assp}\label{assp:k_1}
The kernels $k_{x}$, $k^{\prime}_{x}$ and $k_{\gamma}$ are bounded respectively by $U^{2}_{ k_{x} }$, $U^{2}_{ k^{\prime}_{x} }$ and $U^{2}_{k_{\gamma}}$.
\end{assp}
\begin{assp}\label{assp:k_2}
The canonical feature map $\gamma_{ k_{\gamma} }: \mathcal{H}_{ k^{\prime}_{x} } \mapsto \mathcal{H}_{ k_{\gamma} }$, where $\mathcal{H}_{ k_{\gamma} }$ is the RKHS associated with $k_{\gamma}$, fulfills that $\forall v, w \in \mathcal{H}_{ k^{\prime}_{X} }$, there is a constant $L_{ k_{\gamma} }$ satisfying
$$ \Vert \gamma_{ k_{\gamma} }(v) - \gamma_{ k_{\gamma} }(w) \Vert \le L_{ k_{\gamma} } \Vert v - w \Vert. $$
\end{assp}
\begin{assp}\label{assp:loss}
The loss function $\ell: \mathbb{R} \times \mathcal{Y} \mapsto \mathbb{R}_{+}$ is $L_{\ell}$-Lipschitz in its first variable and bounded by $U_{\ell}$.
\end{assp}

Assumption 2 and 3 are satisfied when the kernels are bounded. An example of widely adopted bounded kernel is the Gaussian kernel. As a result, we also adopt Gaussian kernel throughout our algorithm. 

Let $\tilde{X}^{t} = ( \mathbb{P}^{t}_{X}, X^{t} )$ and $Y^{t}$ denote the extended input and output pattern of decision function $f$ over target domain, respectively. The quantity of interest is the excess risk, which is the difference between expected test loss of empirical loss minimizer and expected loss minimizer. For functions in the unit ball centered at the origin in the RKHS of $\phi(\tilde{X}^{t})$, the control of the excess risk is given in the following theorem.

\begin{thrm}
Under assumptions 2 -- 4, and further assuming that $\Vert \hat{f} \Vert_{\mathcal{H}_{ \bar{k} } } \le 1$ and $\Vert f^{*} \Vert_{\mathcal{H}_{ \bar{k} } } \le 1$, where $\hat{f}$ denotes the empirical risk minimizer and $f^{*}$ denotes the expected risk minimizer, then with probability at least $1 - \delta$ there is
\begin{align}
& \mathbb{E}[ \ell( \hat{f}(\tilde{X^{t}} \mathbf{W}), Y^{t} ) ] - \mathbb{E}[ \ell ( f^{*}(\tilde{X^{t}} \mathbf{W}), Y^{t} ) ] \nonumber \\
\le & \ 4 L_{\ell} L_{ k_{\gamma} } U_{k^{\prime}_{x}} U_{k_{x}} \sqrt{ \frac{ \tr( \mathbf{B}^{T} \mathbf{KB}) }{ n } } + \sqrt{ \frac{ 2 \log 2 \delta^{-1} } { n } },
\label{eq:bound_1}
\end{align}
where the expectations are taken over the joint distribution of the test domain $\mathbb{P}^{t}(X^{t}, Y^{t})$, $n$ is the number of training samples, and $\mathbf{K} = \bm{\Phi} \bm{\Phi}^{T}$.
\end{thrm}

See Appendix C for proof. The first term in the bound above involves the size of the distortion $\tr(\mathbf{B}^{T} \mathbf{KB})$ introduced by $\mathbf{B}$. Therefore, a poor choice of $\mathbf{B}$ would loose the guarantee. The second term is of order $O(n^{-1/2})$ so it would converge to zero as $n$ tends to infinity given $\delta$.

Another quantity of interest is the generalization error bound, which is the difference between the expected test loss and empirical training loss of the empirical loss minimizer. The generalization error bound of DG in a general setting is given in \cite{blanchard2011generalizing}. Therefore, we derive it for the case where one applies feature transformation involving $\mathbf{B}$. Let $\hat{ \tilde{X} }^{s}_{i}$ denote the input pattern $(\hat{ \mathbb{P} }^{s}, x^{s}_{i})$, where $\hat{ \mathbb{P} }^{s}$ is the empirical distribution over features in domain $s$, and $x^{s}_{i}$ is the $i$th observed feature in domain $s$. Similarly, $y^{s}_{i}$ is the $i$th label in domain $s$. With $\mathcal{E}(f, \infty)$ being the expected test loss, the generalization bound involving $\mathbf{B}$ is given in the following theorem.

\begin{table*}[t]
	\centering
	\caption{Generating Distributions of Synthetic Data}
	\label{tb:syn_data}
	\resizebox{\linewidth}{!}{%
	\begin{tabular}{cccccccccc}
		\toprule
        Domain & \multicolumn{3}{c}{Domain 1} & \multicolumn{3}{c}{Domain 2} & \multicolumn{3}{c}{Domain 3} \\
        \cmidrule(r){2-4} \cmidrule(r){5-7} \cmidrule(r){8-10}
		Class & 1 & 2 & 3 & 1 & 2 & 3 & 1 & 2 & 3 \\
		\midrule
		$X_{1}$ & (1, 0.3) & (2, 0.3) & (3, 0.3) & (3.5, 0.3) & (4.5, 0.3) & (5.5, 0.3) & (8, 0.3) & (9.5, 0.3) & (10, 0.3) \\
		$X_{2}$ & (2, 0.3) & (1, 0.3) & (2, 0.3) & (2.5, 0.3) & (1.5, 0.3) & (2.5, 0.3) & (2.5, 0.3) & (1.5, 0.3) & (2.5, 0.3) \\
		\# instances & 50 & 50 & 50 & 50 & 50 & 50 & 50 & 50 & 50 \\
		\bottomrule
	\end{tabular}}
\end{table*}

\begin{thrm}
Under assumptions 2 -- 4, and assuming that all source sample sets are of the same size, i.e. $n^{s} = \bar{n}$ for $s=1, \dots, m$, then with probability at least $1 - \delta$ there is
\begin{flalign}
    & \sup_{ \Vert f \Vert_{ \mathcal{H}_{ \bar{k} } } \le 1 } \left\vert \frac{1}{m} \sum^{m}_{s=1} \frac{1}{n^{s}} \sum^{n^{s}}_{i=1} \ell \left( f( \hat{\tilde{X}}^{s}_{i} \mathbf{W} ), y^{s}_{i} \right) - \mathcal{E}(f, \infty) \right\vert \nonumber \\
    \le & U_{\ell} \left( \sqrt{ \frac{ \log 2 \delta^{-1} }{ 2m \bar{n} } } + \sqrt{ \frac{ \log \delta^{-1} }{ 2m } } \right) + \sqrt{ \tr(\mathbf{B}^{T} \mathbf{KB} ) } \nonumber \\
    & \ \left( c_{1} \sqrt{ \frac{\log 2 \delta^{-1} m }{ \bar{n}} } + c_{2} \left( \sqrt{ \frac{1}{m \bar{n}} } + \sqrt{ \frac{1}{m} } \right) \right), \end{flalign} 
where $c_{1} = 2\sqrt{2} L_{\ell} U_{k_{x}} L_{k_{\gamma}} U_{k^{\prime}_{x}}$, $c_{2} = 2 L_{\ell} U_{k_{x}} U_{k_{\gamma}}$.
\end{thrm}

\begin{figure*}[t]
	\begin{subfigure}{.19\linewidth}
		\centering
		\includegraphics[width=\linewidth]{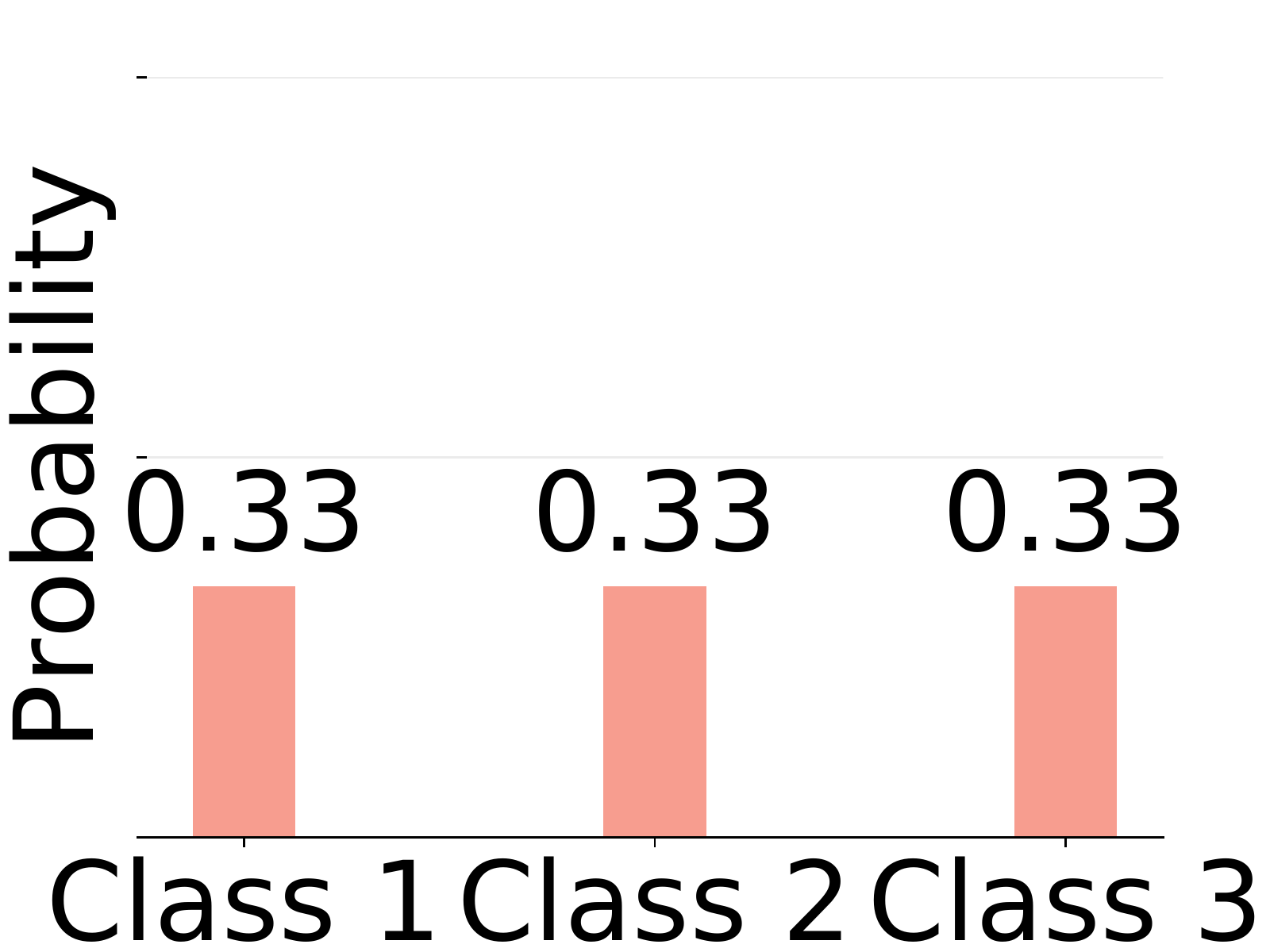}
		\caption{}
		\label{Fig:syn_0}
	\end{subfigure}
	\begin{subfigure}{.19\linewidth}
		\centering
		\includegraphics[width=\linewidth]{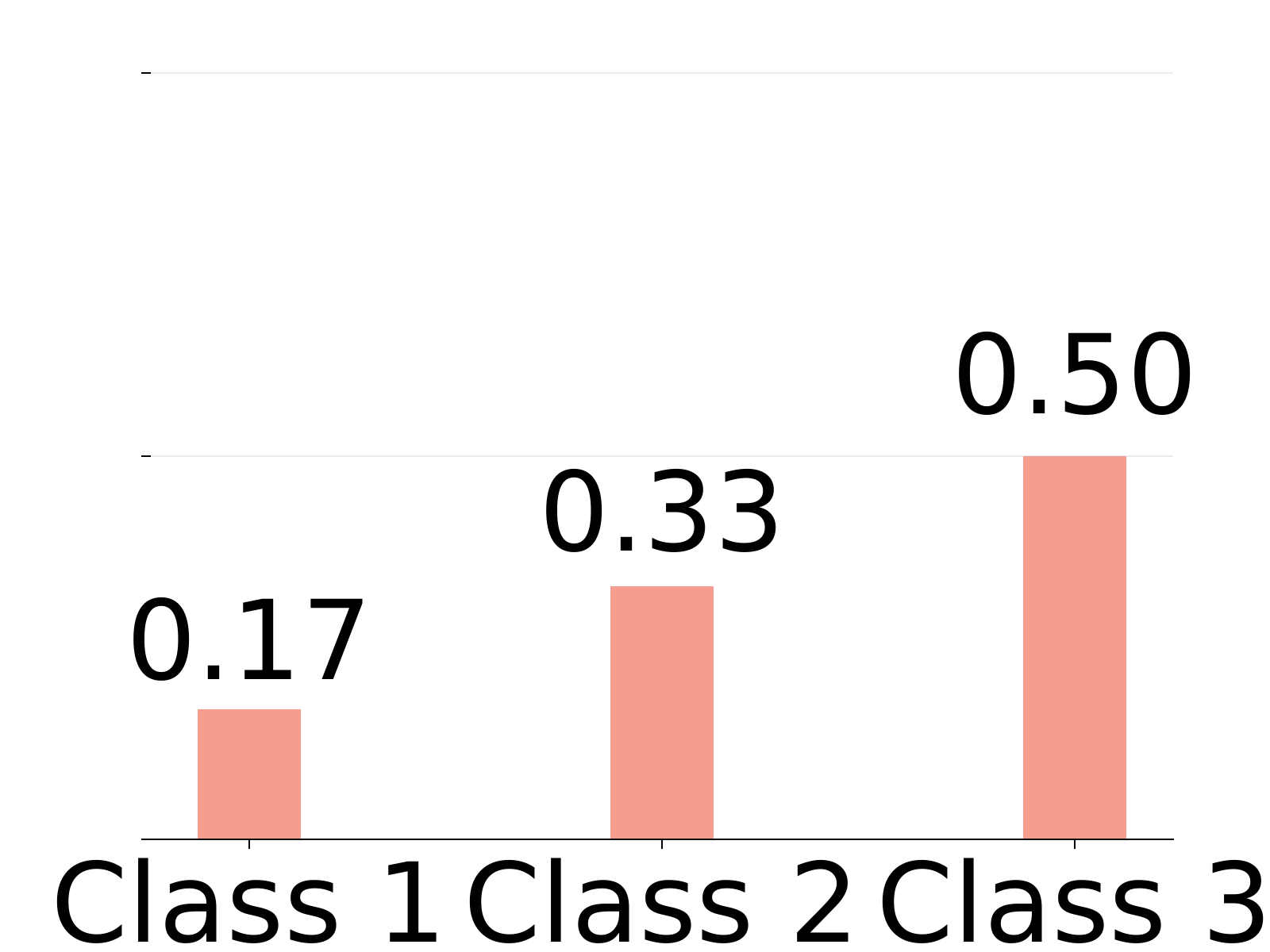}
		\caption{}
		\label{Fig:syn_1}
	\end{subfigure}
	\begin{subfigure}{.19\linewidth}
		\centering
		\includegraphics[width=\linewidth]{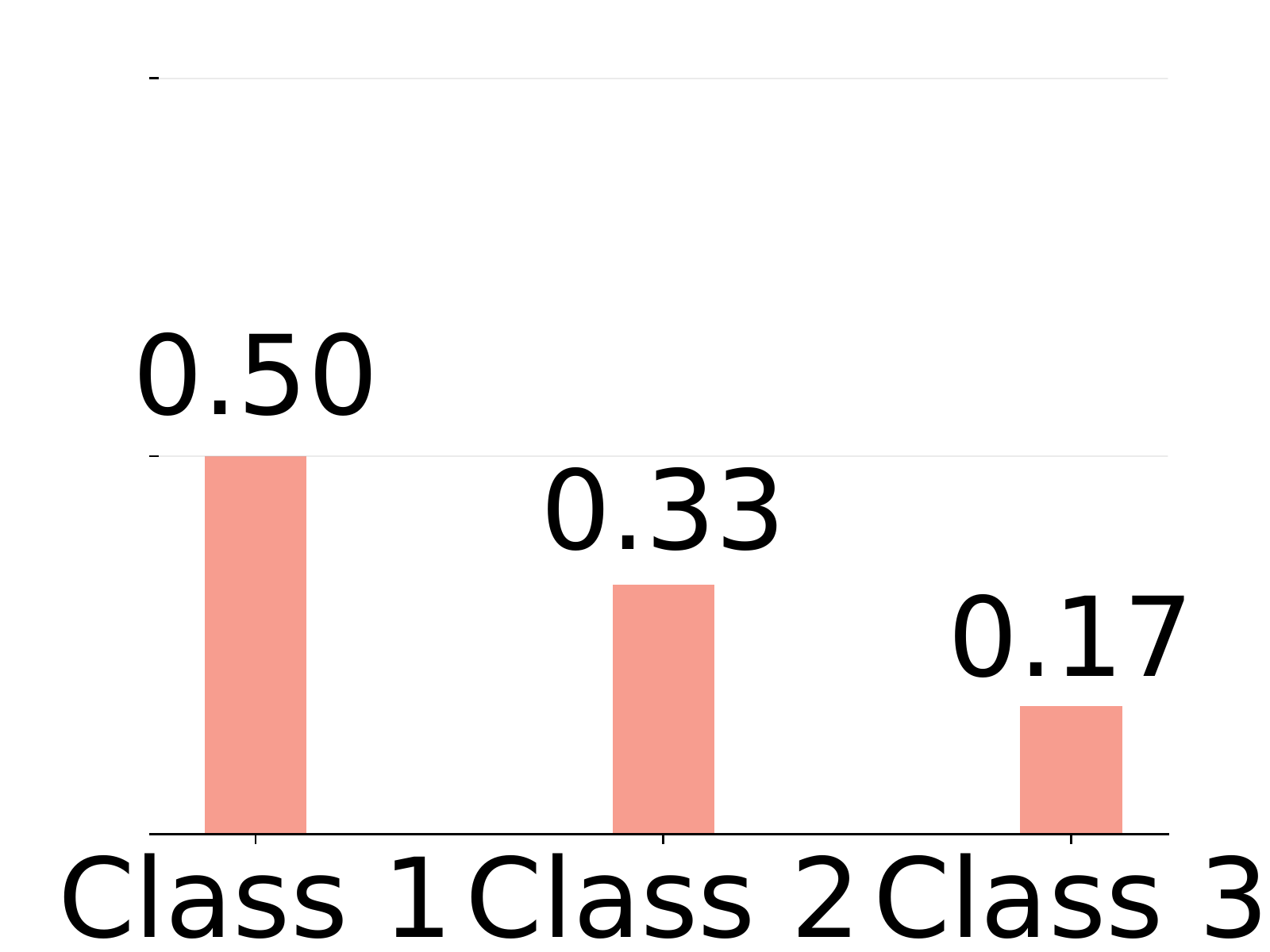}
		\caption{}
		\label{Fig:syn_2}
	\end{subfigure}
	\begin{subfigure}{.19\linewidth}
		\centering
		\includegraphics[width=\linewidth]{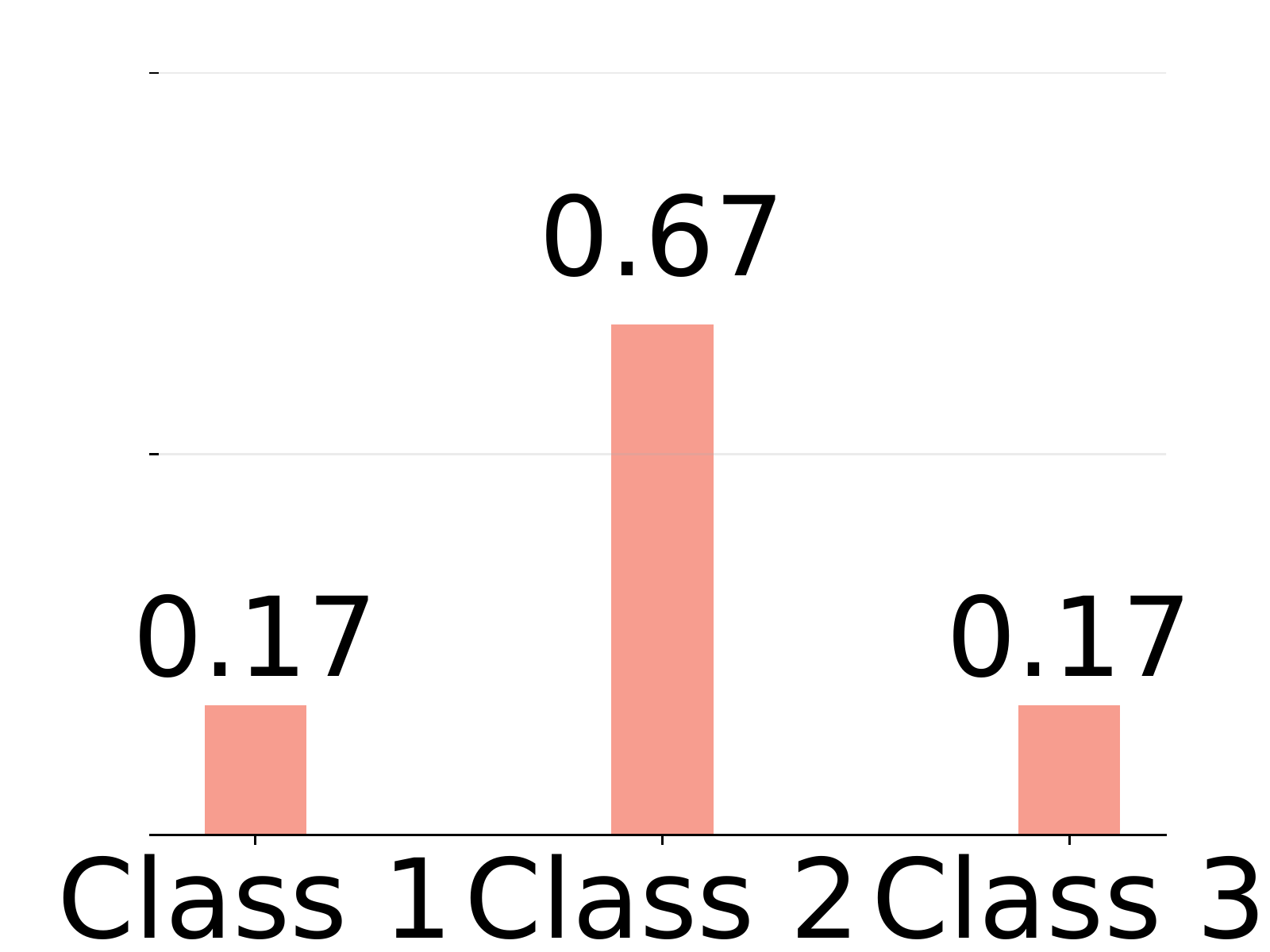}
		\caption{}
		\label{Fig:syn_3}
	\end{subfigure}
	\begin{subfigure}{.19\linewidth}
		\centering
		\includegraphics[width=\linewidth]{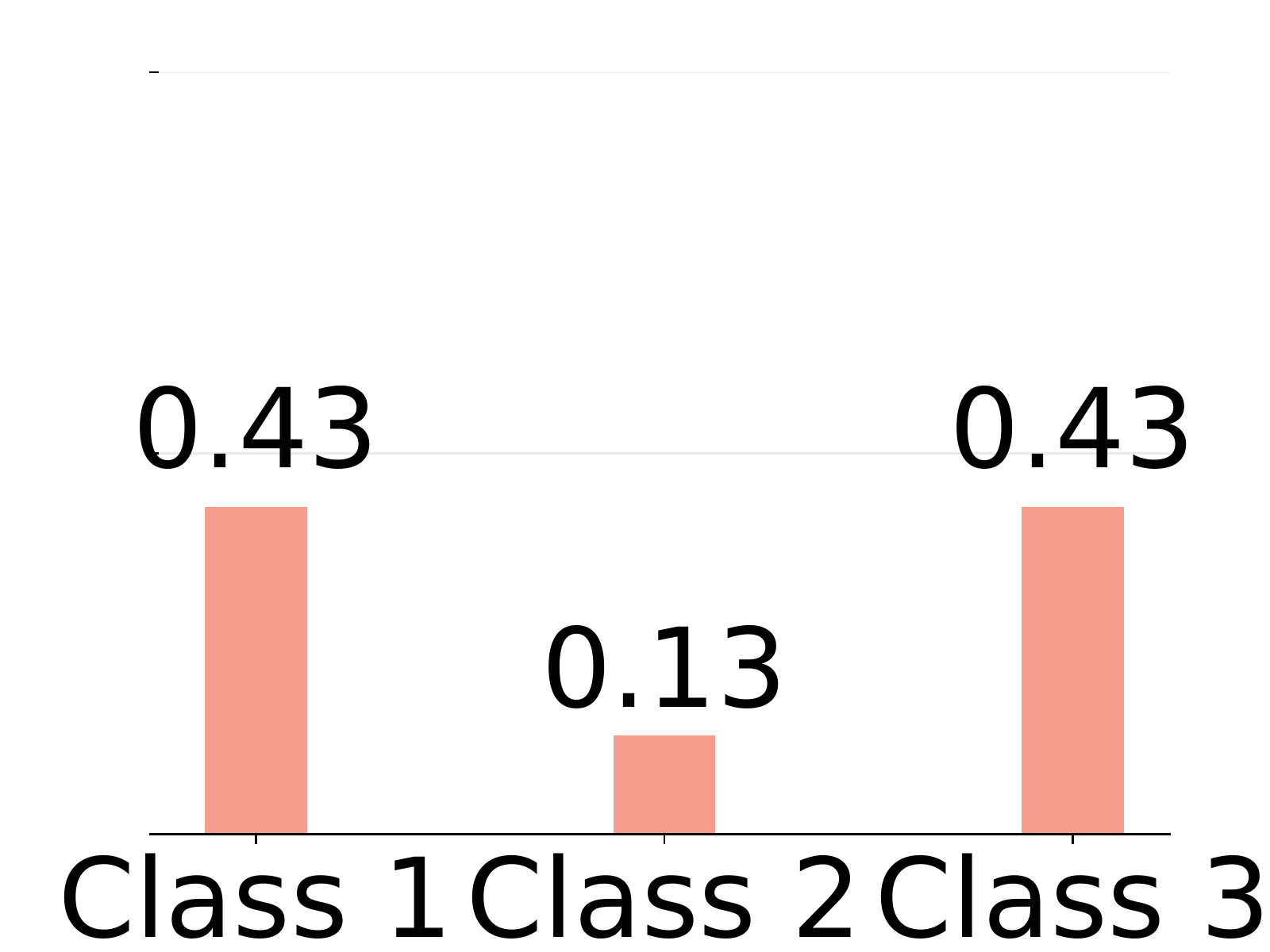}
		\caption{}
		\label{Fig:syn_4}
	\end{subfigure}
	\caption{Class Prior Distributions $\mathbb{P}(Y)$ in Synthetic Experiments.}
\end{figure*}

\begin{table*}[htbp]
	\centering
	\caption{Accuracy (\%) of Synthetic Experiments (\boldred{bold italic} and \textbf{bold} indicate the best and second best).}
	\label{tb:re_syn}
	\resizebox{\linewidth}{!}{%
	\begin{tabular*}{\textwidth}{l @{\extracolsep{\fill}} ccccccccc}
		\toprule
		$\mathbb{P}^{1}(Y)$ & \ref{Fig:syn_0} & \ref{Fig:syn_1} & \ref{Fig:syn_2} & \ref{Fig:syn_3} & \ref{Fig:syn_4} & \ref{Fig:syn_0} & \ref{Fig:syn_0} & \ref{Fig:syn_0} & \ref{Fig:syn_0} \\
		$\mathbb{P}^{2}(Y)$ & \ref{Fig:syn_0} & \ref{Fig:syn_0} & \ref{Fig:syn_0} & \ref{Fig:syn_0} & \ref{Fig:syn_0} & \ref{Fig:syn_1} & \ref{Fig:syn_2} & \ref{Fig:syn_3} & \ref{Fig:syn_4} \\ 
		\midrule
		SVM & 56.00 & 34.00 & 33.33 & 33.33 & 33.33 & 33.33 & 40.00 & 36.00 & 60.00 \\
		KPCA & 66.00 & 62.00 & 66.67 & 33.33 & 33.33 & 65.33 & 36.00 & 40.00 & 14.00 \\
		KFD & 78.67 & 38.67 & 46.00 & 74.67 & 47.33 & 49.33 & 34.00 & 19.33 & 76.00 \\
		L-SVM & 56.00 & 60.00 & 64.00 & 62.00 & 60.67 & 64.67 & 45.33 & 46.00 & 59.33 \\
        DICA & \textbf{93.33} & 84.67 & 76.00 & \textbf{84.00} & 84.67 & 54.00 & \boldred{95.33} & 71.33 & \textbf{88.67} \\
        SCA & 79.33 & 72.00 & \textbf{84.67} & 57.33 & 76.00 & 59.33 & 84.67 & 61.33 & 81.33 \\
        CIDG & 90.67 & \textbf{87.33} & 74.67 & 77.33 & \textbf{86.67} & \textbf{83.33} & \textbf{92.00} & \textbf{82.00} & 86.00 \\
        MDA & \boldred{96.67} & \boldred{96.00} & \boldred{97.33} & \boldred{94.00} & \boldred{94.00} & \boldred{91.33} & \boldred{95.33} & \boldred{94.00} & \boldred{94.00} \\
		\bottomrule
	\end{tabular*} }
\end{table*}

See Appendix D for proof. The first term is of order $O(m^{-1/2})$ and converges to zero as $m \rightarrow \infty$. The second term, involving $\tr(\mathbf{B}^{T} \mathbf{KB})$, again depends on the choice of $\mathbf{B}$. The remaining part would converge to zero only if both $m$ and $\bar{n}$ tend to infinity and $\log m / \bar{n} = o(1)$. In a general perspective, our method, as well as existing ones relying on feature extraction %\citep{muandet2013domain, ghifary2017scatter, AAAI1816595}, 
can all be viewed as ways of finding transformation $\mathbf{B}$, which could minimize the generalization bound on the test domain, under different understandings of the DG problem.

\section{EXPERIMENTS}

\subsection{EXPERIMENTAL CONFIGURATION}
We compare MDA with the following 9 methods:
\begin{itemize}
    \item Baselines: 1-nearest neighbor (1NN) and support vector machine (SVM) with RBF kernel.
    \item Feature extraction methods: kernel principal component analysis (KPCA; \citep{scholkopf1998nonlinear}) and kernel Fisher discriminant analysis (KFD; \citep{mika1999fisher}). 1NN is applied on the transformed features for classification.
    \item SVM-based DG method: low-rank exemplar-SVMs (L-SVM; \citep{xu2014exploiting}).
    \item Neural network-based DG method: CCSA \citep{Motiian_2017_ICCV}. The network setting follows \citep{Motiian_2017_ICCV}.% which used two fully connected layers of output size 1,024 and 128, respectively, and another fully connected layer with softmax activation for classification.
    \item Kernel-based DG methods: domain invariant component analysis (DICA; \citep{muandet2013domain}), scatter component analysis (SCA; \citep{ghifary2017scatter}), and conditional invariant DG (CIDG; \citep{AAAI1816595}). 1NN is applied on the domain-invariant representations for classification.
\end{itemize}

For 1NN and SVM baselines, instances in source domains are directly combined for training in both synthetic and real data experiments. For other methods, in experiments with synthetic data, the models are trained on two source sample sets, validated on one target sample set, and tested on the other target sample set. In real data experiments, we first selected hyper-parameters by 5-fold cross-validation using only labeled source sample sets. Then the model with optimal parameter settings was applied on the target domain. The classification accuracy on the target domain serves as the evaluation criterion for different methods. Since measures in section 3.2 are defined as the averaged distance, we naturally put them on a equal footing by setting $\beta = 0.5$, $\alpha = \gamma = 1$. Thus in practice, these parameters are set to be an interval containing values in the above balanced case. The hyper-parameters required by each method and the values validated in the experiments are given in Appendix E.

\begin{table}[htbp] 
	\centering
	\caption{Accuracy (\%) of Office+Caltech Dataset}
	\label{tb:re_office}
	\resizebox{\linewidth}{!}{%
	\begin{tabular}{lcccccc}
		\toprule
		Target & A & C & A, C & W, D & W, C & D, C \\
		\midrule
		1NN & 89.80 & 84.16 & 78.63 & 80.60 & 86.29 & 85.28 \\
		SVM & 91.96 & \textbf{85.75} & 77.66 & \textbf{84.51} & 87.31 & 86.72 \\
		KPCA & 89.87 & 83.35 & 66.46 & 79.65 & 85.83 & 84.45 \\
		KFD & 91.75 & 85.66 & 74.68 & 82.96 & 87.59 & 86.64 \\
		L-SVM & 91.64 & 85.39 & \textbf{80.55} & 83.33 & \textbf{88.09} & \textbf{87.10} \\
		CCSA & 90.98 & 83.37 & 77.56 & 80.04 & 85.80 & 84.91 \\
        DICA & \textbf{92.59} & 83.17 & 63.67 & 83.85 & 87.59 & 86.25 \\
        SCA & 91.96 & 83.35 & 73.04 & 83.85 & 87.31 & 86.25 \\
        CIDG & 92.38 & 81.39 & 69.87 & 82.74 & 87.45 & 85.63 \\
        MDA & \boldred{93.47} & \boldred{86.89} & \boldred{82.56} & \boldred{84.89} & \boldred{88.91} & \boldred{88.23} \\
		\bottomrule
	\end{tabular} }
\end{table}

\begin{table*}[htbp]
	\centering
	\caption{Accuracy (\%) of VLCS Dataset}\label{tb:re_vlcs}
	\resizebox{\linewidth}{!}{%
	\begin{tabular*}{\textwidth}{l@{\extracolsep{\fill}} cccccccccc}
		\toprule
		Target & V & L & C & S & V, L & V, C & V, S & L, C & L, S & C, S \\
		\midrule
		1NN & 60.19 & 53.57 & 89.94 & 55.74 & 57.26 & 58.54 & 50.59 & 66.06 & 58.13 & 66.25 \\
		SVM & \boldred{68.57} & 59.26 & \boldred{93.99} & \boldred{65.27} & \boldred{61.80} & \boldred{64.39} & \textbf{55.89} & 70.08 & \boldred{64.10} & \textbf{71.09} \\
		KPCA & 60.69 & 54.86 & 83.89 & 55.61 & 57.54 & 57.50 & 49.46 & 67.48 & 56.05 & 66.15 \\
		KFD & 61.64 & \textbf{60.54} & 86.78 & 58.75 & 57.33 & 46.84 & 53.20 & 70.03 & 61.64 & 67.87 \\
		L-SVM & 58.14 & 39.87 & 75.56 & 52.92 & 52.25 & 56.64 & 48.27 & 61.24 & 56.65 & 66.27 \\
		CCSA & 60.39 & 58.80 & 86.88 & 59.87 & 59.27 & 55.02 & 51.56 & 69.94 & 61.41 & 68.49 \\
        DICA & 62.71 & 59.38 & 86.15 & 57.28 & 58.11 & 55.08 & 55.17 & 70.01 & 61.44 & 70.30 \\
        SCA & 62.13 & 58.24 & 88.48 & \textbf{60.66} & \textbf{60.66} & 57.59 & 54.66 & \textbf{71.90} & 61.57 & 70.71 \\
        CIDG & 64.16 & 57.91 & 90.11 & 59.48 & 60.54 & 54.56 & 55.77 & 70.74 & 62.48 & 69.83 \\
        MDA & \textbf{66.86} & \boldred{61.78} & \textbf{92.64} & 59.58 & 59.60 & \textbf{63.72} & \boldred{55.98} & \boldred{72.88} & \textbf{62.83} & \boldred{72.00} \\
		\bottomrule
	\end{tabular*} }
\end{table*}

\subsection{SYNTHETIC DATA}
We investigate the influence of variation in the class prior distribution, $\mathbb{P}(Y)$, on different DG methods. Two-dimensional data is generated from three different domains and each domain consists of three classes. Each dimension of the data follows a Gaussian distribution $\mathcal{N}(\mu, \sigma)$, where $\mu$ is the mean and $\sigma$ is the standard deviation. The settings of the distribution of the synthetic data are listed in Table \ref{tb:syn_data}. Domains 1 and 2 are source domains and domain 3 is the target domain. % In the synthetic experiment, one sample set is generated from domain 1 and 2, respectively, for training. Two sample sets are generated from domain 3. One is used for validation and the other for testing. 
The setting in Table \ref{tb:syn_data} is the base condition where class prior distributions are uniform in all domains, i.e., $\mathbb{P}^{1}(Y)= \mathbb{P}^{2}(Y)=\mathbb{P}^{3}(Y)$. Then we change $\mathbb{P}(Y)$ of one source domain to be distributions shown in Figure \ref{Fig:syn_1} to \ref{Fig:syn_4} and keep $\mathbb{P}(Y)$ of the other source domain and target domain uniform to compare different DG methods. Note that CCSA is based on convolutional neural network and thus not suitable for 2-dimensional synthetic data.

The results of different methods on different settings of class prior distributions in source domains are given in Table \ref{tb:re_syn} (also visualized in Appendix F). The accuracy of 1NN is 33.33\% in all cases thus omitted in Table \ref{tb:re_syn}. It can be seen that MDA performs best in the base setting, as well as all settings with different $\mathbb{P}(Y)$ in source domains. DICA performs equally well as MDA in (2a, 2c) setting but its accuracy is heavily influenced by the variation in $\mathbb{P}(Y)$. Compared with other methods, MDA is much more robust against the variation in $\mathbb{P}(Y)$, which is consistent with our expectation because we essentially work with the class-conditional, not the marginal, distributions.

\subsection{OFFICE+CALTECH DATASET}
We evaluate the performance of different DG methods on Office+Caltech dataset \citep{gong2012geodesic}, which is a widely used benchmark for DG tasks. Office+Caltech consists of photos from four different datasets: Amazon (A), Webcam (W), DSLR (D), and Caltech-256 (C) \citep{griffinHolubPerona}. Since there are 10 shared classes in these datasets, photos of these classes are selected and those from the same original dataset form one domain in Office+Caltech. Thus, the domains within Office+Caltech corresponds to the biases of different data collection procedures \citep{torralba2011unbiased}.
The 4096-dimensional $\text{DeCAF}_{6}$ features \citep{donahue2014decaf} are adopted in the experiments to ensure that the feature spaces, $\mathcal{X}$, are consistent across all domains.% For  each choice of target domain, we uniformly sample 70\% of the instances from the source domain for training and use the remaining 30\% for validation. After all choices of trade-off parameters are validated, the setting with the highest validation accuracy would be applied on the target domain.

The accuracies on different choices of target domains are shown in Table \ref{tb:re_office}. MDA again performs best, yet by a smaller margin of improvement compared to that of the synthetic experiment. In particular, MDA is the only kernel-based method that outperforms 1NN in ($A$, $C$) case which is probably because of the newly proposed average class discrepancy \eqref{eq:acd_def}. L-SVM outperforms other kernel-based methods and ranks the second. Note that other 4 cases, such as A, D, C $\to$ W, are not reported since 1NN baseline could already achieve accuracies higher than $90\%$.

\subsection{VLCS DATASET}
The second real data experiment uses the VLCS dataset. It consists of photos of five common classes extracted from four datasets: Pascal VOC2007 (V) \citep{everingham2010pascal}, LabelMe (L) \citep{russell2008labelme}, Caltech-101 (C) \citep{griffinHolubPerona}, and SUN09 (S) \citep{choi2010exploiting}. Photos from the same dataset form one domain in VLCS. $\text{DeCAF}_{6}$ features of 4096 dimensions are again adopted in the experiments to ensure the consistency of feature spaces over different domains. The training and test procedures are the same as in experiments on the Office+Caltech dataset. The parameters of L-SVM were trained (validated) on 70\% (30\%) source instances due to its high complexity.

The accuracies are given in Table \ref{tb:re_vlcs}. It is interesting to see that SVM baseline outperforms all DG methods in 6 cases. This is probably because many instances of different classes are overlapped in VLCS, so using 1NN in the transformed space is more likely to misclassify them compared with SVM. Apart from SVM baseline, MDA performs best in 8 out of 10 cases compared with other DG methods. CCSA outperforms MDA in the case of S being the target domain, which may indicate that neural networks extracted better features in this case. % The success of SVM on VLCS dataset may suggest that 1NN is not able to transform the distribution of VLCS. 
Inspired by the results of SVM, kernel-based methods together with SVM classifier may be a promising direction for further VLCS accuracy improvement. %Within these 8 cases, MDA and CIDG perform equally in the sense that each of them achieves highest accuracy in 4 cases and the they both attain only a small improvement compared with the other. It is probably because instances in VLCS is highly mixed so changes in the objective do not have significant influence on the accuracies.

\section{CONCLUSION}
In this paper, we proposed a method called Multidomain Discriminant Analysis (MDA) to solve the DG problem of classification tasks. Unlike existing works, which typically assume stability of certain (conditional) distributions, %$\mathbb{P}(Y|X)$ is fixed and only $\mathbb{X}$ changes across domains. The most recent work with causal analysis assume that $\mathbb{P}(Y)$ is fixed and only $\mathbb{P}(X|Y)$ changes across domains. Compared with existing methods, 
MDA is able to solve DG problems in a more general setting where both $\mathbb{P}(Y)$ and $\mathbb{P}(X|Y)$ change across domains. The newly proposed measures, average domain discrepancy and average class discrepancy, together with two measures based on kernel Fisher discriminant analysis, are theoretically analyzed and incorporated into the objective for learning the domain-invariant feature transformation. We also prove bounds on the excess risk and generalization error for kernel-based DG methods. The effectiveness of MDA is verified by experiments on synthetic and two real benchmark datasets. 

\subsubsection*{Acknowledgements}
SH thanks Lequan Yu for comments on a previous draft of this paper. KZ acknowledges the support by National Institutes of Health (NIH) under Contract No. NIH-1R01EB022858-01, FAINR01EB022858, NIH-1R01LM012087, NIH-5U54HG008540-02, and FAIN- U54HG008540, by the United States Air Force under Contract No. FA8650-17-C-7715, and by National Science Foundation (NSF) EAGER Grant No. IIS-1829681. The NIH, the U.S. Air Force, and the NSF are not responsible for the views reported in this article. This work was partially funded by the Hong Kong Research Grants Council. 

\bibliography{egbib}

\begin{thebibliography}{39}
\providecommand{\natexlab}[1]{#1}
\providecommand{\url}[1]{\texttt{#1}}
\expandafter\ifx\csname urlstyle\endcsname\relax
  \providecommand{\doi}[1]{doi: #1}\else
  \providecommand{\doi}{doi: \begingroup \urlstyle{rm}\Url}\fi

\bibitem[Barrio et~al.(1999)Barrio, Cuestaalbertos, Matran, and
  Rodriguezrodriguez]{del1999tests}
Eustasio~Del Barrio, J~A Cuestaalbertos, Carlos Matran, and Jes U S~M
  Rodriguezrodriguez.
\newblock Tests of goodness of fit based on the $l_2$-wasserstein distance.
\newblock \emph{Annals of Statistics}, 27\penalty0 (4):\penalty0 1230--1239,
  1999.

\bibitem[Blanchard et~al.(2011)Blanchard, Lee, and
  Scott]{blanchard2011generalizing}
Gilles Blanchard, Gyemin Lee, and Clayton Scott.
\newblock Generalizing from several related classification tasks to a new
  unlabeled sample.
\newblock In \emph{Advances in Neural Information Processing Systems (NIPS)},
  pages 2178--2186, 2011.

\bibitem[Choi et~al.(2010)Choi, Lim, Torralba, and Willsky]{choi2010exploiting}
Myung~Jin Choi, Joseph~J Lim, Antonio Torralba, and Alan~S Willsky.
\newblock Exploiting hierarchical context on a large database of object
  categories.
\newblock In \emph{Proceedings of IEEE Conference on Computer Vision and
  Pattern Recognition (CVPR)}, pages 129--136, 2010.

\bibitem[Daniu{\v{s}}is et~al.(2010)Daniu{\v{s}}is, Janzing, Mooij,
  Zscheischler, Steudel, Zhang, and Sch{\"o}lkopf]{daniuvsis2010inferring}
Povilas Daniu{\v{s}}is, Dominik Janzing, Joris Mooij, Jakob Zscheischler,
  Bastian Steudel, Kun Zhang, and Bernhard Sch{\"o}lkopf.
\newblock Inferring deterministic causal relations.
\newblock In \emph{Proceedings of the Twenty-Sixth Conference on Uncertainty in
  Artificial Intelligence (UAI 2010)}, pages 143--150, 2010.

\bibitem[Donahue et~al.(2014)Donahue, Jia, Vinyals, Hoffman, Zhang, Tzeng, and
  Darrell]{donahue2014decaf}
Jeff Donahue, Yangqing Jia, Oriol Vinyals, Judy Hoffman, Ning Zhang, Eric
  Tzeng, and Trevor Darrell.
\newblock Decaf: A deep convolutional activation feature for generic visual
  recognition.
\newblock In \emph{Proceedings of the 31st International Conference on Machine
  Learning (ICML 2014)}, pages 647--655, 2014.

\bibitem[Everingham et~al.(2010)Everingham, Van~Gool, Williams, Winn, and
  Zisserman]{everingham2010pascal}
Mark Everingham, Luc Van~Gool, Christopher~KI Williams, John Winn, and Andrew
  Zisserman.
\newblock The pascal visual object classes (voc) challenge.
\newblock \emph{International Journal of Computer Vision}, 88\penalty0
  (2):\penalty0 303--338, 2010.

\bibitem[Fang et~al.(2013)Fang, Xu, and Rockmore]{fang2013unbiased}
Chen Fang, Ye~Xu, and Daniel~N Rockmore.
\newblock Unbiased metric learning: On the utilization of multiple datasets and
  web images for softening bias.
\newblock In \emph{Proceedings of the IEEE International Conference on Computer
  Vision (ICCV)}, pages 1657--1664, 2013.

\bibitem[Ghifary et~al.(2015)Ghifary, Bastiaan~Kleijn, Zhang, and
  Balduzzi]{ghifary2015domain}
Muhammad Ghifary, W~Bastiaan~Kleijn, Mengjie Zhang, and David Balduzzi.
\newblock Domain generalization for object recognition with multi-task
  autoencoders.
\newblock In \emph{Proceedings of the IEEE International Conference on Computer
  Vision (ICCV)}, pages 2551--2559, 2015.

\bibitem[Ghifary et~al.(2017)Ghifary, Balduzzi, Kleijn, and
  Zhang]{ghifary2017scatter}
Muhammad Ghifary, David Balduzzi, W~Bastiaan Kleijn, and Mengjie Zhang.
\newblock Scatter component analysis: A unified framework for domain adaptation
  and domain generalization.
\newblock \emph{IEEE Transactions on Pattern Analysis and Machine
  Intelligence}, 39\penalty0 (7):\penalty0 1414--1430, 2017.

\bibitem[Girshick et~al.(2014)Girshick, Donahue, Darrell, and
  Malik]{girshick2014rich}
Ross Girshick, Jeff Donahue, Trevor Darrell, and Jitendra Malik.
\newblock Rich feature hierarchies for accurate object detection and semantic
  segmentation.
\newblock In \emph{Proceedings of IEEE Conference on Computer Vision and
  Pattern Recognition (CVPR)}, pages 580--587, 2014.

\bibitem[Gong et~al.(2012)Gong, Shi, Sha, and Grauman]{gong2012geodesic}
Boqing Gong, Yuan Shi, Fei Sha, and Kristen Grauman.
\newblock Geodesic flow kernel for unsupervised domain adaptation.
\newblock In \emph{Proceedings of IEEE Conference on Computer Vision and
  Pattern Recognition (CVPR)}, pages 2066--2073, 2012.

\bibitem[Gong et~al.(2016)Gong, Zhang, Liu, Tao, Glymour, and
  Sch{\"o}lkopf]{pmlr-v48-gong16}
Mingming Gong, Kun Zhang, Tongliang Liu, Dacheng Tao, Clark Glymour, and
  Bernhard Sch{\"o}lkopf.
\newblock Domain adaptation with conditional transferable components.
\newblock In \emph{Proceedings of The 33rd International Conference on Machine
  Learning (ICML 2016)}, pages 2839--2848, 2016.

\bibitem[Gretton et~al.(2007)Gretton, Borgwardt, Rasch, Sch{\"o}lkopf, and
  Smola]{gretton2007kernel}
Arthur Gretton, Karsten~M Borgwardt, Malte Rasch, Bernhard Sch{\"o}lkopf, and
  Alex~J Smola.
\newblock A kernel method for the two-sample-problem.
\newblock In \emph{Advances in Neural Information Processing Systems (NIPS)},
  pages 513--520, 2007.

\bibitem[Griffin et~al.(2007)Griffin, Holub, and Perona]{griffinHolubPerona}
Gregory Griffin, Alex Holub, and Pietro Perona.
\newblock Caltech-256 object category dataset.
\newblock Technical Report 7694, California Institute of Technology, 2007.
\newblock URL \url{http://authors.library.caltech.edu/7694}.

\bibitem[Janzing and Scholkopf(2010)]{janzing2010causal}
Dominik Janzing and Bernhard Scholkopf.
\newblock Causal inference using the algorithmic markov condition.
\newblock \emph{IEEE Transactions on Information Theory}, 56\penalty0
  (10):\penalty0 5168--5194, 2010.

\bibitem[Khosla et~al.(2012)Khosla, Zhou, Malisiewicz, Efros, and
  Torralba]{khosla2012undoing}
Aditya Khosla, Tinghui Zhou, Tomasz Malisiewicz, Alexei~A Efros, and Antonio
  Torralba.
\newblock Undoing the damage of dataset bias.
\newblock In \emph{The European Conference on Computer Vision (ECCV)}, pages
  158--171, 2012.

\bibitem[Krizhevsky et~al.(2012)Krizhevsky, Sutskever, and
  Hinton]{krizhevsky2012imagenet}
Alex Krizhevsky, Ilya Sutskever, and Geoffrey~E Hinton.
\newblock Imagenet classification with deep convolutional neural networks.
\newblock In \emph{Advances in Neural Information Processing Systems (NIPS)},
  pages 1097--1105, 2012.

\bibitem[Li et~al.(2017)Li, Yang, Song, and Hospedales]{8237853}
Da~Li, Yongxin Yang, Yizhe Song, and Timothy~M Hospedales.
\newblock Deeper, broader and artier domain generalization.
\newblock In \emph{Proceedings of the IEEE International Conference on Computer
  Vision (ICCV)}, pages 5543--5551, 2017.

\bibitem[Li et~al.(2018{\natexlab{a}})Li, Pan, Wang, and Kot]{lihl2018domain}
Haoliang Li, Sinno~Jialin Pan, Shiqi Wang, and Alex~C Kot.
\newblock Domain generalization with adversarial feature learning.
\newblock In \emph{Proceedings of IEEE Conference on Computer Vision and
  Pattern Recognition (CVPR)}, pages 5400--5409, 2018{\natexlab{a}}.

\bibitem[Li et~al.(2018{\natexlab{b}})Li, Gong, Tian, Liu, and
  Tao]{AAAI1816595}
Ya~Li, Mingming Gong, Xinmei Tian, Tongliang Liu, and Dacheng Tao.
\newblock Domain generalization via conditional invariant representations.
\newblock In \emph{Proceedings of the Thirty-Second AAAI Conference on
  Artificial Intelligence (AAAI 2018)}, pages 3579--3587, 2018{\natexlab{b}}.

\bibitem[Li et~al.(2018{\natexlab{c}})Li, Tian, Gong, Liu, Liu, Zhang, and
  Tao]{Li_2018_ECCV}
Ya~Li, Xinmei Tian, Mingming Gong, Yajing Liu, Tongliang Liu, Kun Zhang, and
  Dacheng Tao.
\newblock Deep domain generalization via conditional invariant adversarial
  networks.
\newblock In \emph{The European Conference on Computer Vision (ECCV)}, pages
  647--663, 2018{\natexlab{c}}.

\bibitem[Mika et~al.(1999)Mika, Ratsch, Weston, Scholkopf, and
  Mullers]{mika1999fisher}
Sebastian Mika, Gunnar Ratsch, Jason Weston, Bernhard Scholkopf, and
  Klaus-Robert Mullers.
\newblock Fisher discriminant analysis with kernels.
\newblock In \emph{Neural networks for signal processing IX, 1999. Proceedings
  of the 1999 IEEE signal processing society workshop.}, pages 41--48, 1999.

\bibitem[Motiian et~al.(2017)Motiian, Piccirilli, Adjeroh, and
  Doretto]{Motiian_2017_ICCV}
Saeid Motiian, Marco Piccirilli, Donald~A. Adjeroh, and Gianfranco Doretto.
\newblock Unified deep supervised domain adaptation and generalization.
\newblock In \emph{Proceedings of the IEEE International Conference on Computer
  Vision (ICCV)}, pages 5716--5726, 2017.

\bibitem[Muandet et~al.(2013)Muandet, Balduzzi, and
  Sch{\"o}lkopf]{muandet2013domain}
Krikamol Muandet, David Balduzzi, and Bernhard Sch{\"o}lkopf.
\newblock Domain generalization via invariant feature representation.
\newblock In \emph{Proceedings of the 30th International Conference on Machine
  Learning (ICML 2013)}, pages 10--18, 2013.

\bibitem[Patel et~al.(2015)Patel, Gopalan, Li, and Chellappa]{patel2015visual}
Vishal~M Patel, Raghuraman Gopalan, Ruonan Li, and Rama Chellappa.
\newblock Visual domain adaptation: A survey of recent advances.
\newblock \emph{IEEE signal processing magazine}, 32\penalty0 (3):\penalty0
  53--69, 2015.

\bibitem[Russell et~al.(2008)Russell, Torralba, Murphy, and
  Freeman]{russell2008labelme}
Bryan~C Russell, Antonio Torralba, Kevin~P Murphy, and William~T Freeman.
\newblock Labelme: a database and web-based tool for image annotation.
\newblock \emph{International journal of computer vision}, 77\penalty0
  (1-3):\penalty0 157--173, 2008.

\bibitem[Sch{\"o}lkopf and Smola(2001)]{973}
Bernhard Sch{\"o}lkopf and Alexander~J. Smola.
\newblock \emph{Learning with Kernels: Support Vector Machines, Regularization,
  Optimization, and Beyond}.
\newblock MIT Press, Cambridge, MA, USA, 2001.
\newblock ISBN 0262194759.

\bibitem[Sch{\"o}lkopf et~al.(1998)Sch{\"o}lkopf, Smola, and
  M{\"u}ller]{scholkopf1998nonlinear}
Bernhard Sch{\"o}lkopf, Alexander Smola, and Klaus-Robert M{\"u}ller.
\newblock Nonlinear component analysis as a kernel eigenvalue problem.
\newblock \emph{Neural computation}, 10\penalty0 (5):\penalty0 1299--1319,
  1998.

\bibitem[Sch{\"o}lkopf et~al.(2012)Sch{\"o}lkopf, Janzing, Peters, Sgouritsa,
  Zhang, and Mooij.]{ScholkopfJPSZMJ2012}
Bernhard Sch{\"o}lkopf, Dominik Janzing, Jonas Peters, Eleni Sgouritsa, Kun
  Zhang, and Joris Mooij.
\newblock On causal and anticausal learning.
\newblock In \emph{Proceedings of the 29th International Conference on Machine
  Learning (ICML 2012)}, pages 1255--1262, 2012.

\bibitem[Shimodaira(2000)]{shimodaira2000improving}
Hidetoshi Shimodaira.
\newblock Improving predictive inference under covariate shift by weighting the
  log-likelihood function.
\newblock \emph{Journal of statistical planning and inference}, 90\penalty0
  (2):\penalty0 227--244, 2000.

\bibitem[Simonyan and Zisserman(2014)]{DBLP:journals/corr/SimonyanZ14a}
Karen Simonyan and Andrew Zisserman.
\newblock Very deep convolutional networks for large-scale image recognition.
\newblock \emph{CoRR}, abs/1409.1556, 2014.
\newblock URL \url{http://arxiv.org/abs/1409.1556}.

\bibitem[Smola et~al.(2007)Smola, Gretton, Song, and
  Sch{\"o}lkopf]{smola2007hilbert}
Alex Smola, Arthur Gretton, Le~Song, and Bernhard Sch{\"o}lkopf.
\newblock A hilbert space embedding for distributions.
\newblock In \emph{Proceedings of the 18th International Conference on
  Algorithmic Learning Theory}, pages 13--31, 2007.

\bibitem[Sriperumbudur et~al.(2008)Sriperumbudur, Gretton, Fukumizu, Lanckriet,
  Scholkopf, and Zhang]{sriperumbudur2008injective}
Bharath~K Sriperumbudur, Arthur Gretton, Kenji Fukumizu, Gert R~G Lanckriet,
  Bernhard Scholkopf, and R~A Servedio~T Zhang.
\newblock Injective hilbert space embeddings of probability measures.
\newblock In \emph{Proceedings of the 21st Annual Conference on Learning Theory
  (COLT 2008)}, pages 111--122, 2008.

\bibitem[Sriperumbudur et~al.(2010)Sriperumbudur, Gretton, Fukumizu,
  Sch{\"o}lkopf, and Lanckriet]{sriperumbudur2010hilbert}
Bharath~K Sriperumbudur, Arthur Gretton, Kenji Fukumizu, Bernhard
  Sch{\"o}lkopf, and Gert~RG Lanckriet.
\newblock Hilbert space embeddings and metrics on probability measures.
\newblock \emph{Journal of Machine Learning Research}, 11\penalty0
  (Apr):\penalty0 1517--1561, 2010.

\bibitem[Theodoridis and Koutroumbas(2008)]{Theodoridis:2008:PRF:1457541}
Sergios Theodoridis and Konstantinos Koutroumbas.
\newblock \emph{Pattern Recognition, Fourth Edition}.
\newblock Academic Press, Inc., Orlando, FL, USA, 4th edition, 2008.
\newblock ISBN 1597492728, 9781597492720.

\bibitem[Torralba and Efros(2011)]{torralba2011unbiased}
Antonio Torralba and Alexei~A Efros.
\newblock Unbiased look at dataset bias.
\newblock In \emph{Proceedings of IEEE Conference on Computer Vision and
  Pattern Recognition (CVPR)}, pages 1521--1528, 2011.

\bibitem[Xu et~al.(2014)Xu, Li, Niu, and Xu]{xu2014exploiting}
Zheng Xu, Wen Li, Li~Niu, and Dong Xu.
\newblock Exploiting low-rank structure from latent domains for domain
  generalization.
\newblock In \emph{The European Conference on Computer Vision (ECCV)}, pages
  628--643, 2014.

\bibitem[Zhang et~al.(2013)Zhang, Sch{\"o}lkopf, Muandet, and
  Wang]{zhang2013domain}
Kun Zhang, Bernhard Sch{\"o}lkopf, Krikamol Muandet, and Zhikun Wang.
\newblock Domain adaptation under target and conditional shift.
\newblock In \emph{Proceedings of the 30th International Conference on Machine
  Learning (ICML 2013)}, pages 819--827, 2013.

\bibitem[Zhang et~al.(2015)Zhang, Gong, and Sch{\"o}lkopf]{zhang2015multi}
Kun Zhang, Mingming Gong, and Bernhard Sch{\"o}lkopf.
\newblock Multi-source domain adaptation: A causal view.
\newblock In \emph{Proceedings of the Twenty-Ninth AAAI Conference on
  Artificial Intelligence (AAAI 2015)}, pages 3150--3157, 2015.

\end{thebibliography}

\appendix
\newpage
\onecolumn

\section*{Appendix}

\section{Quantities' Property Illustration}
The illustrations comparing average domain discrepancy with multidomain within-class scatter, and average class discrepancy with multidomain between-class scatter are given in Figure \ref{Fig:illus_1} and \ref{Fig:illus_2}.

\begin{figure}[ht]
	\centering
	\includegraphics[width=0.95\linewidth]{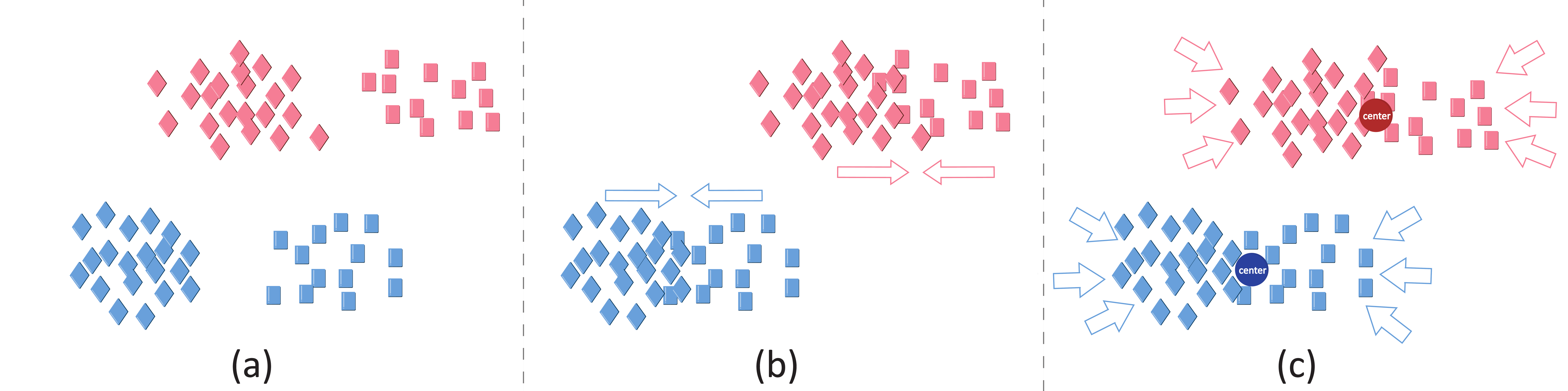}
	\caption{Comparison Between Average Domain Discrepancy and Multidomain Within-class Scatter. Colors denote classes and markers denote domains. (a) The distribution of data in the subspace $\mathbb{R}^{q}$ transformed from RKHS $\mathcal{H}$ using $\mathbf{W}^{0}$. (b) By minimizing average domain discrepancy, the resulting transformation $\mathbf{W}^{add}$ makes the means within each class closer. (c) By minimizing multidomain within-class scatter, the resulting transformation $\mathbf{W}^{mws}$ makes distribution of each class more compact towards the corresponding mean representation.}
    \label{Fig:illus_1}
\end{figure}

\begin{figure}[h]
	\centering
	\includegraphics[width=0.95\linewidth]{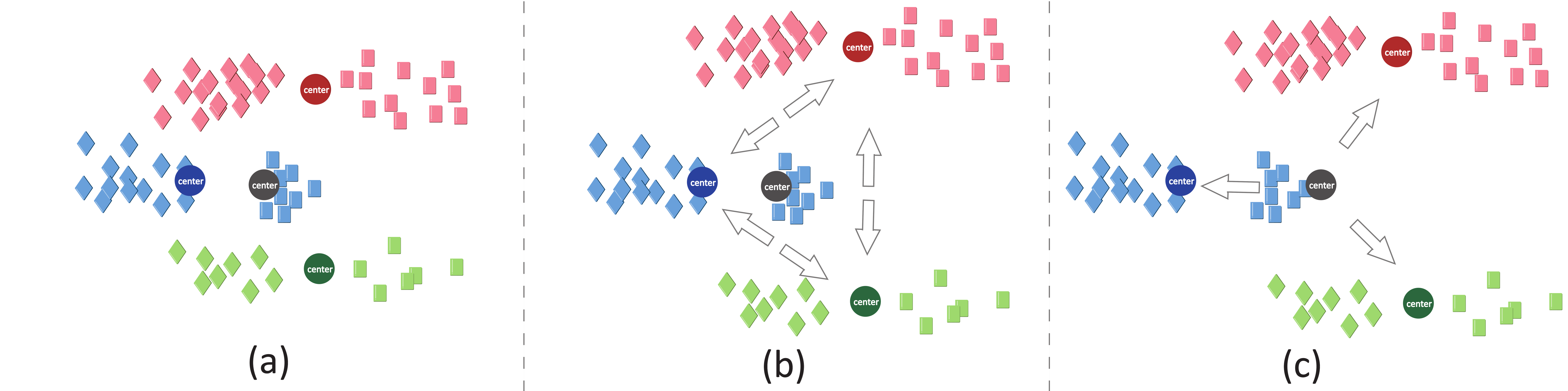}
	\caption{Comparison Between Average Class Discrepancy and Multidomain Between-class Scatter. Colors denote classes and markers denote domains. (a) The distribution of data in the subspace $\mathbb{R}^{q}$ transformed from RKHS $\mathcal{H}$ using $\mathbf{W}^{0}$. (b) By maximizing average class discrepancy, the resulting transformation $\mathbf{W}^{acd}$ treats the distances between each pair of mean representations equally and maximizes them; (c) By maximizing multidomain between-class scatter, the resulting transformation $\mathbf{W}^{mbs}$ maximizes the average distance between the overall mean and the mean representation of different classes. However, each distance is added a weight, which is proportional to the number of instances in the corresponding class. As a result, it is approximate equivalent to the scheme where one pools data of different domains of the same class together and trains classifier.}
    \label{Fig:illus_2}
\end{figure}

\section{Derivation of the Lagrangian}
Since the objective 
\begin{align}
	\argmax_{\mathbf{B}} = \frac{ \tr \left( \mathbf{B}^{T} \left( \beta \mathbf{F} + (1 - \beta) \mathbf{P} \right) \mathbf{B} \right) }{ \tr \left( \mathbf{B}^{T}( \gamma \mathbf{G} + \alpha \mathbf{Q} + \mathbf{K})\mathbf{B} \right) }
	\label{eq:obj_params_supp}
\end{align}
is invariant to re-scaling $\mathbf{B} \to \delta \mathbf{B}$, we rewrite \eqref{eq:obj_params_supp} as a constrained optimization problem:
\begin{align}
\argmax_{\mathbf{B}} & \; \tr \left( \mathbf{B}^{T} \left( \beta  \mathbf{F} + (1 - \beta) \mathbf{P} \right) \mathbf{B} \right) \\
s.t. & \; \tr \left( \mathbf{B}^{T}( \gamma \mathbf{G} + \alpha \mathbf{Q}  + \mathbf{K} )\mathbf{B} \right) = 1,
\end{align}
which yields the Lagrangian
\begin{align}
\mathcal{L} = & \tr \left( \mathbf{B}^{T} \left( \beta  \mathbf{F} + (1 - \beta) \mathbf{P} \right) \mathbf{B} \right) \nonumber \\ 
& - \tr \left( \left( \mathbf{B}^{T}( \gamma \mathbf{G} + \alpha \mathbf{Q}  + \mathbf{K} )\mathbf{B} - \mathbf{I}_{q} \right) \bm{\Gamma} \right),
\label{eq:lagrange}
\end{align}
where $\bm{\Gamma}$ is a diagonal matrix containing the Lagrange multipliers and $\mathbf{I}_{q}$ denotes the identity matrix of dimension $q$. Setting the derivative with respect to $\mathbf{B}$ in the Lagrangian \eqref{eq:lagrange} to zero yields the following generalized eigenvalue problem:
\begin{align}
	\left( \beta \mathbf{F} + (1 - \beta) \mathbf{P} \right) \mathbf{B} = \left( \gamma \mathbf{G} + \alpha \mathbf{Q}  + \mathbf{K} \right) \mathbf{B} \bm{\Gamma}.
	\label{eq:obj_last_supp}
\end{align}

\section{Proof of Theorem 2}
\begin{thrm}
Under assumptions 2 -- 4, and further assuming that $\Vert \hat{f} \Vert_{\mathcal{H}_{ \bar{k} } } \le 1$ and $\Vert f^{*} \Vert_{\mathcal{H}_{ \bar{k} } } \le 1$, where $\hat{f}$ denotes the empirical risk minimizer, $f^{*}$ denotes the expected risk minimizer, then with probability at least $1 - \delta$ there is
\begin{align}
& \mathbb{E}[ \ell( \hat{f}(\tilde{X^{t}} \mathbf{W}), Y^{t} ) ] - \mathbb{E}[  \ell ( f^{*}(\tilde{X^{t}} \mathbf{W}), Y^{t} ) ] \nonumber \\
\le & \ 4 L_{\ell}  L_{ k_{\gamma} } U_{k^{\prime}_{x}} U_{k_{x}} \sqrt{ \frac{ \tr( \mathbf{B}^{T} \mathbf{KB})  }{ n } }  + \sqrt{ \frac{ 2 \log 2 \delta^{-1} } { n } },
\label{eq:bound_1_supp}
\end{align}
where the expectations are taken over the joint distribution of the test domain $\mathbb{P}^{t}(X^{t}, Y^{t})$, $n$ is the number of training samples, and $\mathbf{K} = \bm{\Phi} \bm{\Phi}^{T}$.
\end{thrm}

\begin{proof}
First, we use the following result.

\begin{thrm}[Generalization bound based on Rademacher complexity]\label{thrm:gene_bnd}
Define $\mathcal{A} = \lbrace x \mapsto \ell( f(x), y): f \in \mathcal{H} \rbrace$ to be the loss class, the composition of the loss function with each of the hypotheses. With probability at least $1 - \delta$:
\begin{align}
L(\hat{f}) - L(f^{*}) \le 4 \mathcal{R}_{n}(\mathcal{A}) + \sqrt{ \frac{ 2 \log \frac{2}{\delta} }{ n } },
\end{align}
where $L(\hat{f})$ denotes the expected test risk of the empirical risk minimizer, $L(\hat{f})$ denotes the expected test risk of the expected risk minimizer, $\mathcal{R}_{n}(\mathcal{A})$ denotes the Rademacher complexity of loss class $\mathcal{A}$, and $n$ denotes the number of training points.
\end{thrm}

By applying theorem \ref{thrm:gene_bnd}, with probability at least $1- \delta$ there is
\begin{align}
\mathbb{E}_{\mathbb{P}^{t}_{X}} [ \ell( \hat{f}(\tilde{X^{t}} \mathbf{W}), Y^{t} ) ] - \mathbb{E}_{\mathbb{P}^{t}_{X}} [  \ell ( f^{*}(\tilde{X^{t}} \mathbf{W}), Y^{t} ) ] \le 4 \mathcal{R}_{n}(\mathcal{A}) + \sqrt{ \frac{ 2 \log 2 \delta^{-1} }{ n } },
\label{eq:pf_1}
\end{align}
where $\mathcal{A}$ denotes the loss class $\lbrace x \mapsto \ell(f( \mathbb{P}, x), y): \Vert f \Vert_{\mathcal{H}_{ \bar{k} }} \le 1 \rbrace$, $\mathcal{R}_{n}(\cdot)$ denotes the Rademacher complexity and $n$ is the number of training points.

Since the loss function $\ell$ is $L_{\ell}$-Lipschitz in its first variable, there is
\begin{align}
\mathcal{R}_{n}(\mathcal{A}) = \mathcal{R}_{n}(\ell \circ f) \le L_{\ell} \mathcal{R}_{n}( \mathcal{H}_{\bar{k}} ).
\end{align}

To obtain the Rademacher complexity of $\mathcal{H}_{\bar{k}}$, i.e. $\mathcal{R}_{n}( \mathcal{H}_{\bar{k}} )$, we adopt the following theorem.

\begin{thrm}[Rademacher complexity of $L_{2}$ ball]\label{thrm:l2_ball}
Let $\mathcal{F} = \lbrace z \mapsto \langle w, z \rangle : \Vert w \Vert_{2} \le B_{2} \rbrace$ (bound on weight vectors). Assume $\mathbb{E}_{Z \sim p^{*}} \left[ \Vert Z \Vert^{2}_{2} \right] \le C^{2}_{2}$ (bound on spread of data points). Then
\begin{align}
\mathcal{R}_{n}(\mathcal{F}) \le \frac{B_{2}C_{2}}{\sqrt{n}},
\end{align}
where $n$ denotes the number of training points.
\end{thrm}

According to the function class we restricted, $B_{2}$ in theorem \ref{thrm:l2_ball} in our case is 1. For the bound of feature maps of data in $\mathcal{H}_{\bar{k}}$ (corresponds to $C_{2}$), there is
\begin{align}
& \ \left\Vert \bar{k} \left( \tilde{X^{t}} \mathbf{W}, \cdot \right) \right\Vert \\
= & \ \Vert \gamma_{ k_{\gamma} }\left( \gamma( \mathbb{P}^{t} )\right) \otimes k_{X}(X^{t}, \cdot) \mathbf{W} \Vert \\
\le & \ L_{ k_{\gamma} } \Vert \gamma( \mathbb{P}^{t} ) \Vert \Vert k_{X}(X^{t}, \cdot) \mathbf{W} \Vert \\
\le & \ L_{ k_{\gamma} } U_{k^{\prime}} U_{k} \Vert \mathbf{W} \Vert_{HS}.
\end{align}

Note that $\mathbf{W} = \bm{\Phi}^{T} \mathbf{B}$ and $\mathbf{K} = \bm{\Phi} \bm{\Phi}^{T}$ is invertible. It follows that $\tr( \mathbf{B}^{T} \mathbf{KB} )$ defines a norm consistent with the Hilbert-Schmidt norm $\Vert \mathbf{W} \Vert_{HS}$. Therefore, by applying theorem \ref{thrm:l2_ball}, there is

\begin{align}
\mathcal{R}_{n}(\mathcal{A})  \le L_{\ell} L_{ k_{\gamma} } U_{k^{\prime}} U_{k} \sqrt{ \frac{ \tr( \mathbf{B}^{T} \mathbf{KB} )  }{ n } }.
\end{align}
Combining it with \eqref{eq:pf_1} gives the results.
\end{proof}

\section{Proof of Theorem 3}
\begin{thrm}
Under assumptions 2 -- 4, and assuming that all source sample sets are of the same size, i.e. $n^{s} = \bar{n}$ for $s=1, \dots, m$, then with probability at least $1 - \delta$ there is
\begin{flalign}
    & \sup_{ \Vert f \Vert_{ \mathcal{H}_{ \bar{k} } } \le 1 } \left\vert \frac{1}{m} \sum^{m}_{s=1} \frac{1}{n^{s}} \sum^{n^{s}}_{i=1} \ell \left( f( \hat{\tilde{X}}^{s}_{i} \mathbf{W} ), y^{s}_{i} \right) - \mathcal{E}(f, \infty)  \right\vert \nonumber \\
    \le & U_{\ell} \left( \left( \frac{ \log 2 \delta^{-1} }{ 2m \bar{n} } \right)^{\frac{1}{2}} + \left( \frac{ \log \delta^{-1} }{ 2m } \right)^{\frac{1}{2}} \right) + \sqrt{ \tr(\mathbf{B}^{T} \mathbf{KB} ) } \left( c_{1} \left( \frac{\log 2 \delta^{-1} m }{ \bar{n}} \right)^{\frac{1}{2}}  + c_{2} \left( \left( \frac{1}{m \bar{n}}  \right)^{ \frac{1}{2} } + \left( \frac{1}{m}  \right)^{ \frac{1}{2} } \right) \right) 
\end{flalign} 
where $c_{1} = 2\sqrt{2} L_{\ell} U_{k_{x}} L_{k_{\gamma}} U_{k^{\prime}_{x}}$, $c_{2} = 2 L_{\ell} U_{k_{x}} U_{k_{\gamma}}$.
\end{thrm}

\begin{proof}
Follow the idea in \cite{blanchard2011generalizing}, the supremum of the generalization error bound can be decomposed as
\begin{align}
    & \ \sup_{ \Vert f \Vert_{ \mathcal{H}_{ \bar{k} } } \le 1 } \left\vert \frac{1}{m} \sum^{m}_{s=1} \frac{1}{n^{s}} \sum^{n^{s}}_{i=1} \ell \left( f( \hat{\tilde{X}}^{s}_{i} \mathbf{W} ), y^{s}_{i} \right) - \mathcal{E}(f, \infty)  \right\vert \nonumber \\
    \le & \ \sup_{ \Vert f \Vert_{ \mathcal{H}_{ \bar{k} } } \le 1 } \left\vert \frac{1}{m} \sum^{m}_{s=1} \frac{1}{n^{s}} \sum^{n^{s}}_{i=1} \left( \ell \left( f( \hat{\tilde{X}}^{s}_{i} \mathbf{W} ), y^{s}_{i} \right) - \ell \left( f( \tilde{X}^{s}_{i} \mathbf{W} ), y^{s}_{i} \right) \right) \right\vert \\
    & \ + \sup_{ \Vert f \Vert_{ \mathcal{H}_{ \bar{k} } } \le 1 } \left\vert \frac{1}{m} \sum^{m}_{s=1} \frac{1}{n^{s}} \sum^{n^{s}}_{i=1}  \ell \left( f( \tilde{X}^{s}_{i} \mathbf{W} ), y^{s}_{i} \right) - \mathcal{E}(f, \infty) \right\vert \\
    := & \ (I) + (II),
\end{align}
where $\hat{\tilde{X}}^{s}_{i} = ( \hat{ \mathbb{P} }^{s}, x^{s}_{i} )$, $\tilde{X}^{s}_{i} = ( \mathbb{P}^{s}, x^{s}_{i} )$.

\textbf{Bound of term (I)}

According to the assumption that the loss $\ell$ is $L_{\ell}$ -Lipschitz in its first variable, we have
\begin{align}
    (I) \le & L_{\ell} \sup_{ \Vert f \Vert_{ \mathcal{H}_{ \bar{k} } } \le 1 } \frac{1}{m} \sum^{m}_{s=1} \frac{1}{n^{s}} \sum^{n^{s}}_{i=1} \left\vert f( \hat{\tilde{X}}^{s}_{i} \mathbf{W} ) - f( \tilde{X}^{s}_{i} \mathbf{W} ) \right\vert \\
    \le & L_{\ell} \sup_{ \Vert f \Vert_{ \mathcal{H}_{ \bar{k} } } \le 1 } \frac{1}{m} \sum^{m}_{s=1} \left\Vert  f\left( (\hat{\mathbb{P}}^{s}, \cdot) \mathbf{W} \right) - f\left( (\mathbb{P}^{s}, \cdot) \mathbf{W} \right) \right\Vert_{\infty}
    \label{ineq:1_1}
\end{align}

For any $x \in \mathcal{X}$ and $\Vert f \Vert_{ \mathcal{H}_{ \bar{k} } } \le 1$, using the reproducing property of the kernel $\bar{k}$ and Cauchy-Schwarz inequality, we have
\begin{align}
    \left\vert f\left( (\hat{\mathbb{P}}^{s}, x) \mathbf{W} \right) - f\left( (\mathbb{P}^{s}, x) \mathbf{W} \right) \right\vert & = \left\vert \left\langle \bar{k}\left( (\hat{\mathbb{P}}^{s}, x) \mathbf{W}, \cdot \right) - \bar{k}\left( (\mathbb{P}^{s}, x) \mathbf{W}, \cdot \right), f \right\rangle  \right\vert  \\
    & \le \left\Vert f \right\Vert \left\Vert \bar{k}\left( (\hat{\mathbb{P}}^{s}, x) \mathbf{W}, \cdot \right) - \bar{k}\left( (\mathbb{P}^{s}, x) \mathbf{W}, \cdot \right) \right\Vert 
    \label{ineq:1_5}
\end{align}

According to the assumption, there is $\Vert f \Vert \le 1$. For the second term in \eqref{ineq:1_5} we have
\begin{align}
    & \left\Vert \bar{k}\left( (\hat{\mathbb{P}}^{s}, x) \mathbf{W}, \cdot \right) - \bar{k}\left( (\mathbb{P}^{s}, x) \mathbf{W}, \cdot \right) \right\Vert \\
    = & \left\Vert \gamma_{ k_{\gamma} }\left( \gamma( \hat{\mathbb{P}}^{s} )\right) \otimes k_{X}(x, \cdot) \mathbf{W}  -  \gamma_{ k_{\gamma}}\left( \gamma( \mathbb{P}^{s} )\right) \otimes k_{X}(x, \cdot) \mathbf{W} \right\Vert \\
    \le & \left\Vert \gamma_{ k_{\gamma}}\left( \gamma( \hat{\mathbb{P}}^{s} )\right) \otimes k_{X}(x, \cdot)  -  \gamma_{ k_{\gamma} }\left( \gamma( \mathbb{P}^{s} )\right) \otimes k_{X}(x, \cdot) \right\Vert \left\Vert \mathbf{W} \right\Vert_{HS} \\
    = & \left\Vert \mathbf{W} \right\Vert_{HS} \left( \left\langle  \bar{k}( (\hat{\mathbb{P}}^{s}, x), \cdot ) - \bar{k}( (\mathbb{P}^{s}, x) , \cdot ),  \bar{k}( (\hat{\mathbb{P}}^{s}, x), \cdot ) - \bar{k}( (\mathbb{P}^{s}, x) , \cdot ) \right\rangle \right)^{ \frac{1}{2} } \\
    \le & \left\Vert \mathbf{W} \right\Vert_{HS} k(x,x)^{ \frac{1}{2} } \left( k_{\gamma}( \gamma( \mathbb{P}^{s} ), \gamma( \mathbb{P}^{s} ) ) + k_{\gamma}( \gamma( \hat{\mathbb{P}}^{s} ), \gamma( \hat{\mathbb{P}}^{s} ) ) - 2 k_{\gamma}( \gamma( \mathbb{P}^{s} ), \gamma( \hat{\mathbb{P}}^{s} ) ) \right)^{ \frac{1}{2} } \\
    \le & U_{k} \left\Vert \mathbf{W} \right\Vert_{HS} \left\Vert \gamma_{ k_{\gamma} }\left( \gamma( \mathbb{P}^{s} ) \right) - \gamma_{ k_{\gamma} }\left( \gamma( \hat{\mathbb{P}}^{s} ) \right) \right\Vert \\
    \le & U_{k} L_{ k_{\gamma} } \left\Vert \mathbf{W} \right\Vert_{HS} \left\Vert \gamma( \hat{\mathbb{P}}^{s} ) - \gamma( \mathbb{P}^{s} ) \right\Vert.
    \label{ineq:1_2}
\end{align}

Combining \eqref{ineq:1_5}, \eqref{ineq:1_2} and $\left\Vert f \right\Vert \le 1$, there is
\begin{align}
    \left\vert f\left( (\hat{\mathbb{P}}^{s}, x) \mathbf{W} \right) - f\left( (\mathbb{P}^{s}, x) \mathbf{W} \right) \right\vert \le U_{k} L_{ k_{\gamma} } \left\Vert \mathbf{W} \right\Vert_{HS} \left\Vert \gamma( \hat{\mathbb{P}}^{s} ) - \gamma( \mathbb{P}^{s} ) \right\Vert.
    \label{ineq:1_6}
\end{align}

Now we derive the bound on $ \left\Vert \gamma( \hat{\mathbb{P}}^{s} ) - \gamma( \mathbb{P}^{s} ) \right\Vert $. For independent real zero-mean random variables $x_{1}, \dots, x_{n}$ such that $\vert x_{i} \vert \le C $ for $i = 1, \dots, n$, Hoeffding's inequality states that  $\forall \epsilon > 0$:
\begin{align}
    \text{P}\left[ \left\vert \frac{1}{n} \sum^{n}_{i=1} x_{i} \right\vert > \epsilon \right] \le 2 \exp\left( -\frac{ n\epsilon^{2} }{ 2C^{2} } \right).
\end{align}

Set the $\delta = 2 \exp\left( -\frac{ n\epsilon^{2} }{ 2C^{2} } \right)$, then with probability at least $1- \delta$:
\begin{align}
    \left\vert \frac{1}{n} \sum^{n}_{i=1} x_{i} \right\vert < \sqrt{2}C \sqrt{ \frac{ \log 2 \delta^{-1} }{n} }.
\end{align}

Similar result holds for zero-mean independent random variables $\phi(x_{1}), \dots, \phi(x_{n})$ with values in a separable complex Hilbert space and such that $ \Vert \phi(x_{i}) \Vert \le C $, for $i = 1, \dots, n$:
\begin{align}
    \left\Vert \frac{1}{n} \sum^{n}_{i=1} \phi(x_{i}) \right\Vert < \sqrt{2}C \sqrt{ \frac{ \log 2 \delta^{-1} }{n} }.
\end{align}

For independent uncentered variables $\phi^{\prime}( x_{i} )$ with mean $M$, bounded by $C$. Let $\phi( x_{i} ) = \phi^{\prime}( x_{i} ) - M$ denote the re-centered variables, now bounded at worst by $2C$ by the triangle inequality. Set $\delta = 2 \exp\left( -\frac{ n\epsilon^{2} }{ 8 C^{2} } \right)$,  we obtain with probability at least $1-\delta$ that:
\begin{align}
    \left\Vert \frac{1}{n} \sum^{n}_{i=1} \phi^{\prime}(x_{i}) - M \right\Vert < 2 \sqrt{2}C \sqrt{ \frac{ \log 2 \delta^{-1} }{n} }
    \label{eq:un_hoeff}
\end{align}

Based on the result of \eqref{eq:un_hoeff}, we have 
\begin{align}
    \left\Vert \gamma( \hat{\mathbb{P}}^{s} ) - \gamma( \mathbb{P}^{s} ) \right\Vert = \left\Vert \frac{1}{n^{s}} \sum^{n}_{i=1} \phi^{\prime}(x^{s}_{i}) - \mathbb{E}_{X\sim \mathbb{P}^{s}} [ \phi^{\prime}(X) ] \right\Vert \le 3U_{k^{\prime}} \sqrt{ \frac{ \log 2 \delta^{-1} }{ \bar{n} } }
    \label{ineq:1_3}
\end{align}

Combining \eqref{ineq:1_6} and \eqref{ineq:1_3} we have
\begin{align}
    \sup_{ \Vert f \Vert_{ \mathcal{H}_{ \bar{k} } } \le 1 } \left\Vert  f\left( (\hat{\mathbb{P}}^{s}, \cdot) \mathbf{W} \right) - f\left( (\mathbb{P}^{s}, \cdot) \mathbf{W} \right) \right\Vert_{\infty} \le 2\sqrt{2} U_{k} L_{ k_{\gamma} } U_{k^{\prime}} \left\Vert \mathbf{W} \right\Vert_{HS} \sqrt{ \frac{ \log 2 \delta^{-1} }{ \bar{n} } }
    \label{ineq:1_4}
\end{align}

Conditionally to the draw of $\lbrace \mathbb{P}^{s} \rbrace_{1 \le s \le m}$, we can apply \eqref{ineq:1_4} to each $(\mathbb{P}^{s}, \hat{ \mathbb{P} }^{s} )$ the the union bound over $s = 1, \dots, m$ to get that with probability at least $1 - \delta$:
\begin{align}
        (I) \le 2\sqrt{2} L_{\ell} U_{k} L_{ k_{\gamma} } U_{k^{\prime}} \left\Vert \mathbf{W} \right\Vert_{HS} \sqrt{ \frac{ \log 2 \delta^{-1} + \log m }{ \bar{n} } }
\end{align}

\textbf{Bound of term (II)}

This section follows the idea of \cite{blanchard2011generalizing} so steps of proof that are largely unchanged are omitted. First, we define the conditional (idealized) test error for a given test distribution $\mathbb{P}^{t}_{XY}$ as
\begin{align}
    \mathcal{E}(f, \infty| \mathbb{P}^{t}_{XY}) := \mathbb{E}_{ (X^{t}, Y^{t}) \sim \mathbb{P}^{t}_{XY}}  \left[ \ell \left( f(\tilde{X}^{t} \mathbf{W}), Y^{t} \right) \right],
\end{align}

where $\tilde{X}^{t} = (P^{t}_{X}, X^{t})$.

Then (II) is further decomposed as
\begin{align}
    (II) \le & \frac{1}{m} \sum^{m}_{s=1} \frac{1}{n^{s}} \sum^{n^{s}}_{i=1} \left( \ell \left( f( \tilde{X}^{s}_{i} \mathbf{W} ), y^{s}_{i} \right) - \mathcal{E}(f, \infty \vert \mathbb{P}^{s}_{XY}) \right) + \frac{1}{m} \sum^{m}_{s=1} \left( \mathcal{E}(f, \infty \vert \mathbb{P}^{s}_{XY}) - \mathcal{E}(f, \infty) \right) \\
    := & (IIa) + (IIb)
\end{align}

\textbf{Bound of term (IIa)}

In the case where conditioning on $\lbrace \mathbb{P}^{s}_{XY} \rbrace_{1\le s \le m}$, the observations in $\mathcal{D} = \lbrace (x^{s}_{i}, y^{s}_{i}) \rbrace^{m, n^{s}}_{s=1, i=1}$ are now independent (but not identically distributed) for this conditional distribution. We can thus apply the McDiarmid inequality to the function
\begin{align}
    \zeta \left( \mathcal{D} \right) := \sup_{ \Vert f \Vert_{ \mathcal{H}_{ \bar{k} } } \le 1 } \frac{1}{m} \sum^{m}_{s=1} \frac{1}{n^{s}} \sum^{n^{s}}_{i=1} \left( \ell \left( f( \tilde{X}^{s}_{i} \mathbf{W} ), y^{s}_{i} \right) - \mathcal{E}(f, \infty \vert \mathbb{P}^{s}_{XY}) \right).
\end{align}

When $n^{s} = n^{s^{\prime}} = \bar{n}$ for all $s, s^{\prime}$, that with probability $1- \delta$ over the draw of $\mathcal{D}$, it holds
\begin{align}
    \left\vert \zeta - \mathbb{E}\left[ \zeta \vert \lbrace \mathbb{P}^{s}_{XY} \rbrace_{1\le s \le m} \right] \right\vert \le U_{l} \sqrt{ \frac{ \log{2\delta^{-1}} }{2m \bar{n}} }.
\end{align}

Then by the standard symmetrization technique, $\mathbb{E}\left[ \zeta \vert \lbrace \mathbb{P}^{s}_{XY} \rbrace_{1\le s \le m} \right]$ can be bounded via Rademacher complexity as:
\begin{align}
    \mathbb{E}\left[ \zeta \vert \lbrace \mathbb{P}^{s}_{XY} \rbrace_{1\le s \le m} \right] 
    % = & \mathbb{E}_{(X_{ij}, Y_{ij})} \left[ \frac{1}{N} \sup_{f \in \mathcal{B}_{\bar{k}} (U_{r}) } \sum^{N}_{i=1} \frac{1}{n_{i}}  \sum^{n_{i}}_{j=1} \left( \ell \left( f( \tilde{X}_{ij} W ), Y_{ij} \right) - \mathcal{E}(f, \infty \vert P^{(i)}_{XY}) \right) \vert \lbrace P^{(i)}_{XY} \rbrace_{1\le i \le N} \right] \\
    \le & \frac{2}{m} \mathbb{E}_{(x^{s}_{i}, y^{s}_{i})} \mathbb{E}_{(\epsilon^{s}_{i})} \left[ \sup_{  \Vert f \Vert_{ \mathcal{H}_{ \bar{k} } } \le 1 } \sum^{m}_{s=1} \frac{1}{n^{s}}  \sum^{n^{s}}_{i=1} \epsilon^{s}_{i} \left( \ell \left( f( \tilde{X}^{s}_{i} \mathbf{W} ), y^{s}_{i} \right) \right) \vert \lbrace \mathbb{P}^{s}_{XY} \rbrace_{1\le s \le m} \right] \\
    \le & 2 L_{\ell} U_{k} U_{ k_{\gamma} } \left\Vert \mathbf{W} \right\Vert_{HS} \sqrt{ \frac{1}{m \bar{n}} },
\end{align}
where the last inequality is from the bound of the Rademacher complexity of the loss class $\ell \circ f$.

\textbf{Bound of term (IIb)}

Since the $\lbrace \mathbb{P}^{s}_{XY} \rbrace_{1\le s \le m}$ are i.i.d., the McDiarmid inequality can be applied to the function
\begin{align}
    \xi\left( \lbrace \mathbb{P}^{s}_{XY} \rbrace_{1\le s \le m} \right) := \sup_{ \Vert f \Vert_{ \mathcal{H}_{ \bar{k} } } \le 1 } \frac{1}{m} \sum^{m}_{s=1} \left( \mathcal{E}(f, \infty \vert \mathbb{P}^{s}_{XY}) - \mathcal{E}(f, \infty) \right),
\end{align}

then one obtains that with probability $1- \delta$ over the draw of $\lbrace \mathbb{P}^{s}_{XY} \rbrace_{1\le s \le m}$, it holds
\begin{align}
    \left\vert \xi - \mathbb{E}[\xi] \right\vert \le U_{\ell} \sqrt{ \frac{ \log \delta^{-1} }{ 2m } }.
\end{align}

Similarly, by the standard symmetrization technique, $\mathbb{E}[\xi]$ is bounded as
\begin{align}
    \mathbb{E}[\xi]
    % = & \mathbb{E}_{ \lbrace P^{(i)}_{XY} \rbrace_{1\le i \le N} } \left[ \sup_{f \in \mathcal{B}_{\bar{k}} (U_{r}) } \frac{1}{N} \sum^{N}_{i=1} \mathbb{E}_{ (X,Y)\sim P^{(i)}_{XY} } [ \ell( f(\tilde{X}W), Y ) ] - \mathbb{E}_{P_{XY}\sim \mu} \mathbb{E}_{(X,Y)\sim P_{XY}} [ \ell( f(\tilde{X}W), Y ) ] \right]  \\
    % \le & \frac{2}{N} \mathbb{E}_{ \lbrace P^{(i)}_{XY} \rbrace_{1\le i \le N} } \mathbb{E}_{ (\epsilon_{i})_{ 1\le i \le N } } \left[ \sup_{f \in \mathcal{B}_{\bar{k}} (U_{r}) } \sum^{N}_{i=1} \epsilon_{i} \mathbb{E}_{ (X_{i}, Y_{i}) \sim P^{(i)}_{XY} } [\ell ( f( \tilde{X}_{ij} W ), Y_{ij} ) ]  \right] \\
    \le & \frac{2}{m} \mathbb{E}_{ \lbrace \mathbb{P}^{s}_{XY} \rbrace_{1\le s \le m} } \mathbb{E}_{ (X^{s}, Y^{s})_{1 \le s \le m} } \mathbb{E}_{ (\epsilon^{s})_{1 \le s \le m} } \left[ \sup_{ \Vert f \Vert_{ \mathcal{H}_{ \bar{k} } } \le 1 } \sum^{m}_{s=1} \epsilon^{s}  \ell \left( f( \tilde{X}^{s} \mathbf{W} ), Y^{s} \right)  \right] \\
    \le & 2 L_{\ell} U_{k} U_{k_{\gamma}} \left\Vert \mathbf{W} \right\Vert_{HS} \sqrt{ \frac{ 1 }{ m } },
\end{align}
where the last inequality is again from the bound of the Rademacher complexity of the loss class $\ell \circ f$.

Finally, $\mathbf{W} = \bm{\Phi}^{T} \mathbf{B}$ and $\mathbf{K} = \bm{\Phi} \bm{\Phi}^{T}$ is invertible. It follows that $\tr( \mathbf{B}^{T} \mathbf{KB} )$ defines a norm consistent with the Hilbert-Schmidt norm $\Vert \mathbf{W} \Vert_{HS}$. By combining the above results we obtain the announced result.
\end{proof}

\section{Experimental Configurations}
Due to the difference in techniques adopted in different methods, there is/are different hyper-parameter(s) in each method require tuning in the experiments.

\begin{itemize}
    \item 1NN: since there is no hyper-parameter to be determined in 1NN, instances in source domains are directly combined for training. Then we apply the trained model on target domains and report the test accuracy.
    \item SVM: the regularization coefficient $C$ requires tuning in SVM. $C \in \lbrace 0.1, 0.5, 1.0, 2.0, 5.0, 10.0 \rbrace$ are validated in the experiments.
    \item KPCA and KFD: the kernel width $\sigma_{k}$ requires tuning. $\sigma_{k} \in \lbrace 0.1d_{M}, 0.2d_{M}, 0.5d_{M}, d_{M}, 2d_{M}, 5d_{M} \rbrace$, where $d_{M} = \text{median}\left( \Vert \bm{x}_{i} - \bm{x}_{j} \Vert^{2}_{2}  \right), \forall \bm{x}_{i}, \bm{x}_{j} \in \mathcal{D}$, are validated.
    \item E-SVM: four hyper-parameters ($\lambda_{1}, \lambda_{2}, C_{1}, C_{2}$) require tuning. $\lambda_{1} \in \lbrace 0.1, 1, 10 \rbrace$, $\lambda_{2} \in \lbrace 0.5\lambda_{1}, 1\lambda_{1}, 2\lambda_{1} \rbrace$, and $C_{1}, C_{2} \in \lbrace 0.1, 1, 10 \rbrace$ are validated.
    \item CCSA: two hyper-parameters ($lr, \alpha$) require tuning. learning rate $lr \in \lbrace 0.5, 1.0, 1.5 \rbrace$ and $\alpha \in \lbrace 0.1, 0.25, 0.4 \rbrace$ are validated.
    \item DICA: Two parameters ($\lambda$, $\epsilon$) require tuning. $\lambda \in \lbrace 1\mathrm{e}{-3}, 1\mathrm{e}{-2}, 1\mathrm{e}{-1}, 1.0, 1\mathrm{e}{1}, 1\mathrm{e}{2}, 1\mathrm{e}{3} \rbrace$ \\ and $\epsilon \in \lbrace 1\mathrm{e}{-3}, 1\mathrm{e}{-2}, 1\mathrm{e}{-1}, 1.0, 1\mathrm{e}{1}, 1\mathrm{e}{2}, 1\mathrm{e}{3} \rbrace$ were validated.
    \item SCA: Two parameters ($\beta, \delta$) require tuning. $\beta \in \lbrace 0.1, 0.3, 0.5, 0.7, 0.9 \rbrace$, \\ $\delta \in \lbrace 1\mathrm{e}{-3}, 1\mathrm{e}{-2}, 1\mathrm{e}{-1}, 1.0, 1\mathrm{e}{1}, 1\mathrm{e}{2}, 1\mathrm{e}{3}, 1\mathrm{e}{4}, 1\mathrm{e}{5}, 1\mathrm{e}{6} \rbrace$ were validated.
    
    \item CIDG: Three hyper-parameters ($\beta, \alpha, \gamma$) require tuning. $\beta \in \lbrace 0.1, 0.3, 0.5, 0.7, 0.9 \rbrace$, \\ $\gamma \in \lbrace 1\mathrm{e}{-3}, 1\mathrm{e}{-2}, 1\mathrm{e}{-1}, 1, 1\mathrm{e}{1}, 1\mathrm{e}{2}, 1\mathrm{e}{3}, 1\mathrm{e}{4}, 1\mathrm{e}{5}, 1\mathrm{e}{6} \rbrace$, and \\ $\alpha \in \lbrace 1, 1\mathrm{e}{1}, 1\mathrm{e}{2}, 1\mathrm{e}{3}, 1\mathrm{e}{4}, 1\mathrm{e}{5}, 1\mathrm{e}{6}, 1\mathrm{e}{7}, 1\mathrm{e}{8}, 1\mathrm{e}{9} \rbrace$, were validated.
    
    \item MDA: Three hyper-parameters ($\beta, \alpha, \gamma$) require tuning. $\beta \in \lbrace 0.1, 0.3, 0.5, 0.7, 0.9 \rbrace$, \\ $\gamma \in \lbrace 1\mathrm{e}{-3}, 1\mathrm{e}{-2}, 1\mathrm{e}{-1}, 1.0, 1\mathrm{e}{1}, 1\mathrm{e}{2}, 1\mathrm{e}{3}, 1\mathrm{e}{4}, 1\mathrm{e}{5}, 1\mathrm{e}{6} \rbrace$, and \\ $\alpha \in \lbrace 1, 1\mathrm{e}{1}, 1\mathrm{e}{2}, 1\mathrm{e}{3}, 1\mathrm{e}{4}, 1\mathrm{e}{5}, 1\mathrm{e}{6}, 1\mathrm{e}{7}, 1\mathrm{e}{8}, 1\mathrm{e}{9} \rbrace$, were validated.
\end{itemize}

For feature extraction methods (i.e., KPCA and KFD) and kernel-based DG methods (i.e., DICA, SCA, CIDG, and MDA), %that rely on finding $q$-dimensional features from the raw data to tackle data set bias across domains, features of 2-dimensional are extracted in synthetic experiments. 
in real data experiments, different number of leading eigenvectors (corresponds to the dimension of the transformed subspace) that contribute to certain proportions (i.e. $ \lbrace 0.2, 0.4, 0.6, 0.8, 0.92,$ $ 0.94, 0.96, 0.98 \rbrace$) of the sum of all eigenvalues are tested and the highest accuracies are reported for each method.

\section{Synthetic Experimental Results Visualization}
In this section, we show the data distribution of source domains in the transformed domain-invariant subspace $\mathbb{R}^{q}$ of the synthetic experiment for kernel-based DG methods: SCA, CIDG, MDA, which are proposed for classification problems. The results are given in Figure \ref{fig:syn_re_1} and \ref{fig:syn_re_2}.

We observe from the results that: 1) the transformation learned from MDA performs the best in terms of the separation of different classes of target domains; 2) the overlapped region in source domains(green and red classes) is handled slightly better in MDA than in CIDG; 3) SCA has difficulty in separating instances of different classes in part of the cases.

\begin{figure}[b]
	\begin{subfigure}{.24\linewidth}
		\centering
		\includegraphics[width=\linewidth]{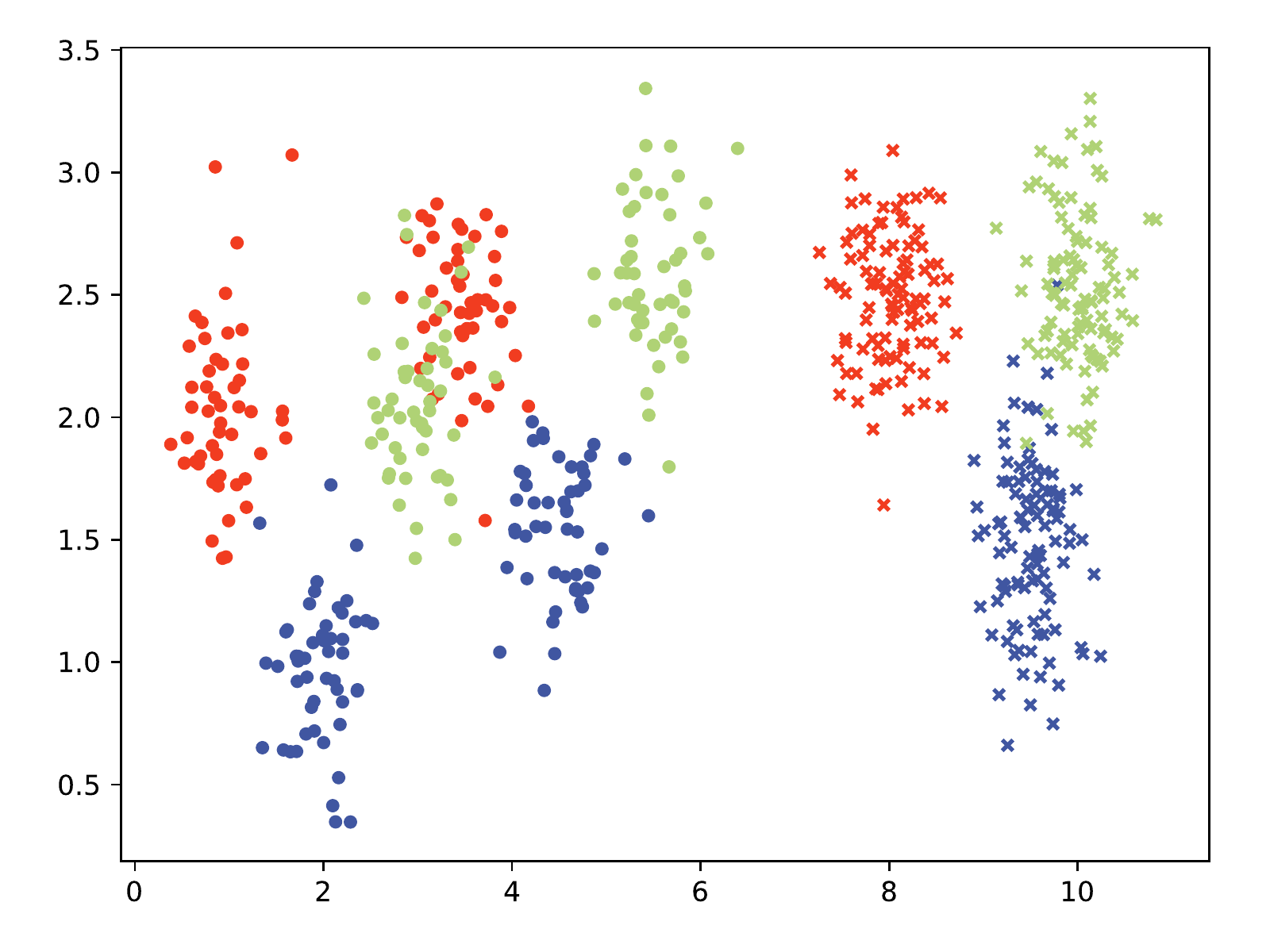}
		\caption{(2a, 2a) Raw data}
		\label{Fig:1_1}
	\end{subfigure}
	\begin{subfigure}{.24\linewidth}
		\centering
		\includegraphics[width=\linewidth]{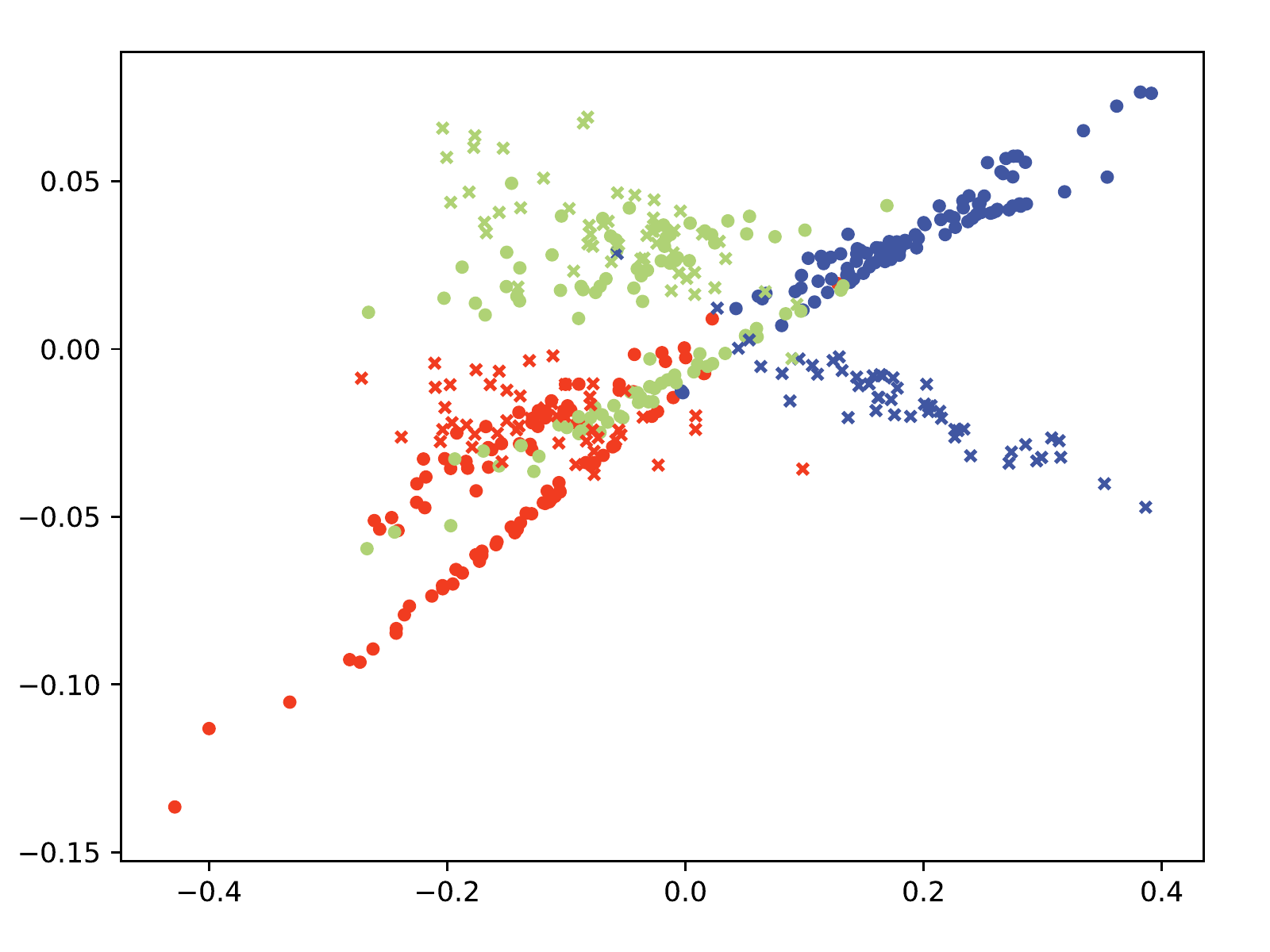}
		\caption{SCA(79.33\%)}
		\label{Fig:1_2}
	\end{subfigure} 
	\begin{subfigure}{.24\linewidth}
		\centering
		\includegraphics[width=\linewidth]{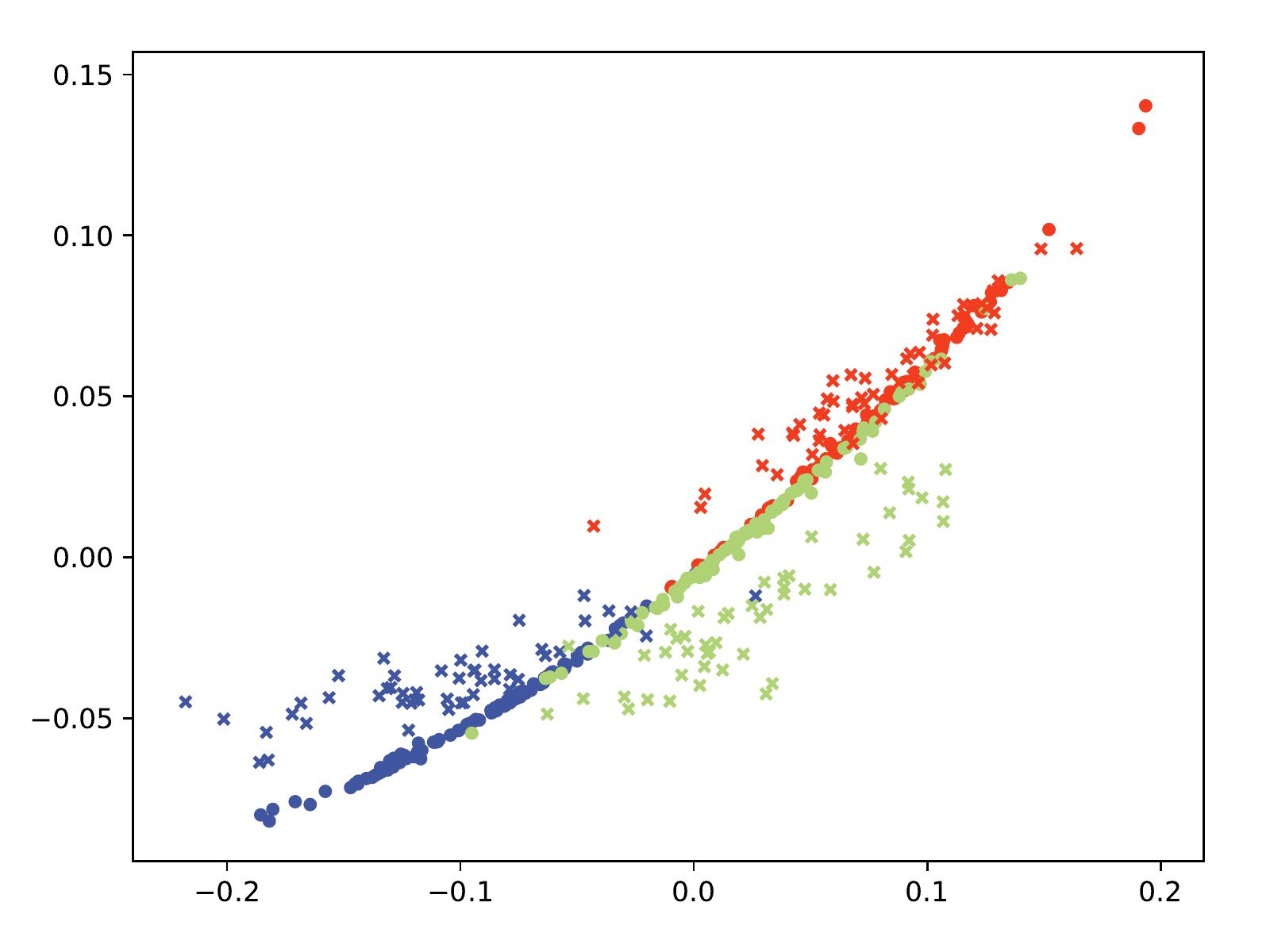}
		\caption{CIDG(90.67\%)}
		\label{Fig:1_3}
	\end{subfigure}
	\begin{subfigure}{.24\linewidth}
		\centering
		\includegraphics[width=\linewidth]{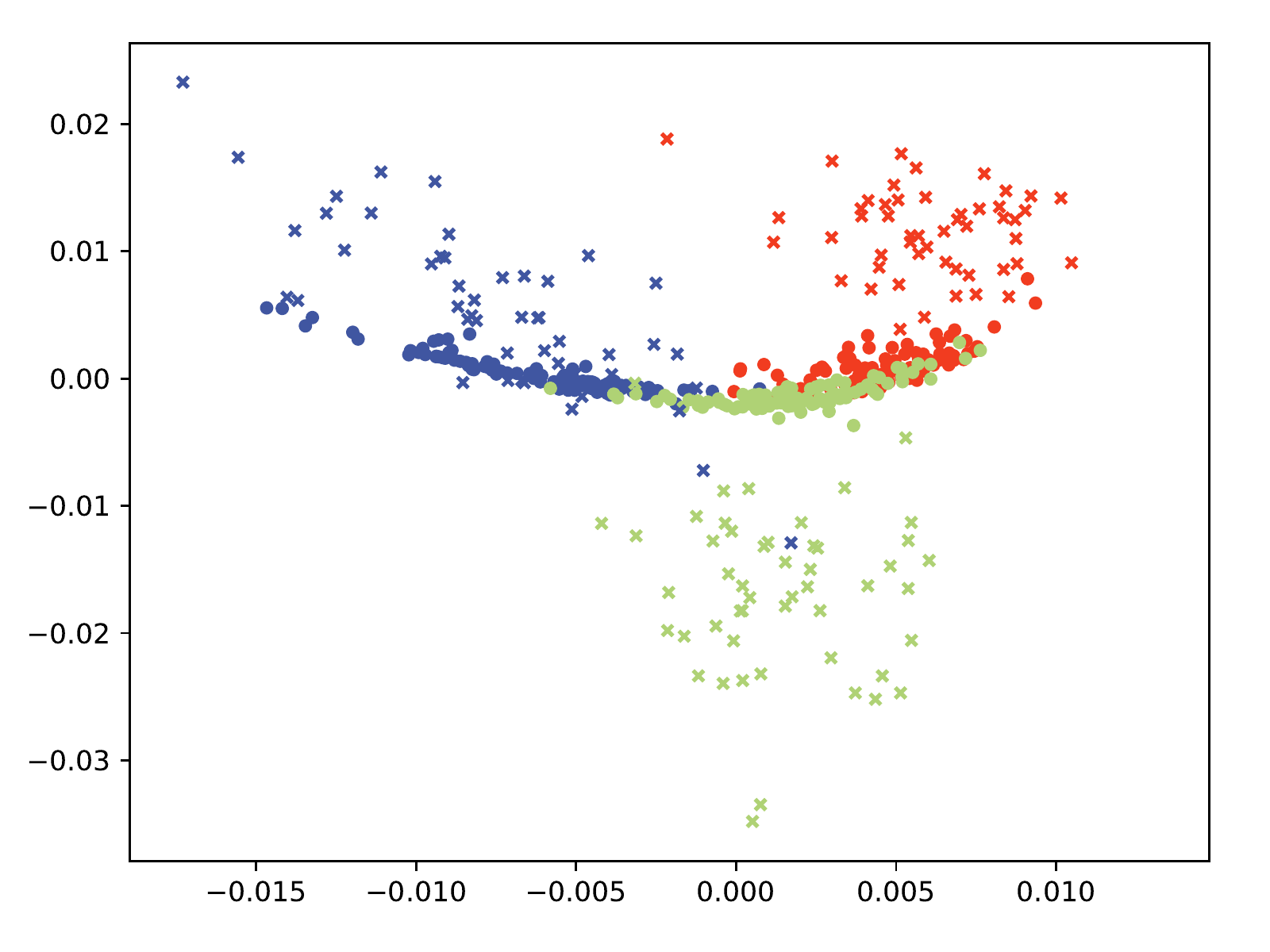}
		\caption{MDA(96.67\%)}
		\label{Fig:1_4}
	\end{subfigure} \\
	
	\begin{subfigure}{.24\linewidth}
		\centering
		\includegraphics[width=\linewidth]{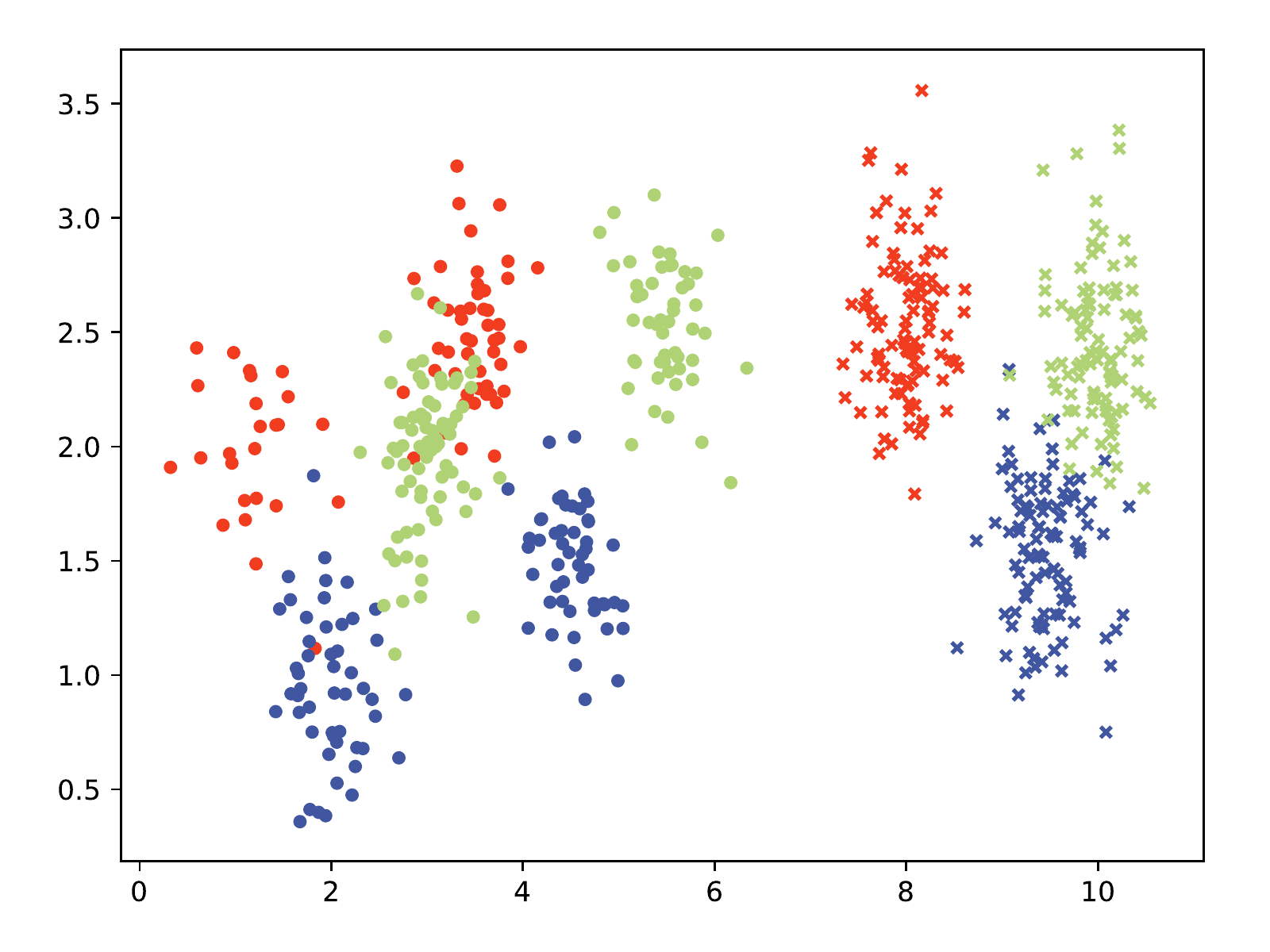}
		\caption{(2b, 2a) Raw data}
		\label{Fig:2_1}
	\end{subfigure}
	\begin{subfigure}{.24\linewidth}
		\centering
		\includegraphics[width=\linewidth]{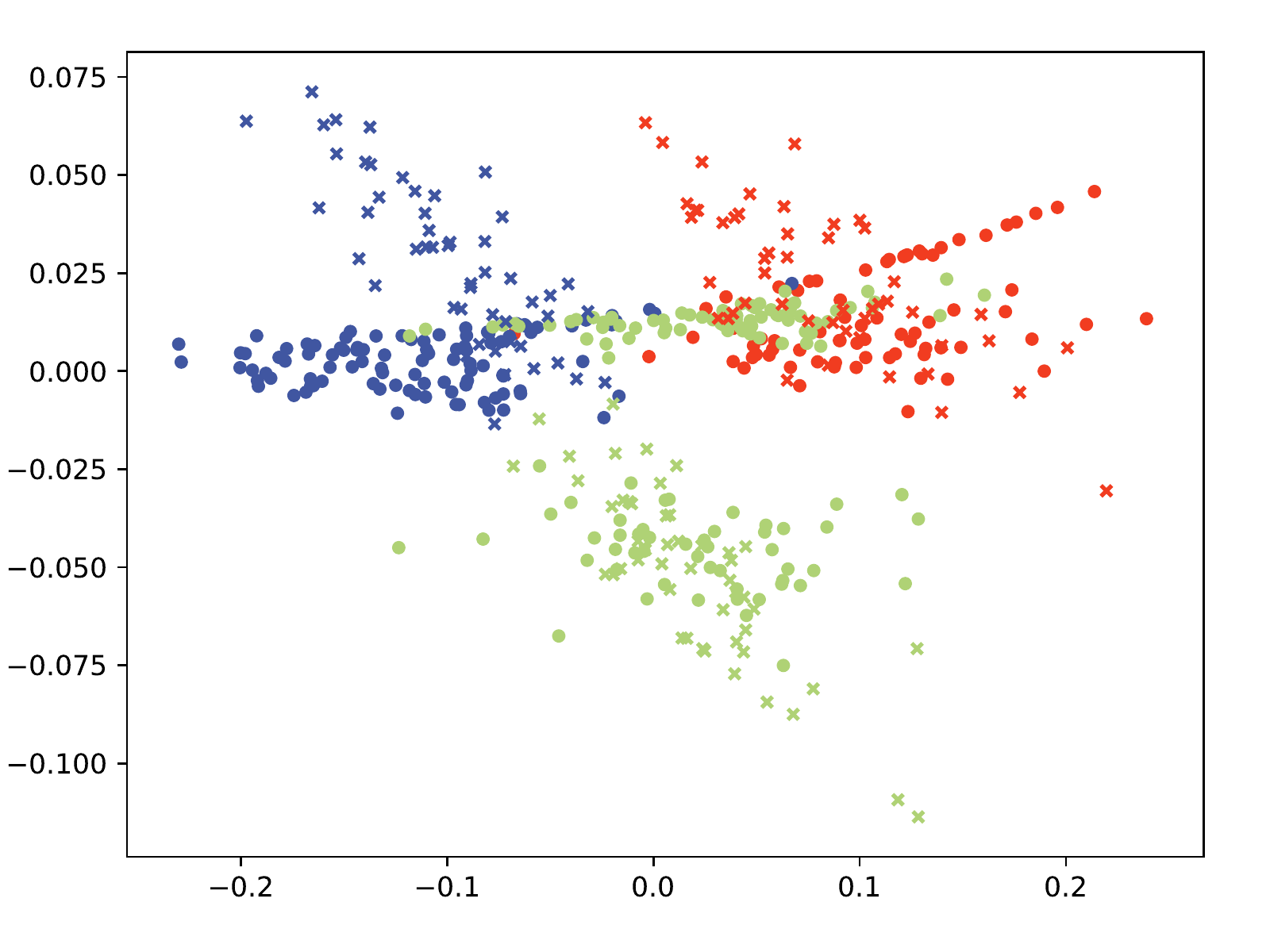}
		\caption{SCA(72.00\%)}
		\label{Fig:2_2}
	\end{subfigure}
	\begin{subfigure}{.24\linewidth}
		\centering
		\includegraphics[width=\linewidth]{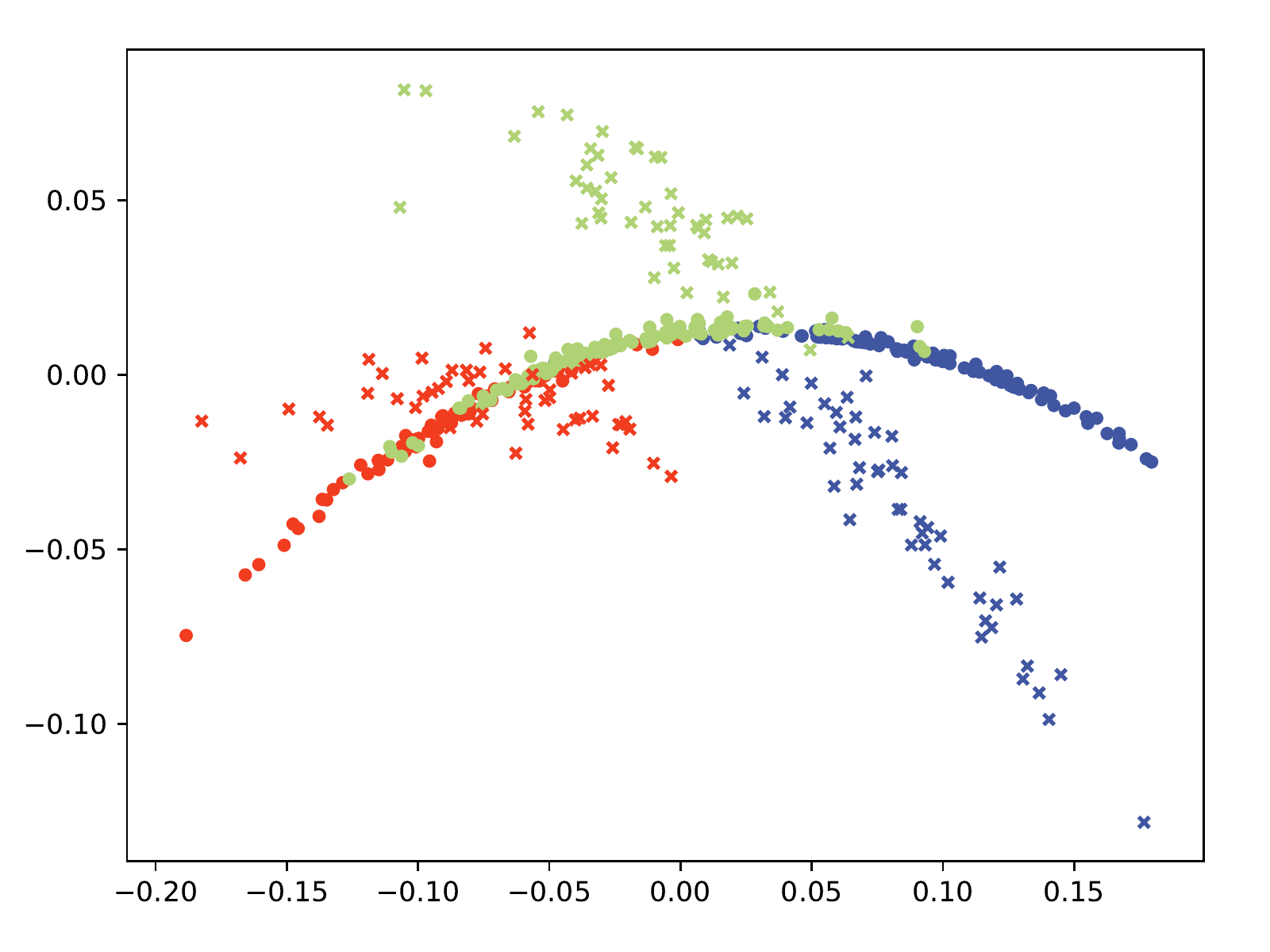}
		\caption{CIDG(87.33\%)}
		\label{Fig:2_3}
	\end{subfigure}
	\begin{subfigure}{.24\linewidth}
		\centering
		\includegraphics[width=\linewidth]{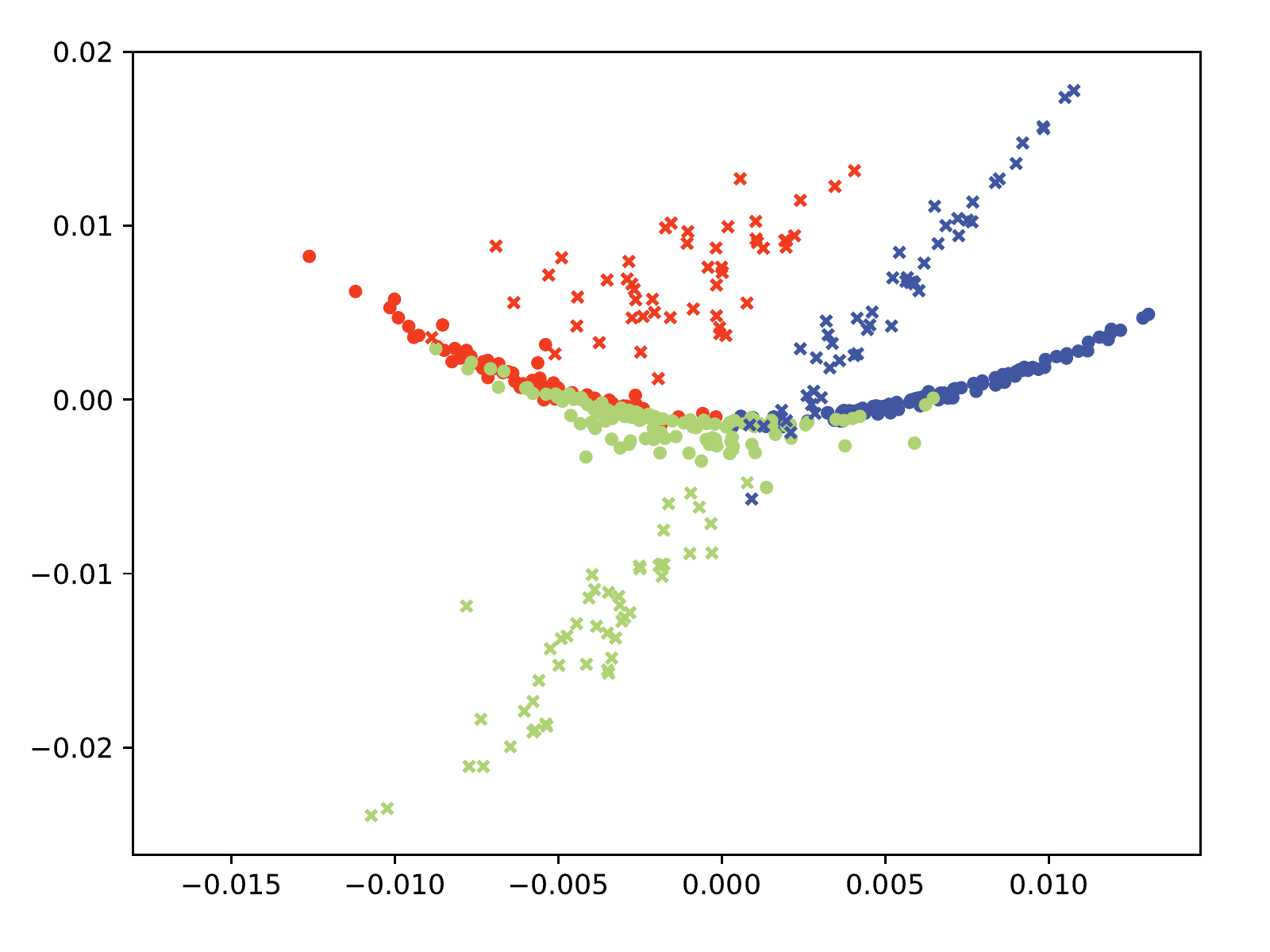}
		\caption{MDA(96.00\%)}
		\label{Fig:2_4}
	\end{subfigure} \\

	\begin{subfigure}{.24\linewidth}
		\centering
		\includegraphics[width=\linewidth]{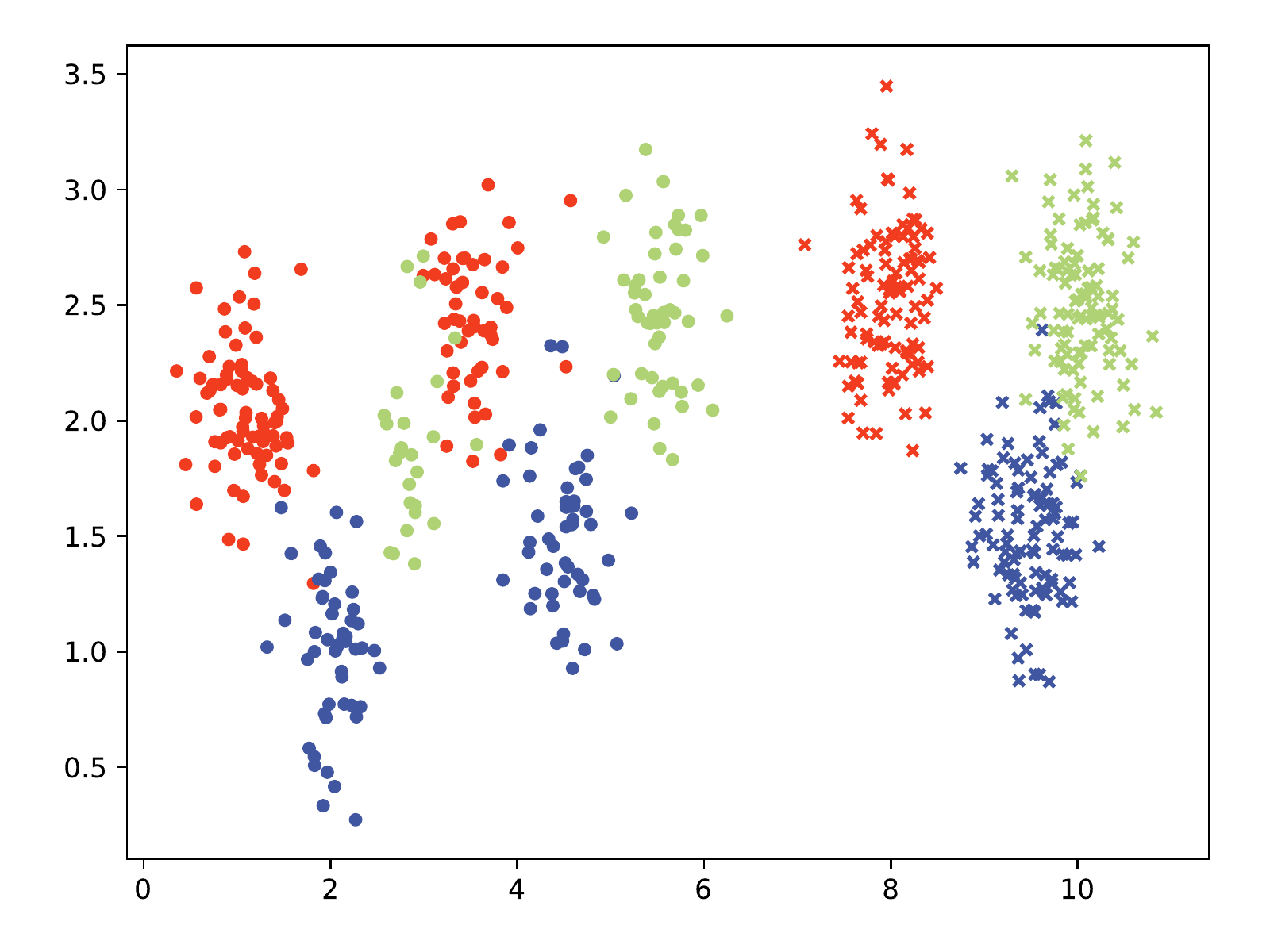}
		\caption{(2c, 2a) Raw data}
		\label{Fig:3_1}
	\end{subfigure}
	\begin{subfigure}{.24\linewidth}
		\centering
		\includegraphics[width=\linewidth]{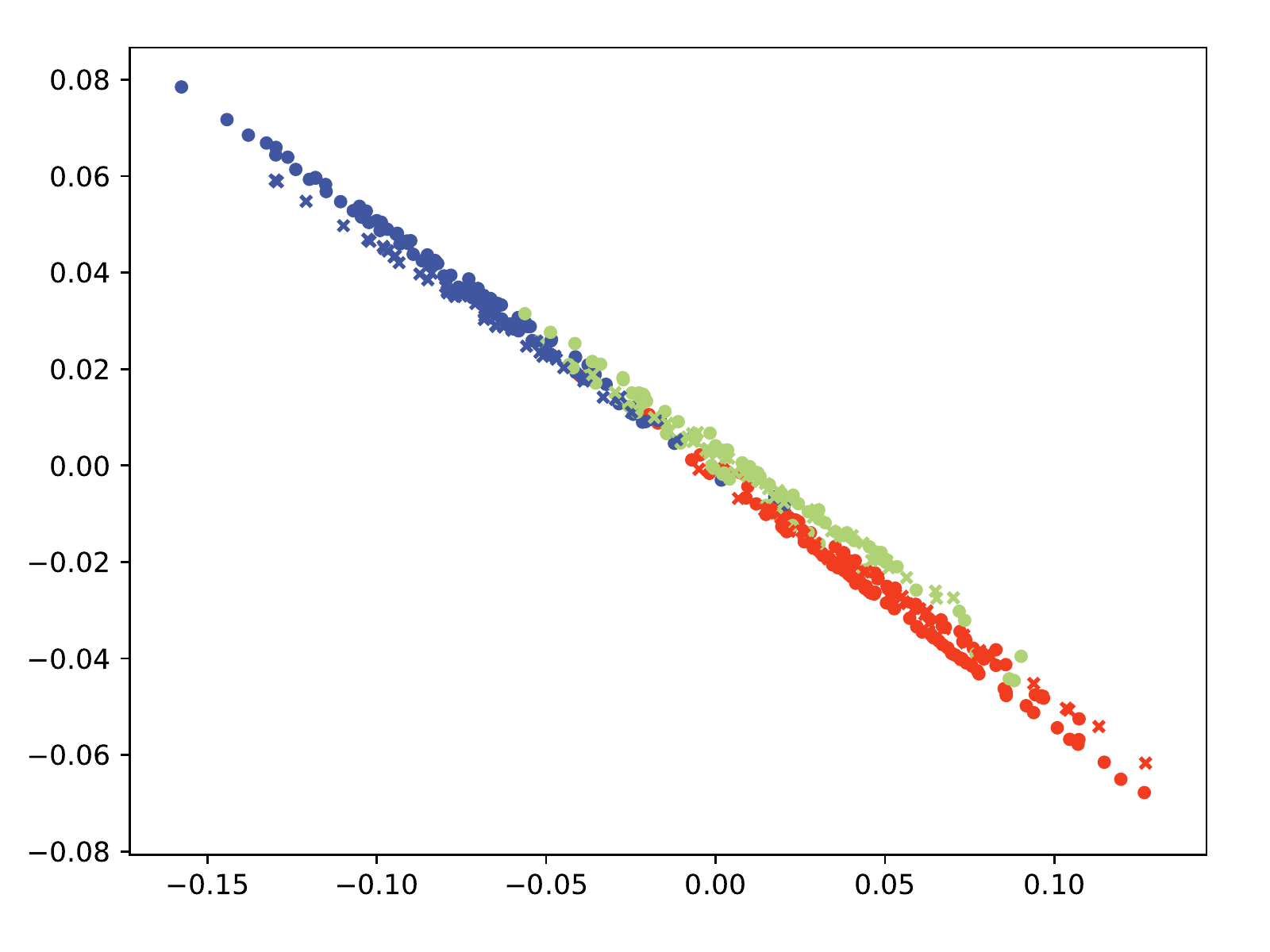}
		\caption{SCA(84.67\%)}
		\label{Fig:3_2}
	\end{subfigure}
	\begin{subfigure}{.24\linewidth}
		\centering
		\includegraphics[width=\linewidth]{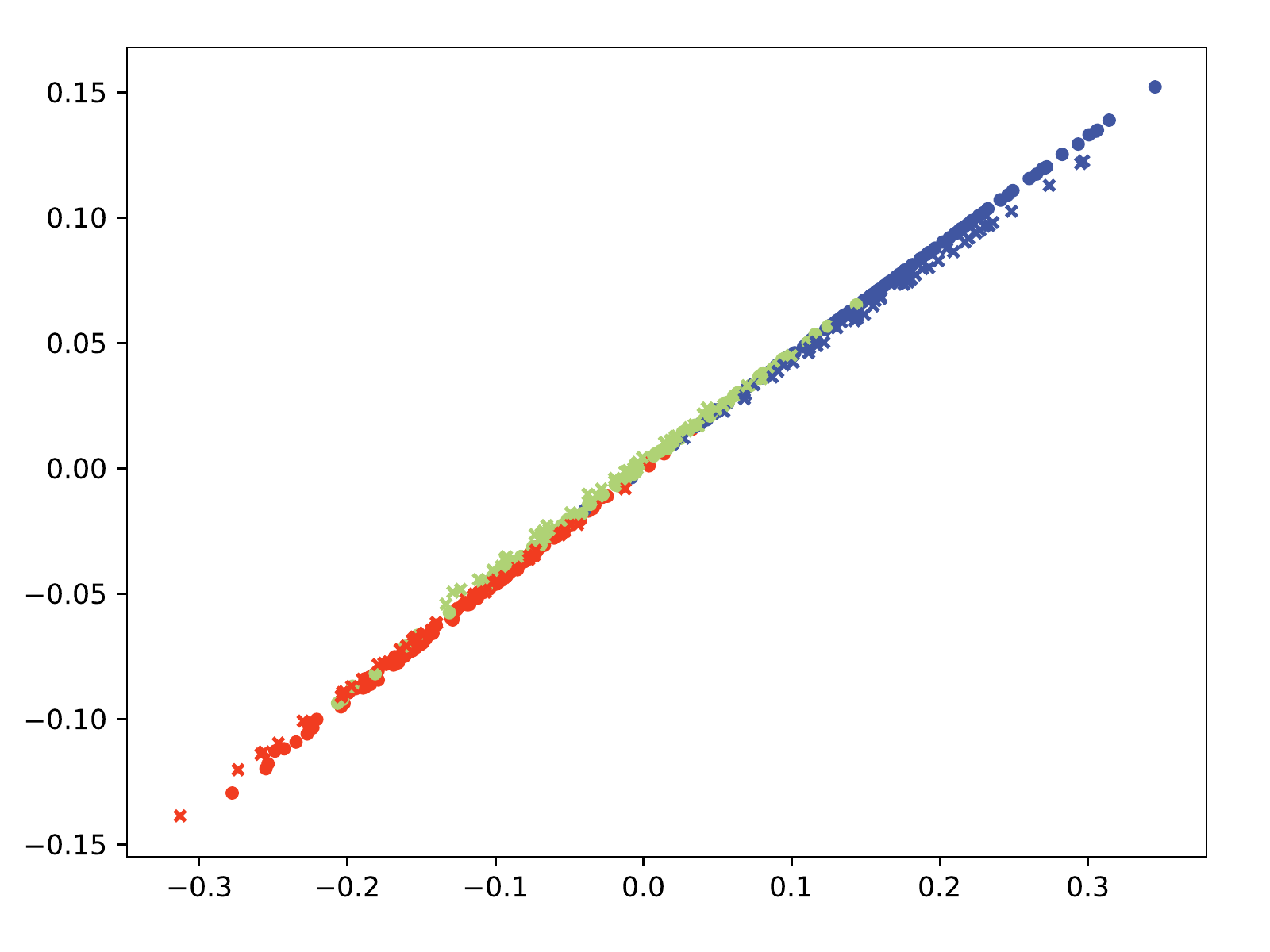}
		\caption{CIDG(74.67\%)}
		\label{Fig:3_3}
	\end{subfigure}
	\begin{subfigure}{.24\linewidth}
		\centering
		\includegraphics[width=\linewidth]{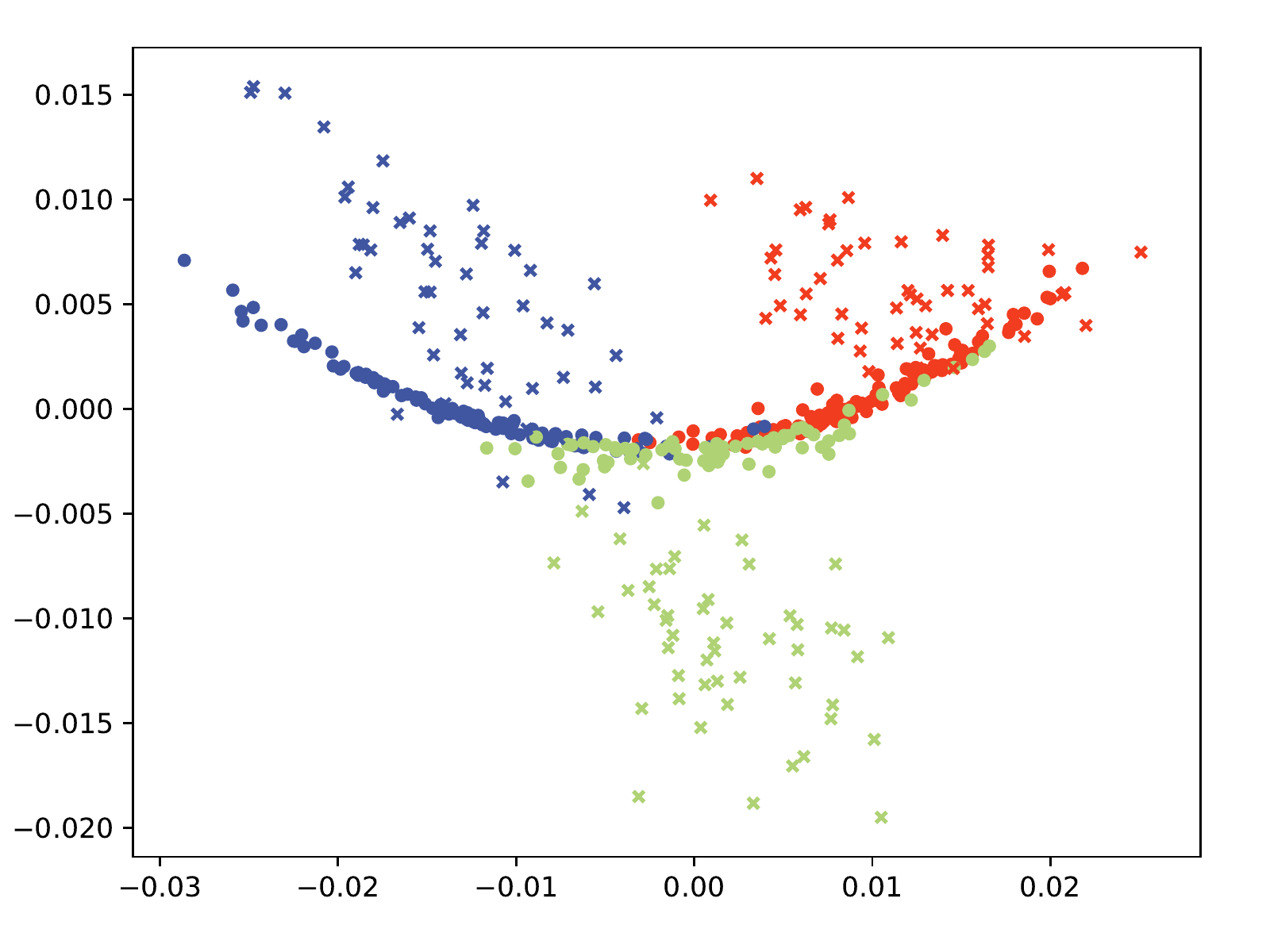}
		\caption{MDA(97.33\%)}
		\label{Fig:3_4}
	\end{subfigure} \\
	
	\begin{subfigure}{.24\linewidth}
		\centering
		\includegraphics[width=\linewidth]{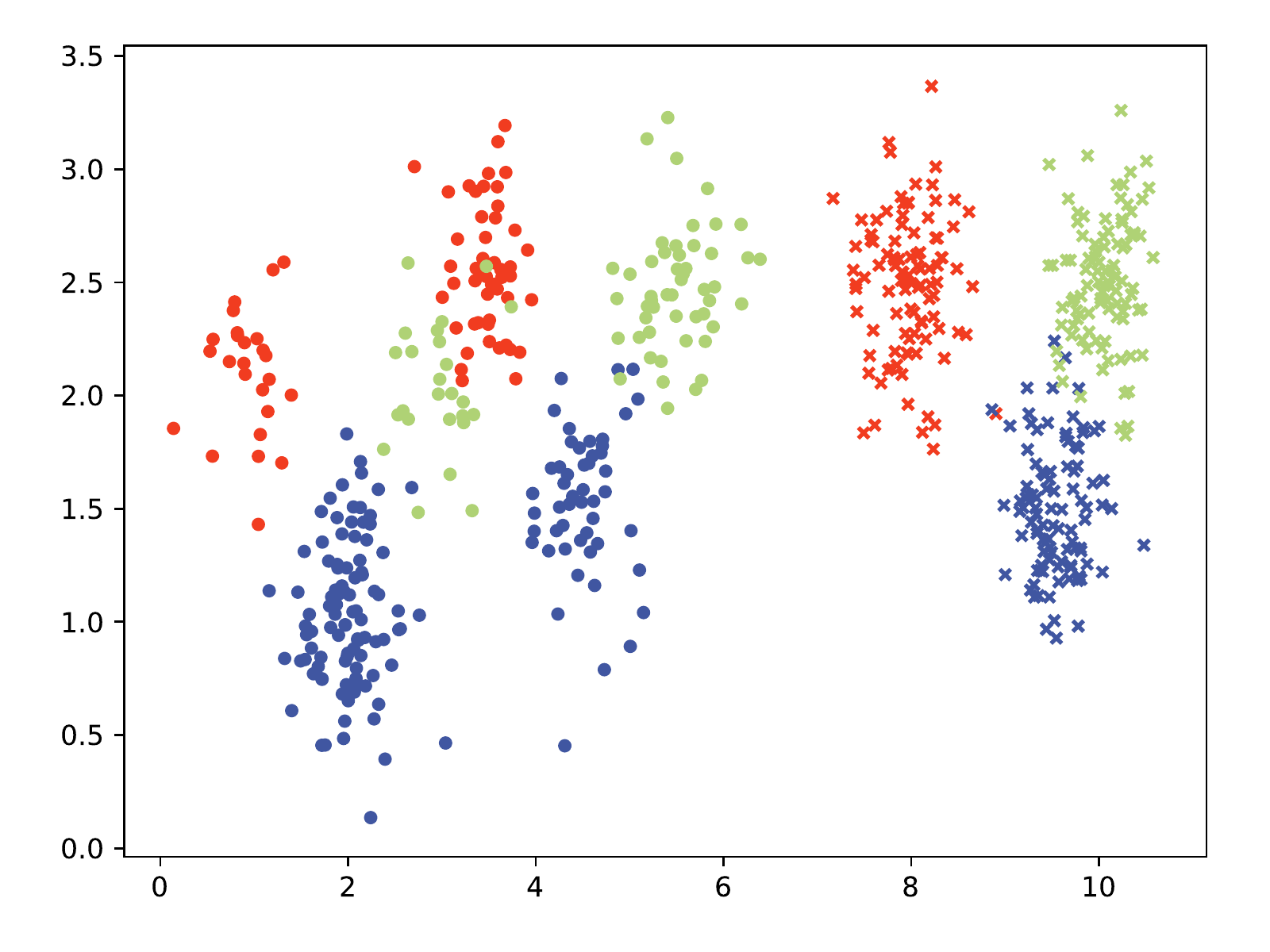}
		\caption{(2d, 2a) Raw data}
		\label{Fig:4_1}
	\end{subfigure}
	\begin{subfigure}{.24\linewidth}
		\centering
		\includegraphics[width=\linewidth]{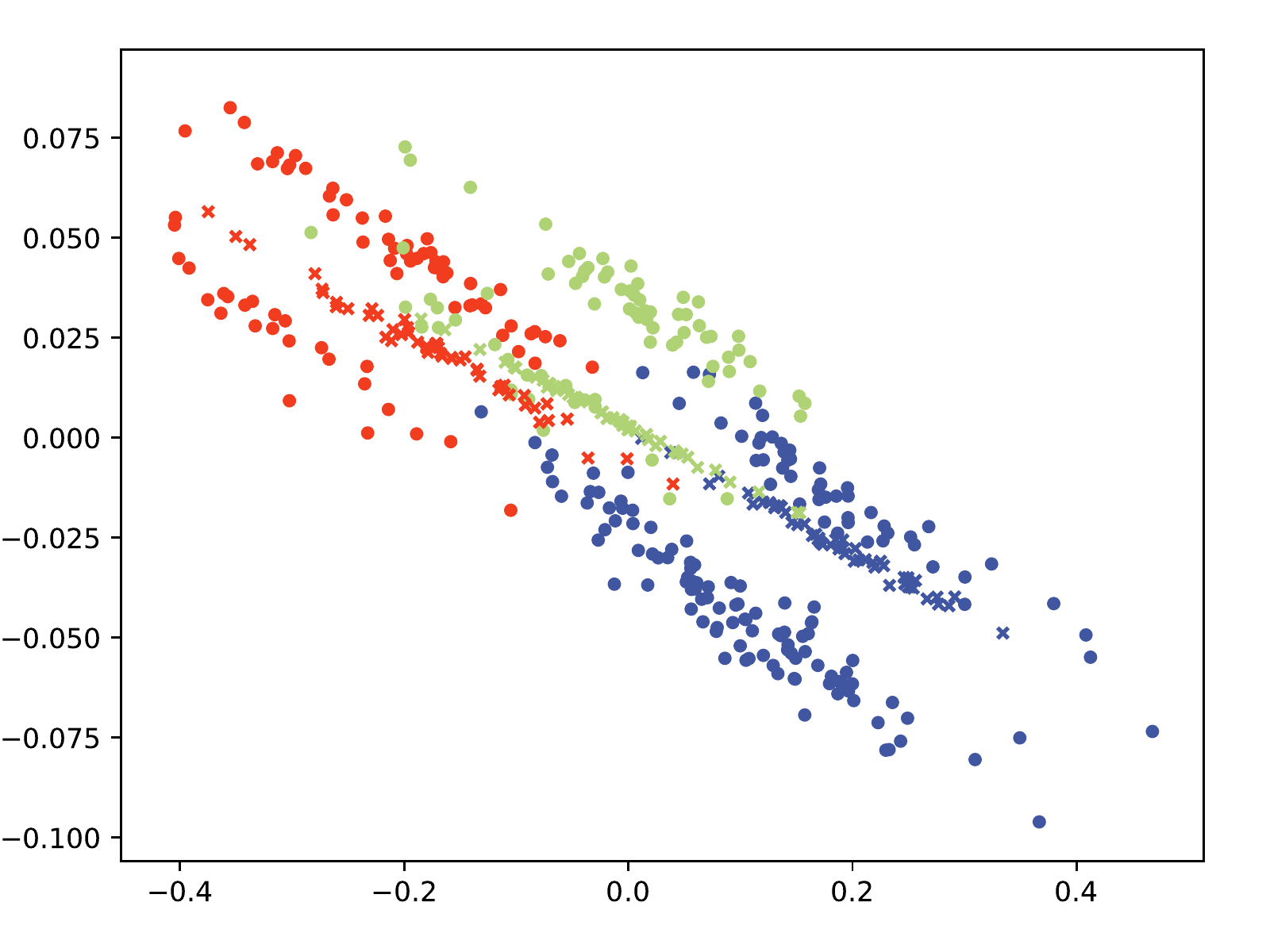}
		\caption{SCA(57.33\%)}
		\label{Fig:4_2}
	\end{subfigure}
	\begin{subfigure}{.24\linewidth}
		\centering
		\includegraphics[width=\linewidth]{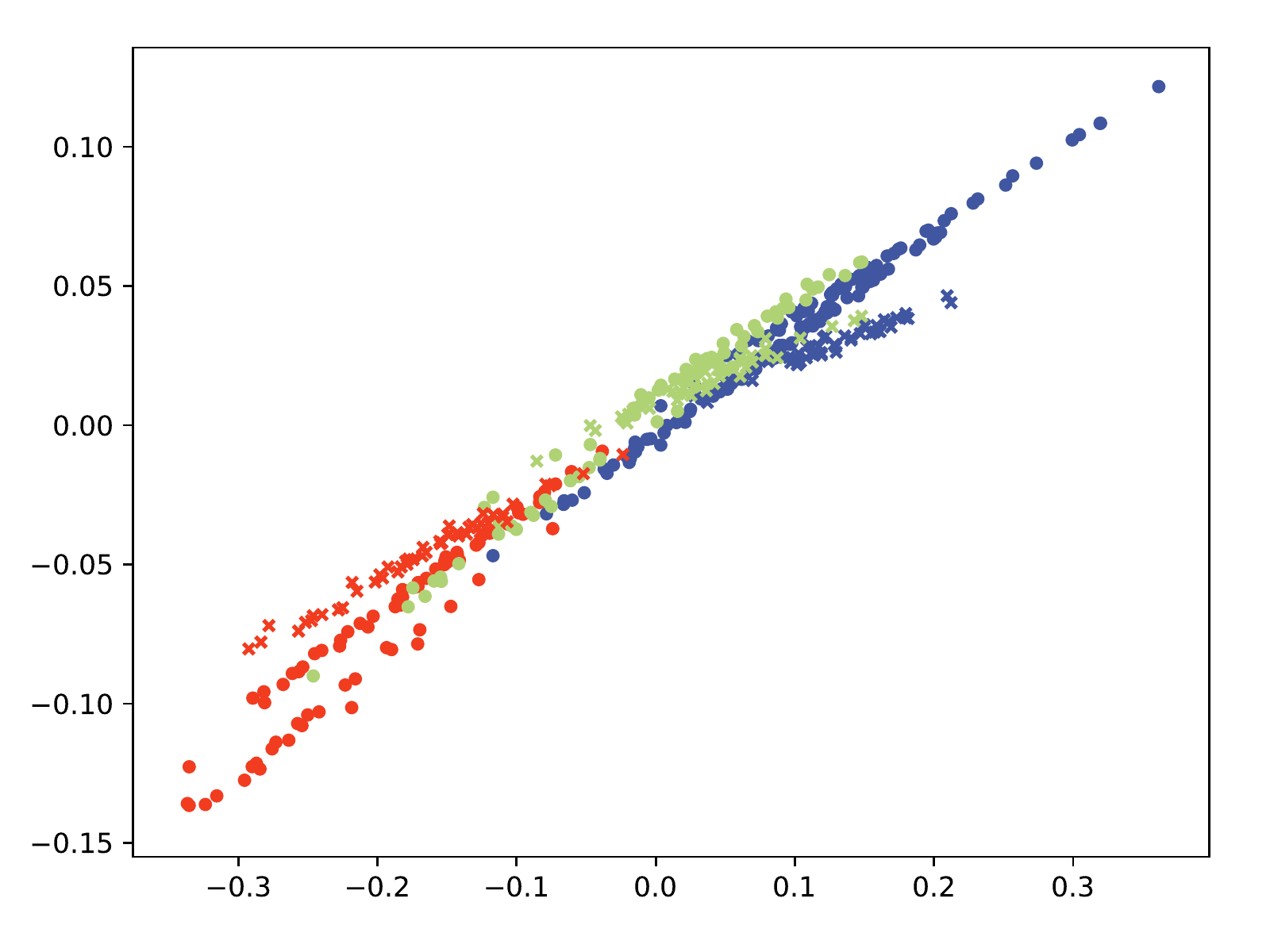}
		\caption{CIDG(77.33\%)}
		\label{Fig:4_3}
	\end{subfigure}
	\begin{subfigure}{.24\linewidth}
		\centering
		\includegraphics[width=\linewidth]{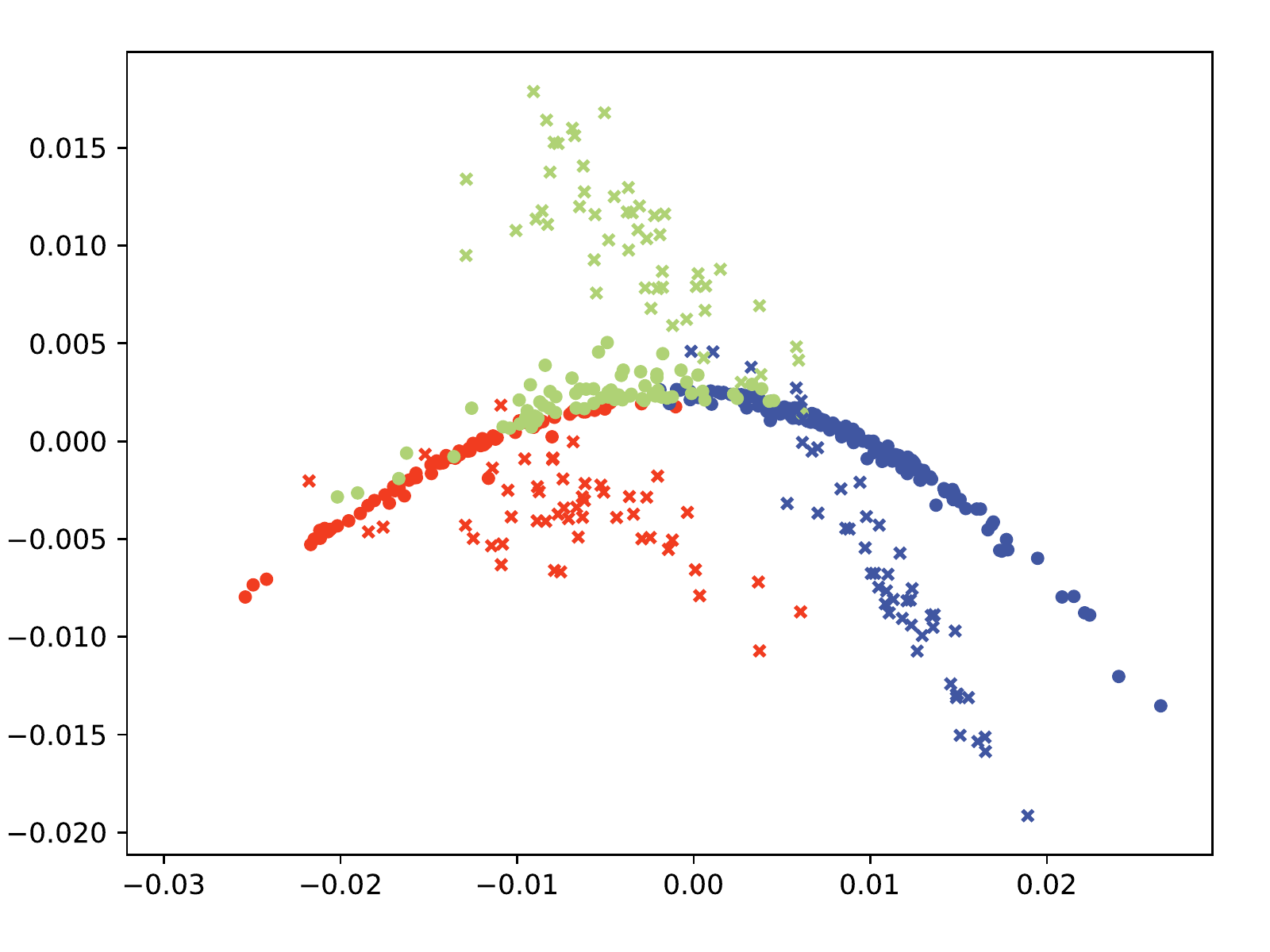}
		\caption{MDA(94.00\%)}
		\label{Fig:4_4}
	\end{subfigure} \\
	
	\begin{subfigure}{.24\linewidth}
		\centering
		\includegraphics[width=\linewidth]{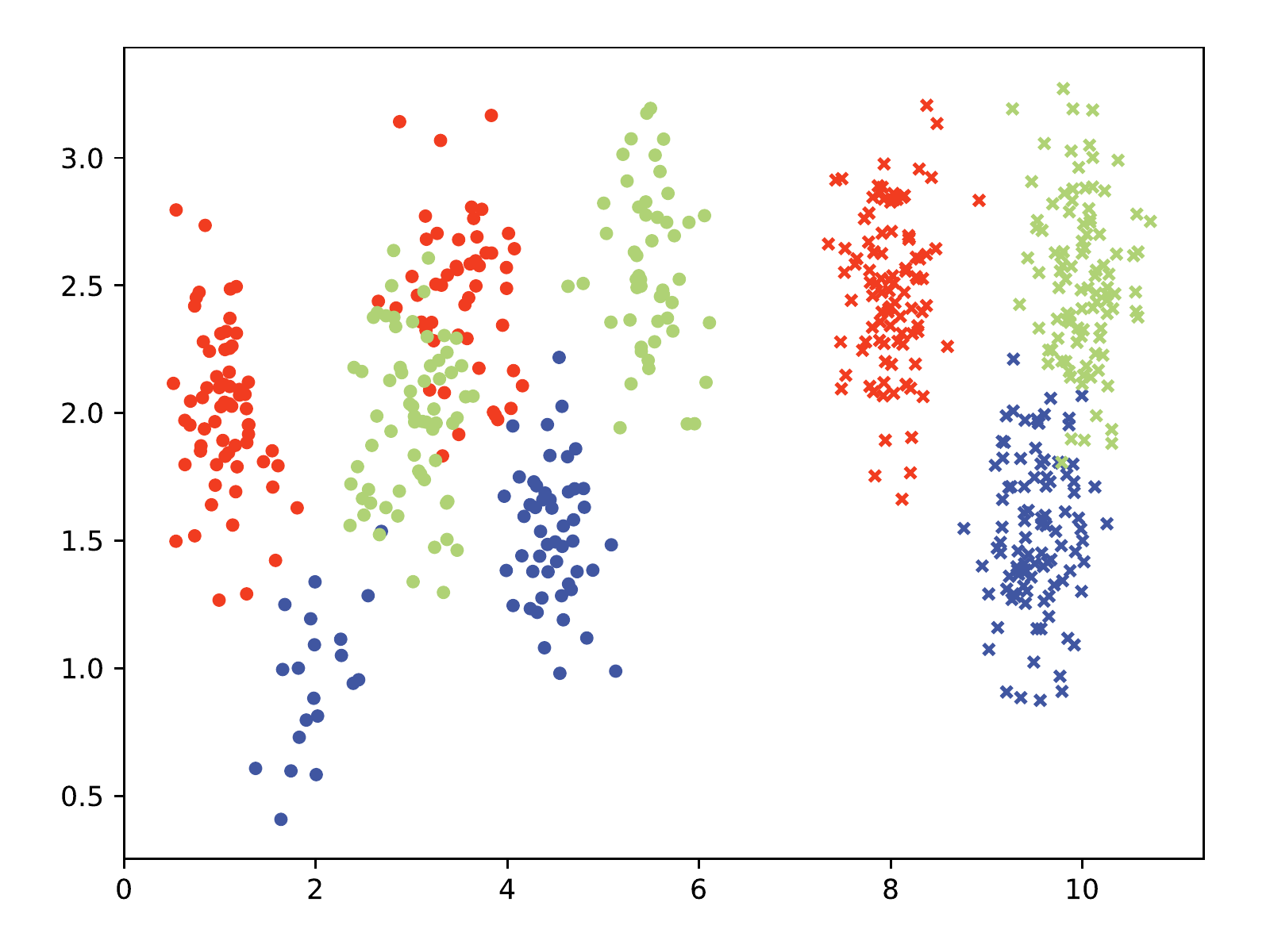}
		\caption{(2e, 2a) Raw data}
		\label{Fig:5_1}
	\end{subfigure}
	\begin{subfigure}{.24\linewidth}
		\centering
		\includegraphics[width=\linewidth]{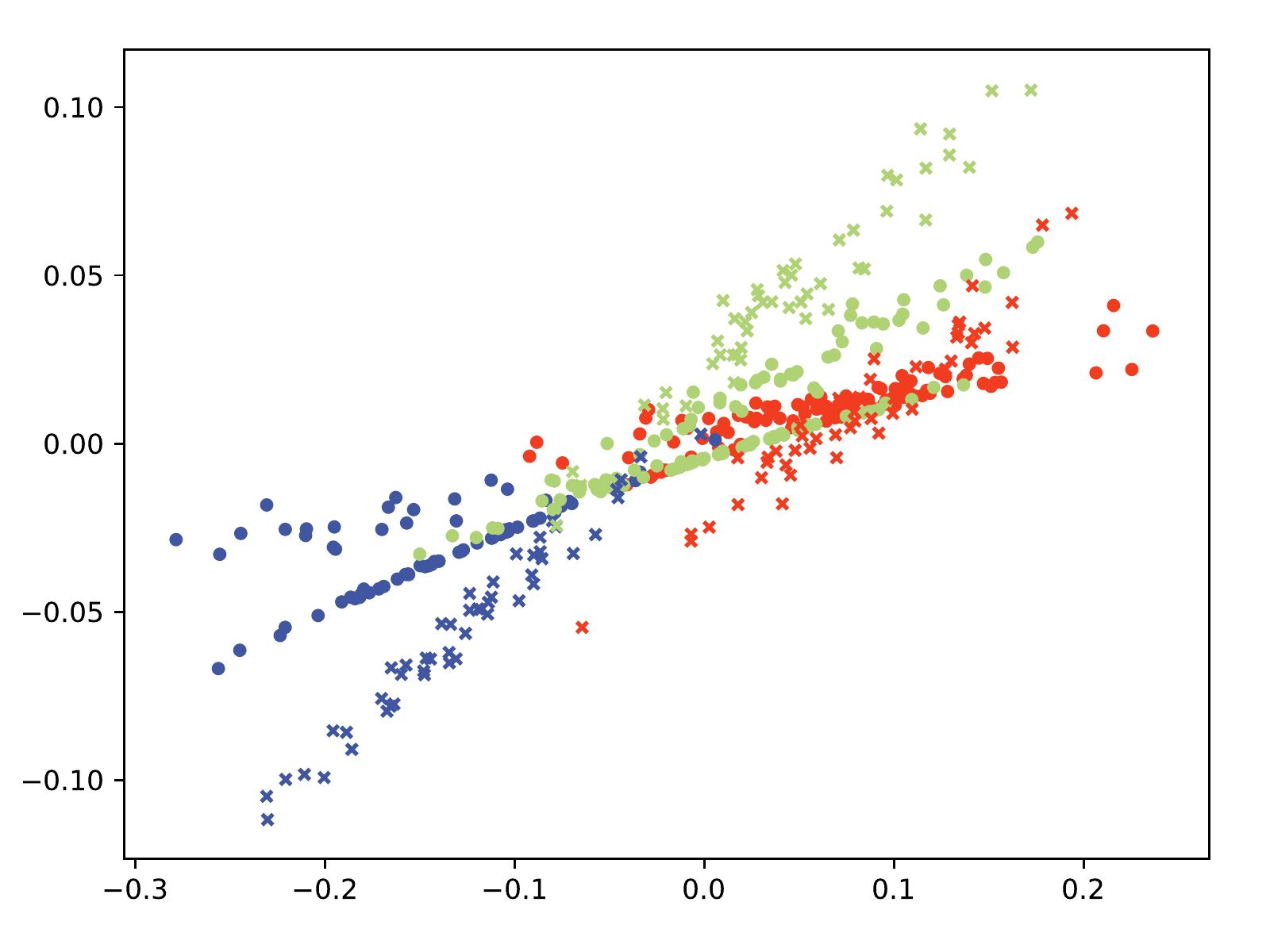}
		\caption{SCA(76.00\%)}
		\label{Fig:5_2}
	\end{subfigure}
	\begin{subfigure}{.24\linewidth}
		\centering
		\includegraphics[width=\linewidth]{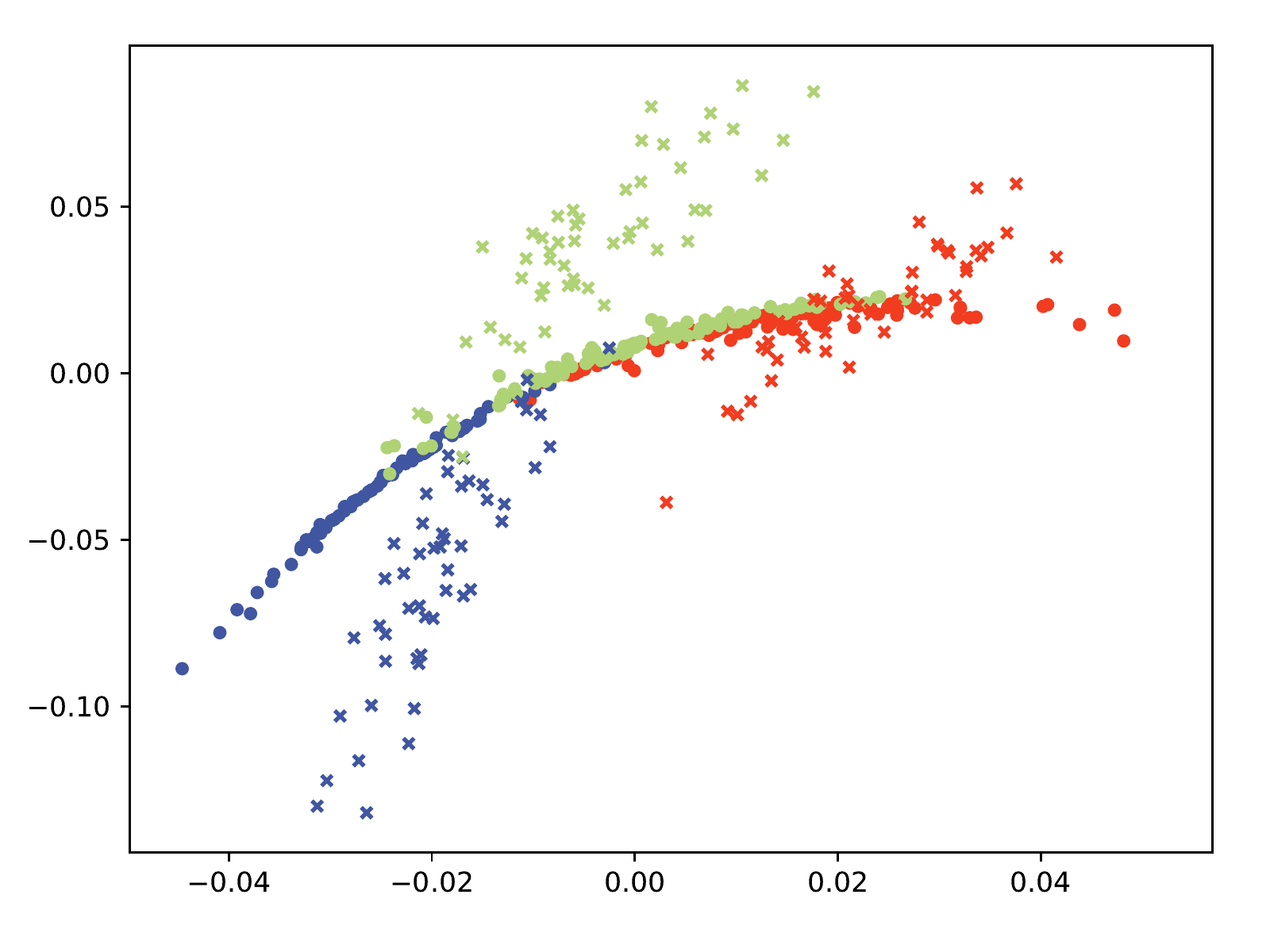}
		\caption{CIDG(86.67\%)}
		\label{Fig:5_3}
	\end{subfigure}
	\begin{subfigure}{.24\linewidth}
		\centering
		\includegraphics[width=\linewidth]{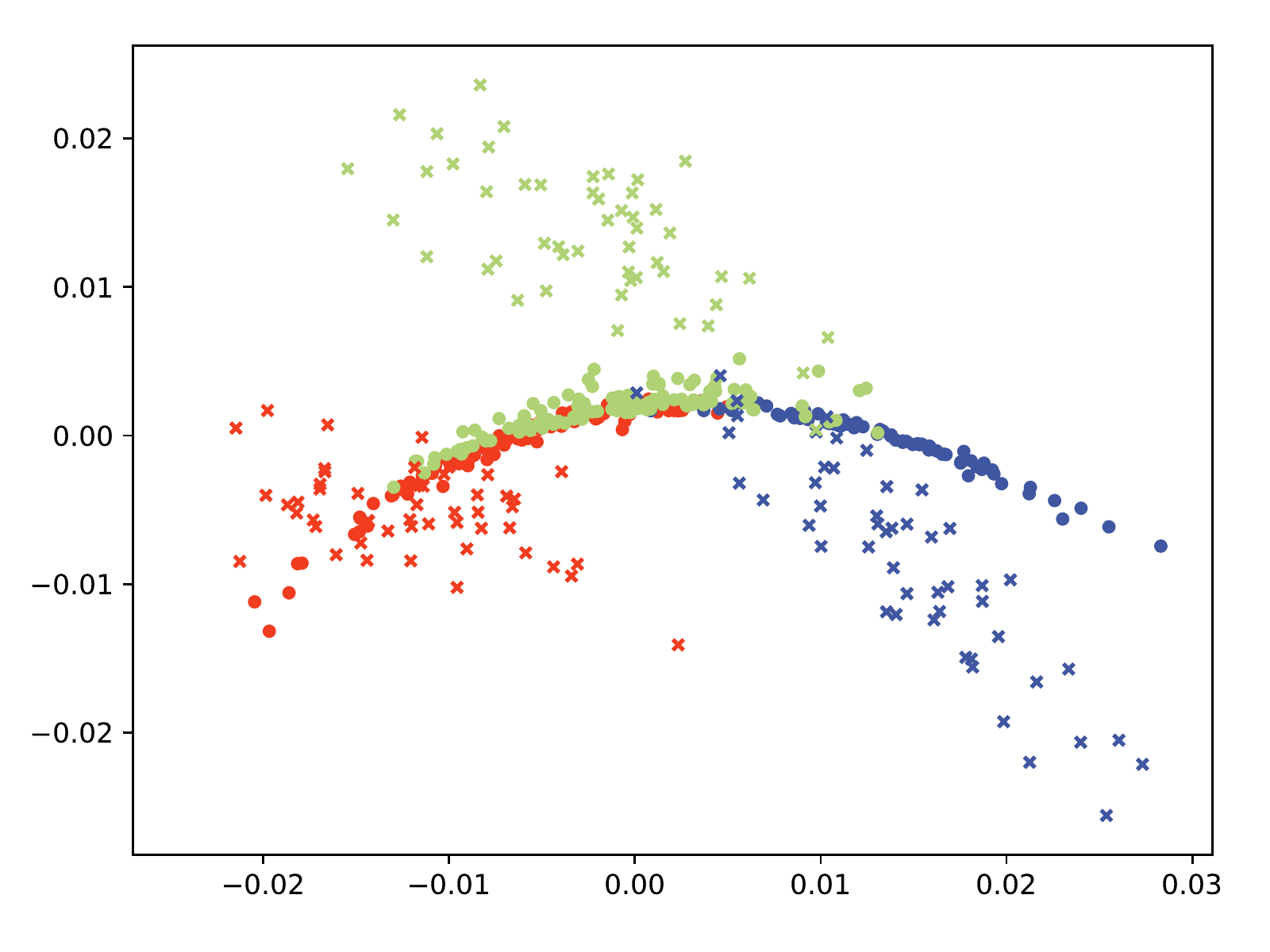}
		\caption{MDA(94.00\%)}
		\label{Fig:5_4}
	\end{subfigure}
	\caption{Visualization of transformed data in $\mathbb{R}^{q}$ of cases (2a, 2a), (2b, 2a), (2c, 2a), (2d, 2a), (2e, 2a). Each row corresponds to a case of class-prior distributions. Each column corresponds to a DG methods. The first column shows the distribution of the raw data. Different colors denote different classes. Circle marker denotes the data of source domain and cross marker denotes the data of target domain.}
	\label{fig:syn_re_1}
\end{figure}

\begin{figure}[b]
	\begin{subfigure}{.24\linewidth}
		\centering
		\includegraphics[width=\linewidth]{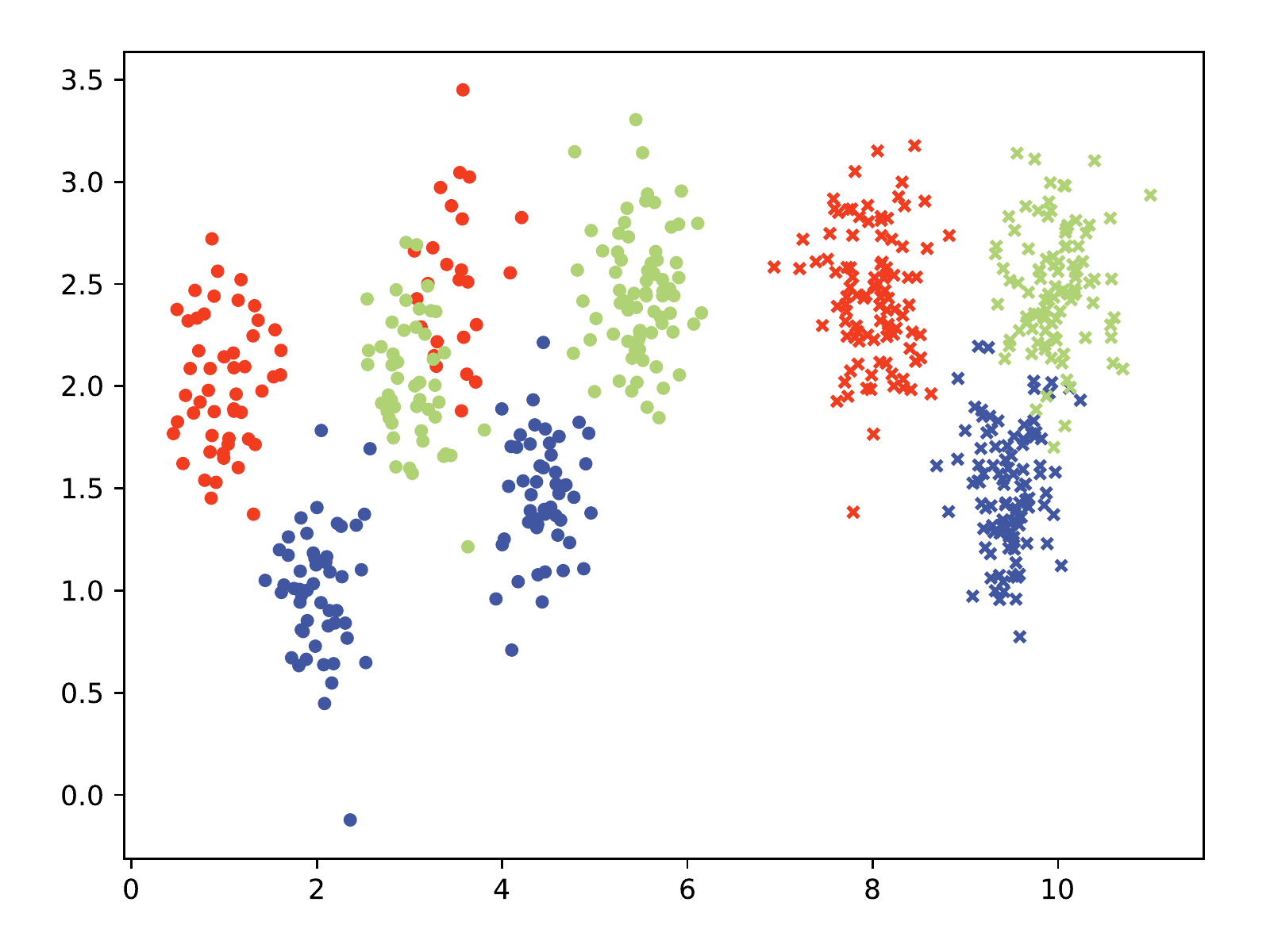}
		\caption{(2a, 2b) Raw data}
	\end{subfigure}
	\begin{subfigure}{.24\linewidth}
		\centering
		\includegraphics[width=\linewidth]{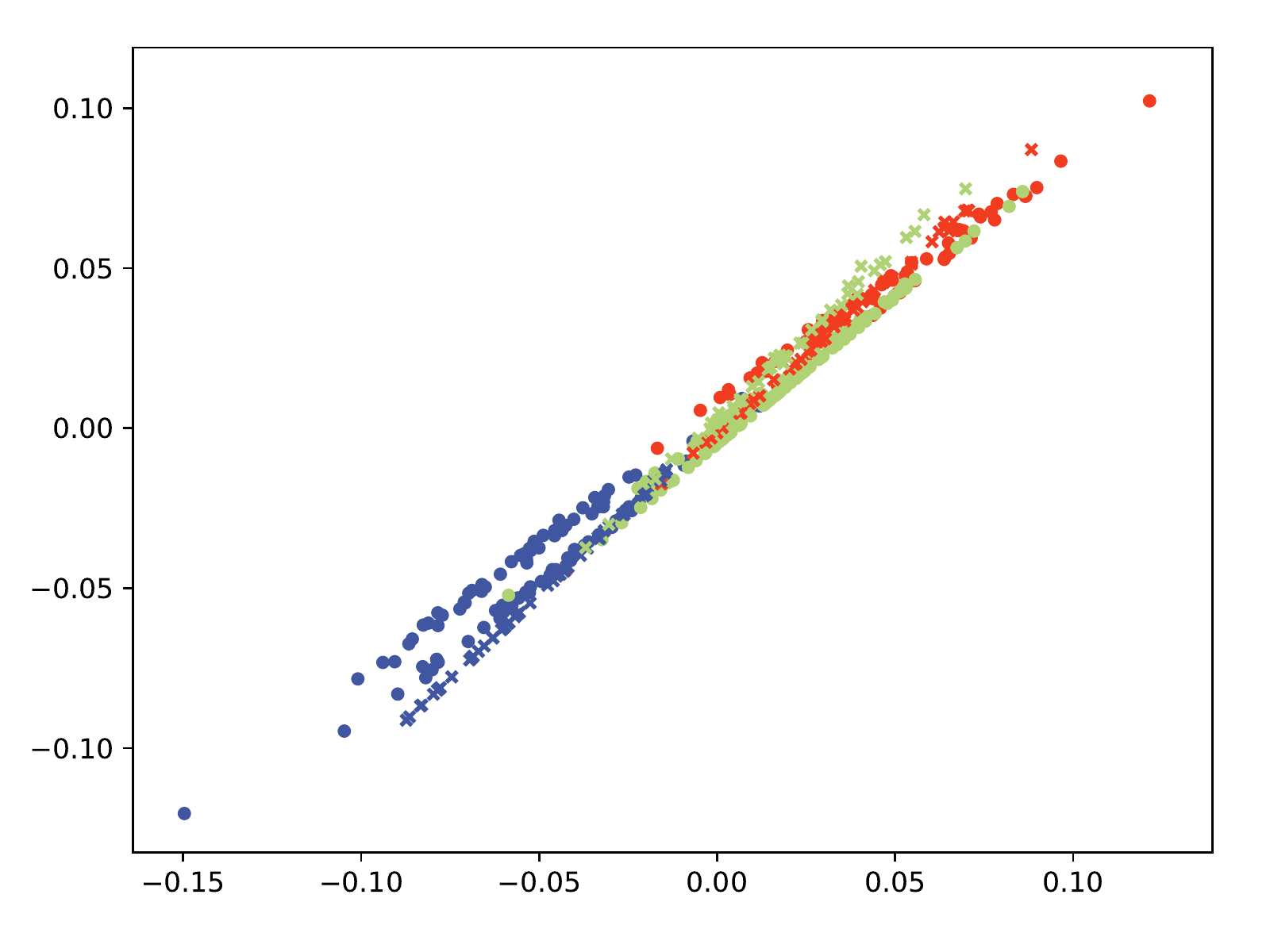}
		\caption{SCA(59.33\%)}
	\end{subfigure} 
	\begin{subfigure}{.24\linewidth}
		\centering
		\includegraphics[width=\linewidth]{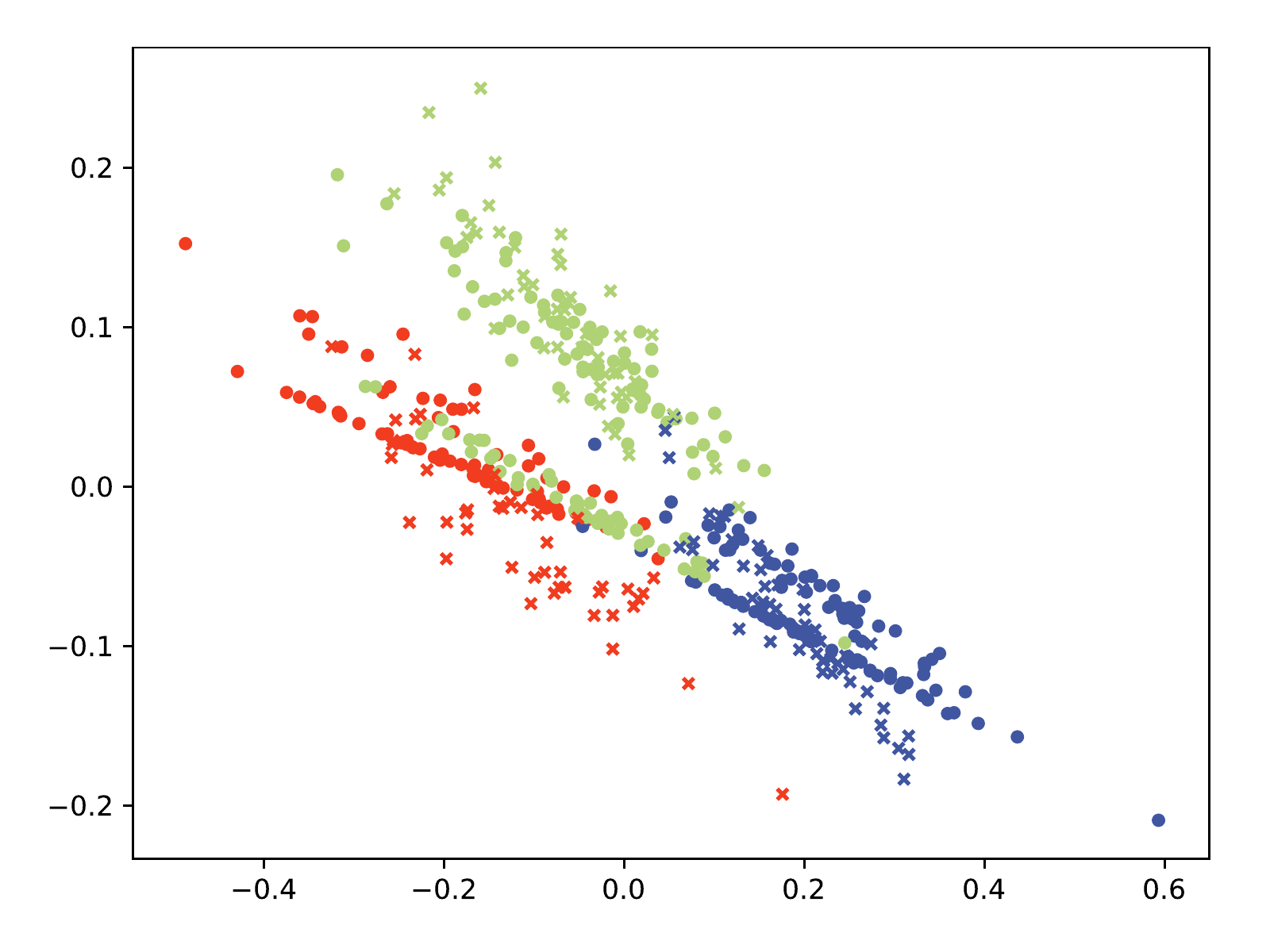}
		\caption{CIDG(83.33\%)}
	\end{subfigure}
	\begin{subfigure}{.24\linewidth}
		\centering
		\includegraphics[width=\linewidth]{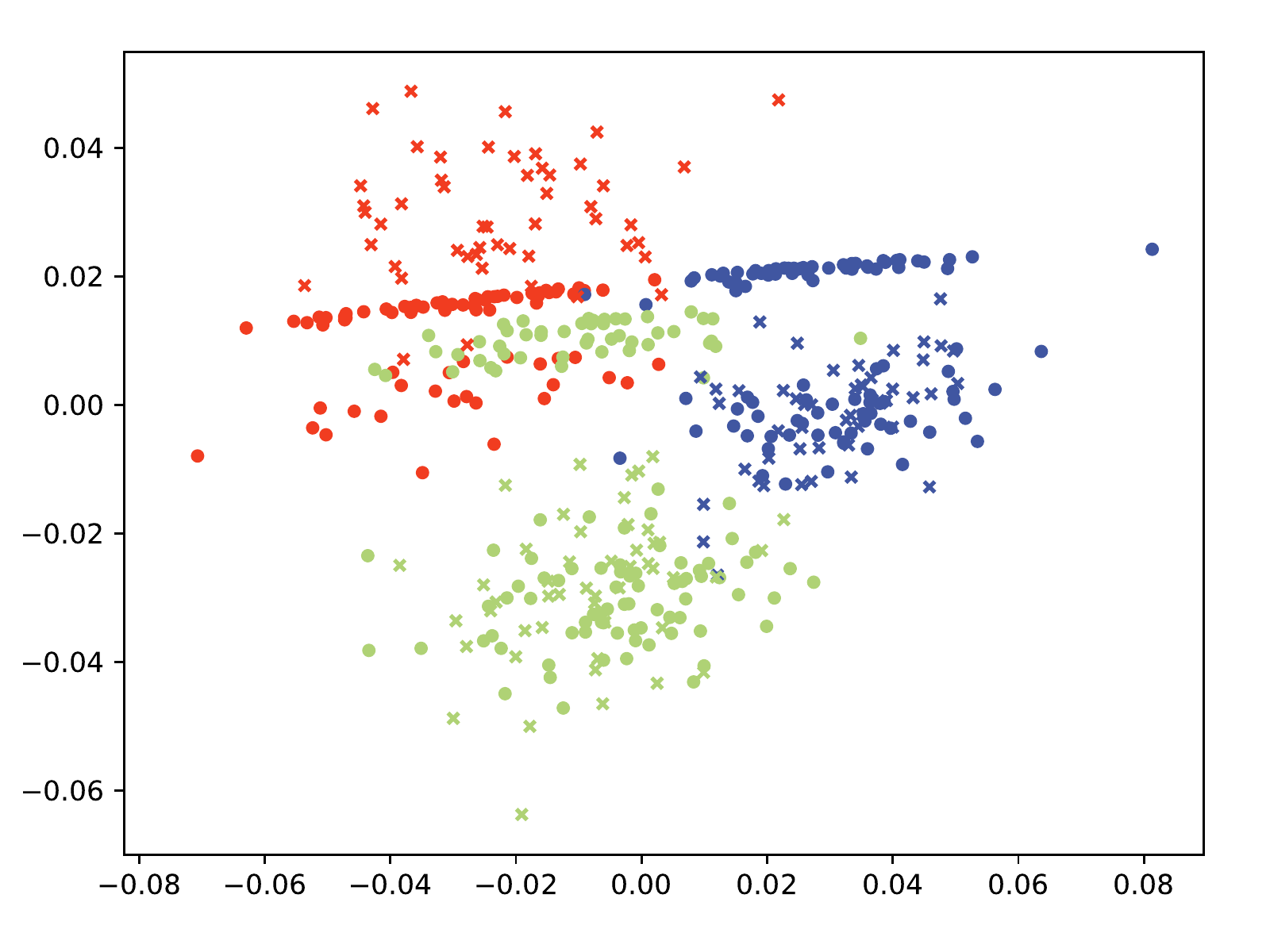}
		\caption{MDA(91.33\%)}
	\end{subfigure} \\
	
	\begin{subfigure}{.24\linewidth}
		\centering
		\includegraphics[width=\linewidth]{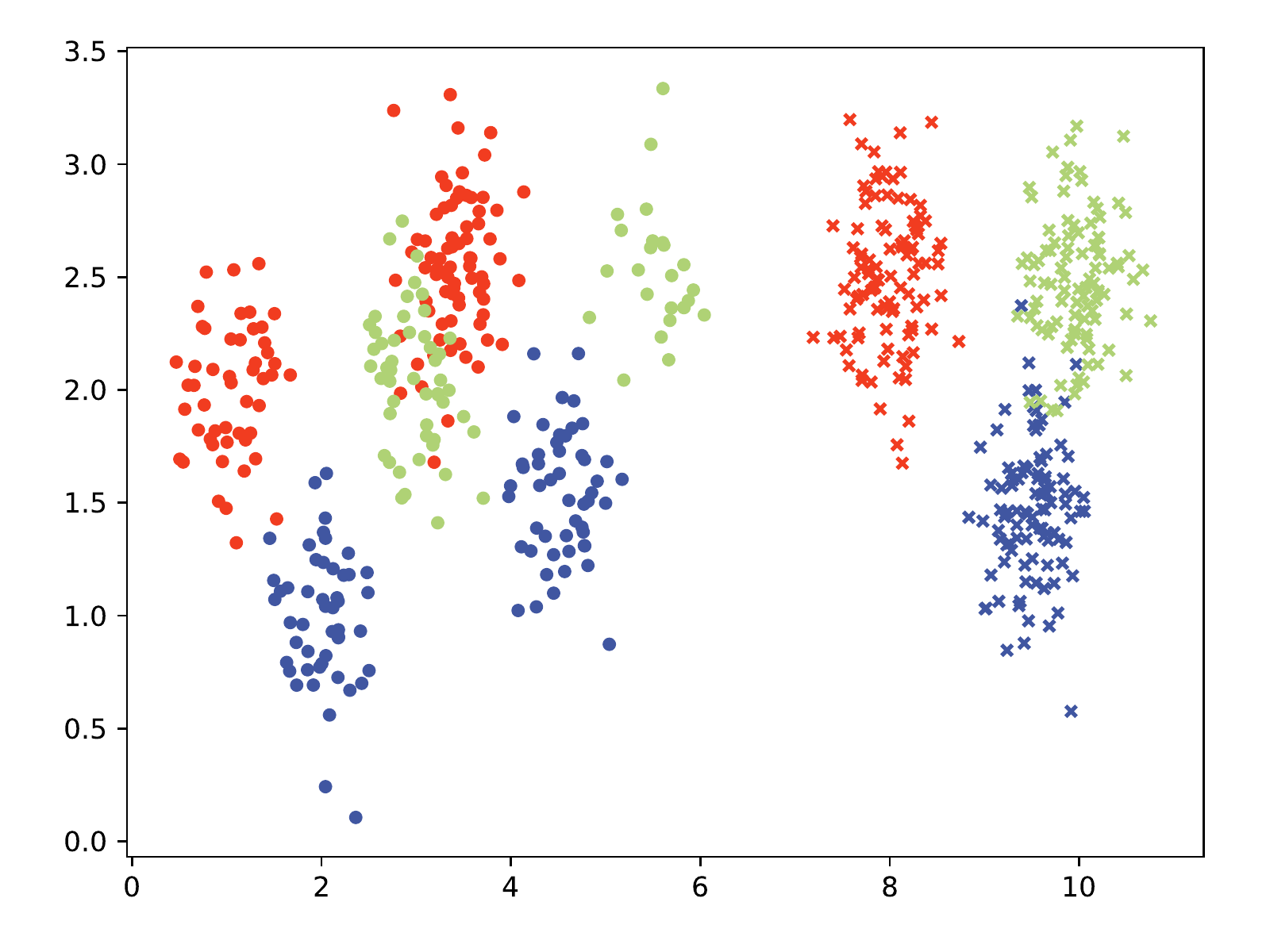}
		\caption{(2a, 2c) Raw data}
	\end{subfigure}
	\begin{subfigure}{.24\linewidth}
		\centering
		\includegraphics[width=\linewidth]{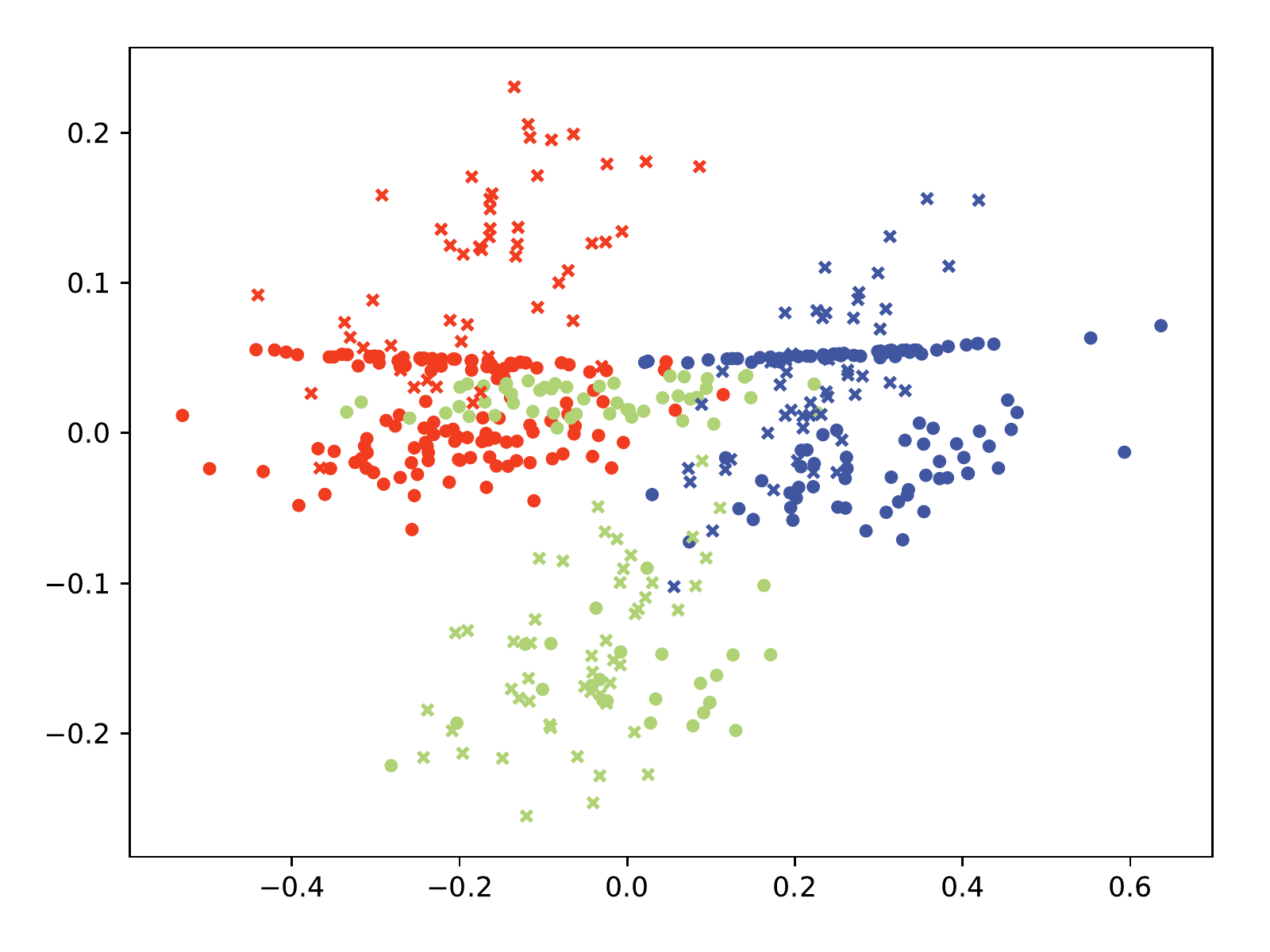}
		\caption{SCA(84.67\%)}
	\end{subfigure}
	\begin{subfigure}{.24\linewidth}
		\centering
		\includegraphics[width=\linewidth]{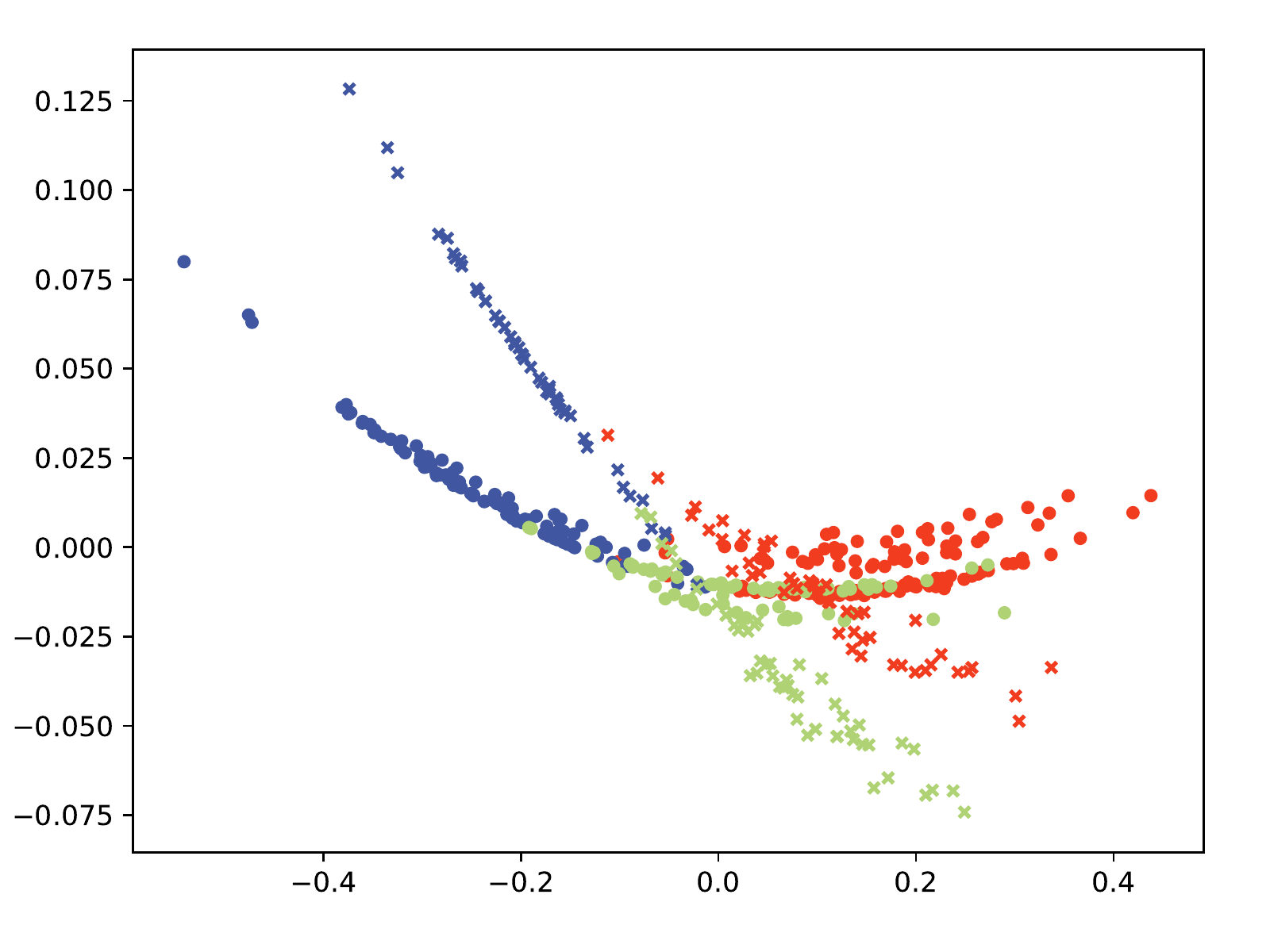}
		\caption{CIDG(92.00\%)}
	\end{subfigure}
	\begin{subfigure}{.24\linewidth}
		\centering
		\includegraphics[width=\linewidth]{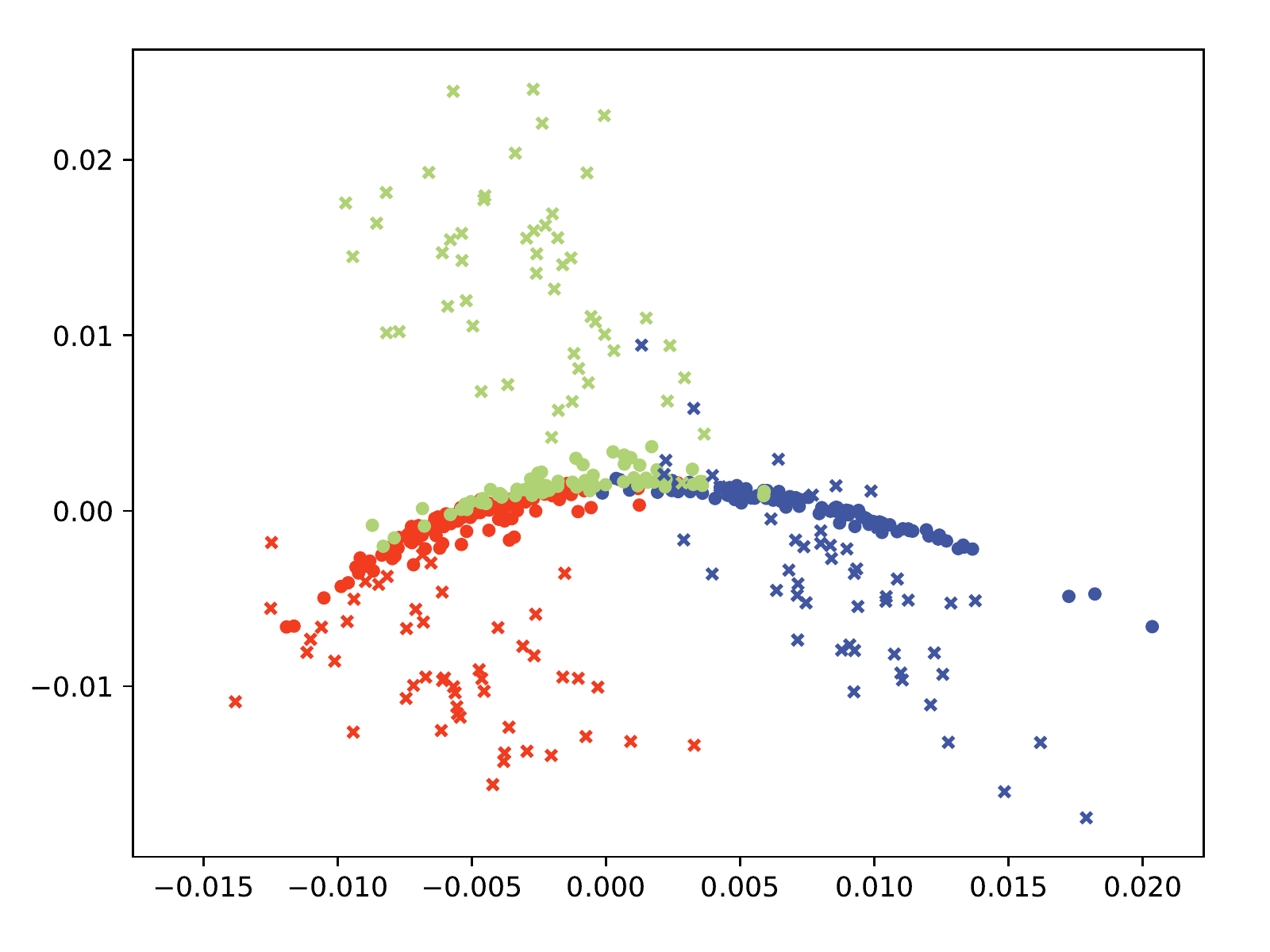}
		\caption{MDA(95.33\%)}
	\end{subfigure} \\

	\begin{subfigure}{.24\linewidth}
		\centering
		\includegraphics[width=\linewidth]{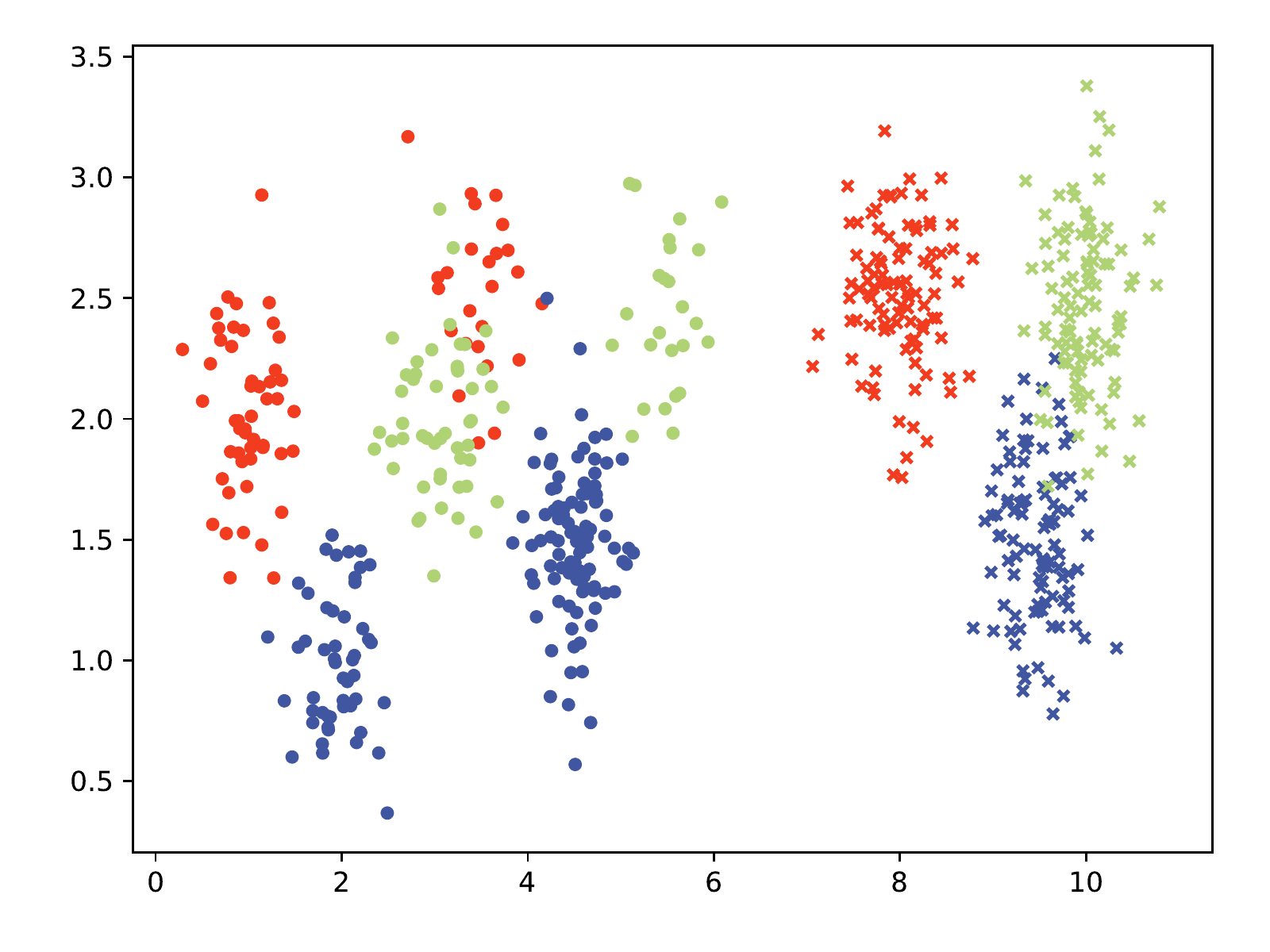}
		\caption{(2a, 2d) Raw data}
	\end{subfigure}
	\begin{subfigure}{.24\linewidth}
		\centering
		\includegraphics[width=\linewidth]{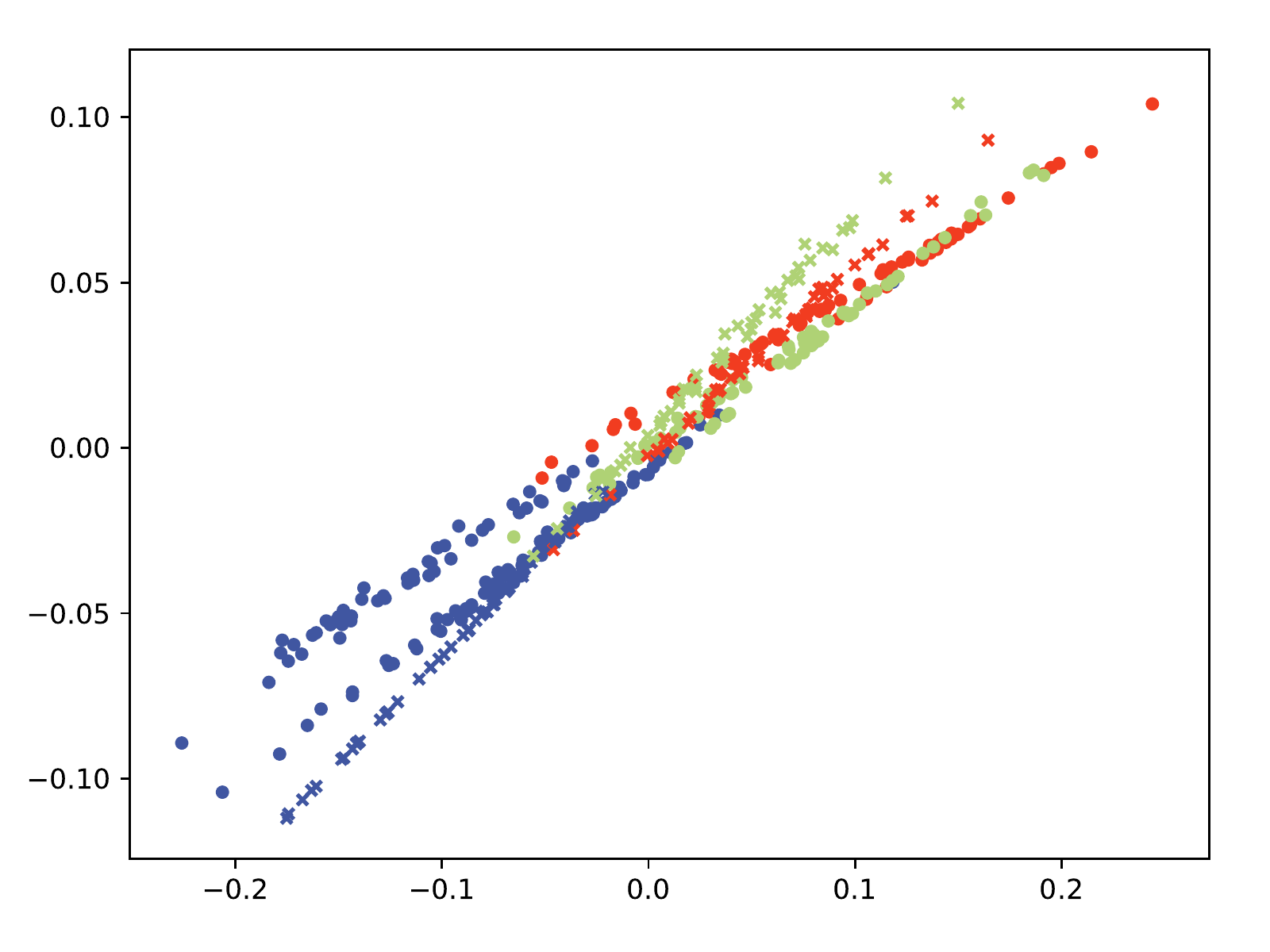}
		\caption{SCA(61.33\%)}
	\end{subfigure}
	\begin{subfigure}{.24\linewidth}
		\centering
		\includegraphics[width=\linewidth]{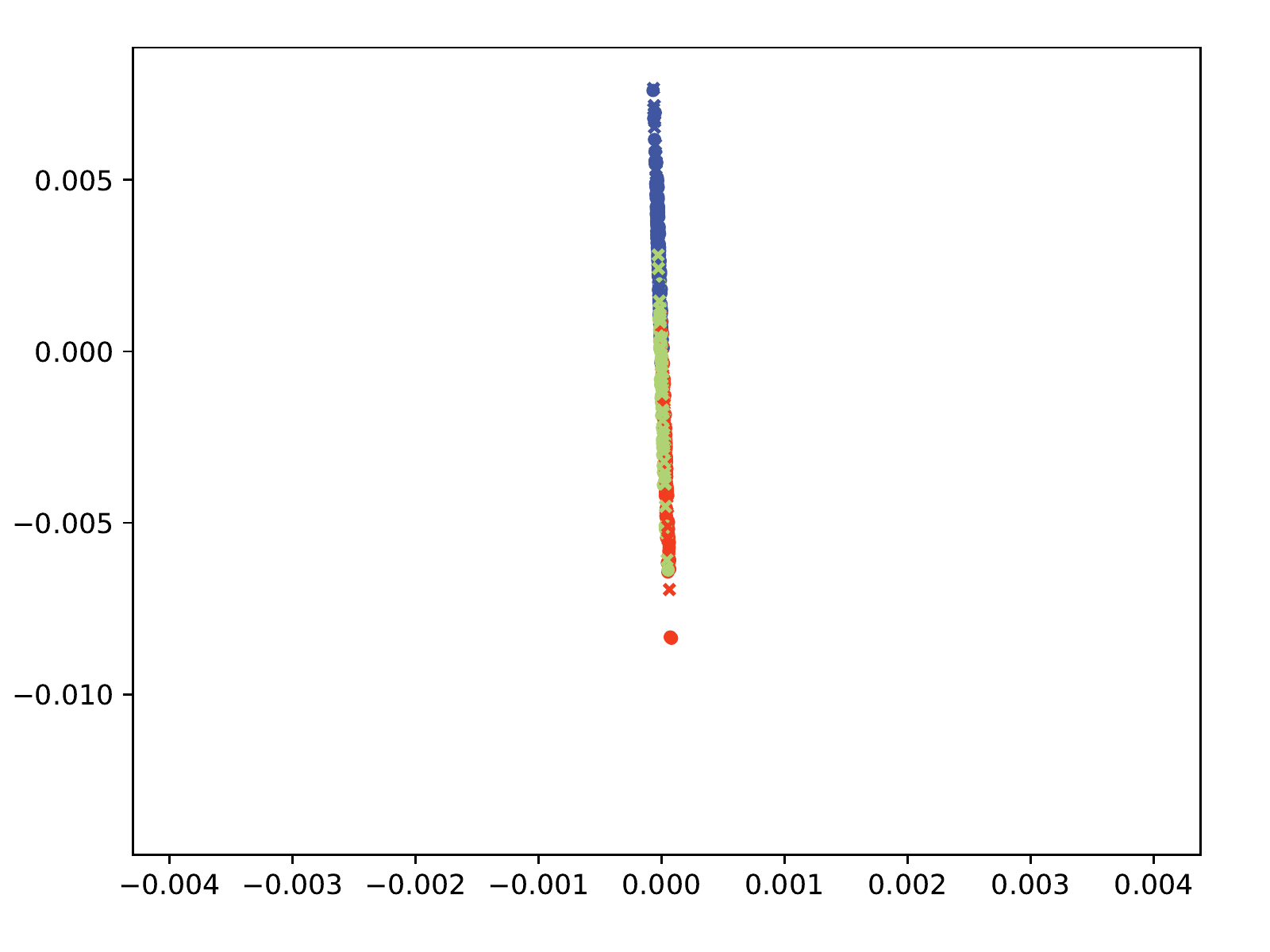}
		\caption{CIDG(82.00\%)}
	\end{subfigure}
	\begin{subfigure}{.24\linewidth}
		\centering
		\includegraphics[width=\linewidth]{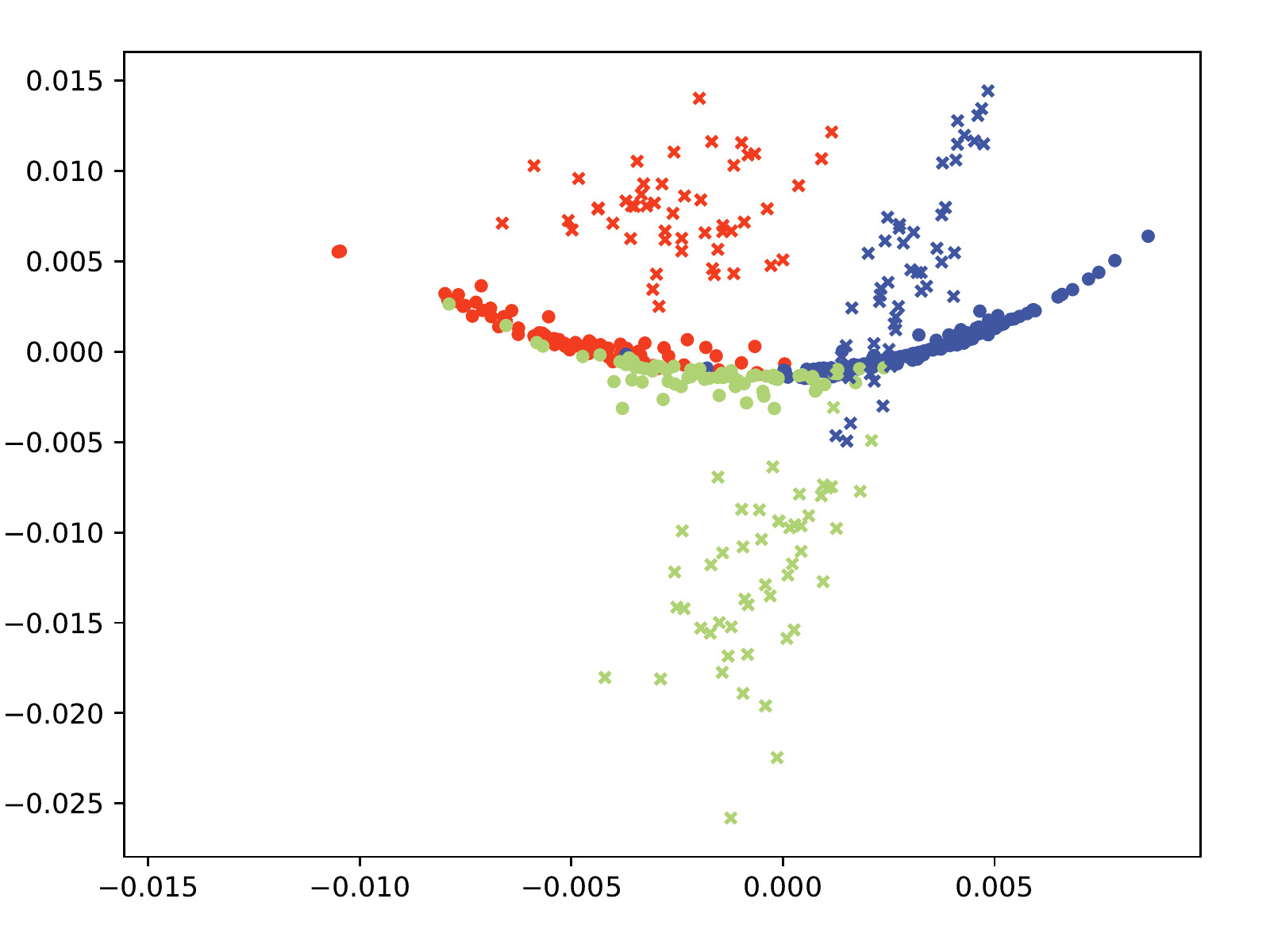}
		\caption{MDA(94.00\%)}
	\end{subfigure} \\
	
	\begin{subfigure}{.24\linewidth}
		\centering
		\includegraphics[width=\linewidth]{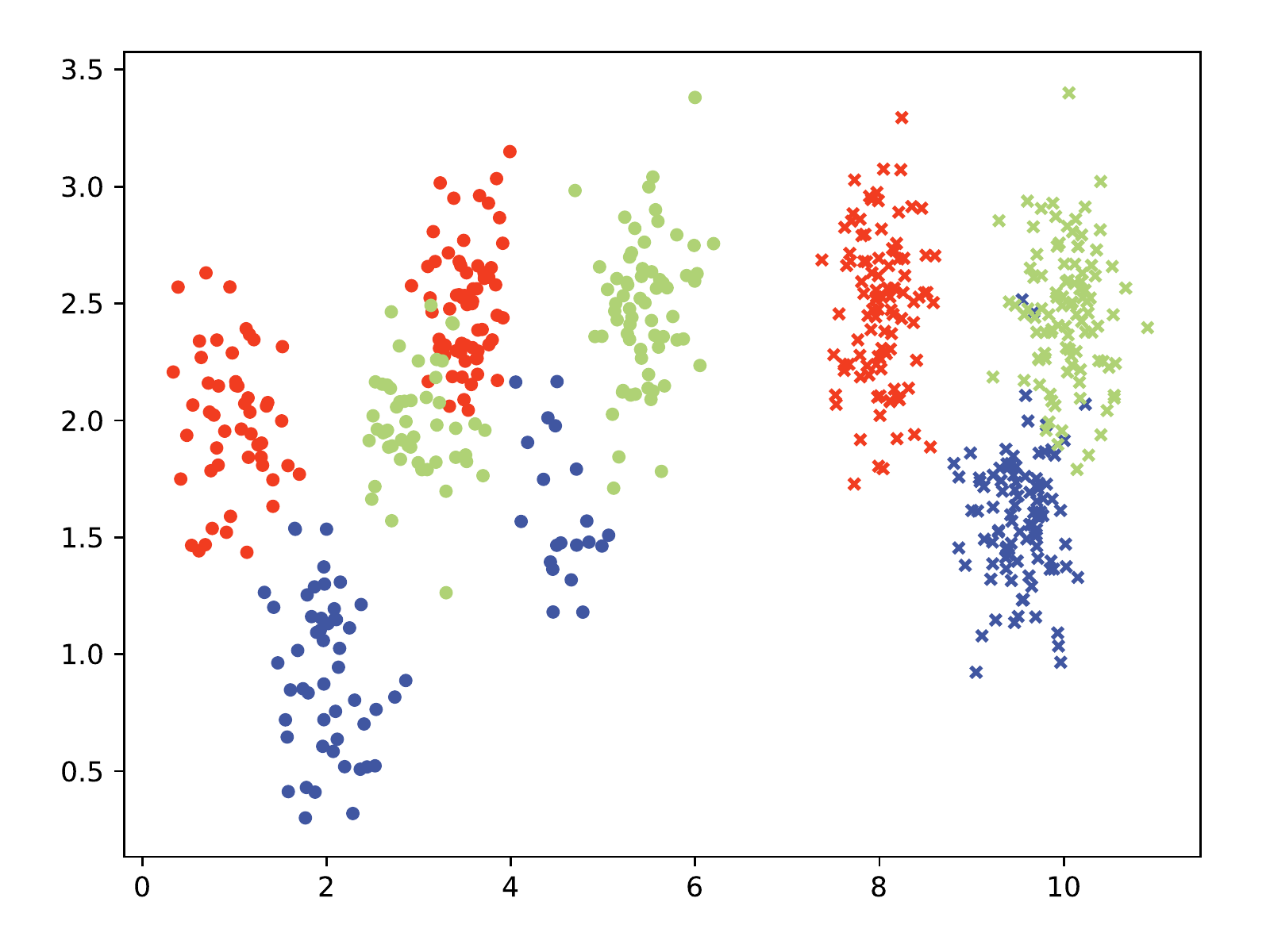}
		\caption{(2a, 2e) Raw data}
	\end{subfigure}
	\begin{subfigure}{.24\linewidth}
		\centering
		\includegraphics[width=\linewidth]{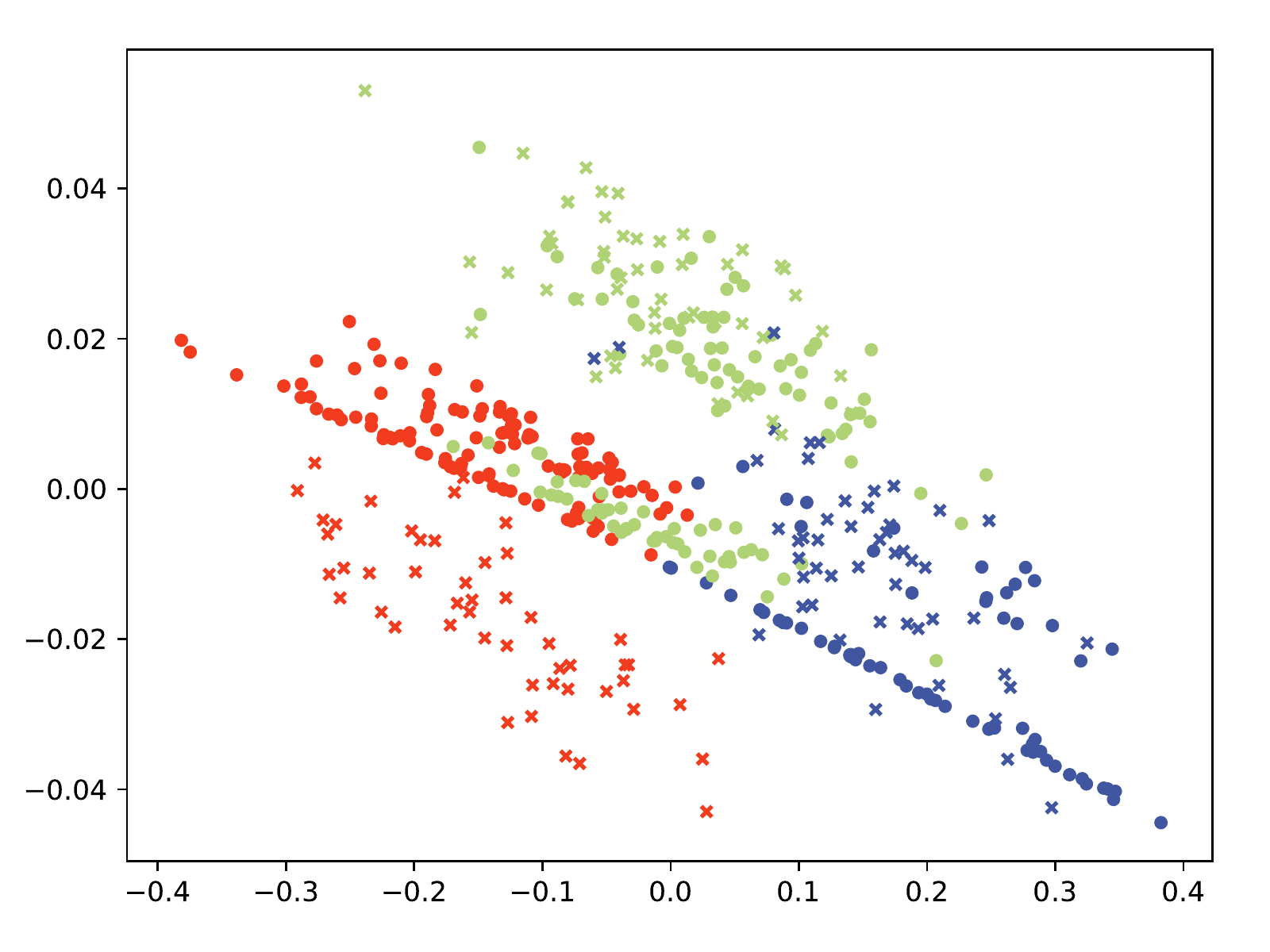}
		\caption{SCA(81.33\%)}
	\end{subfigure}
	\begin{subfigure}{.24\linewidth}
		\centering
		\includegraphics[width=\linewidth]{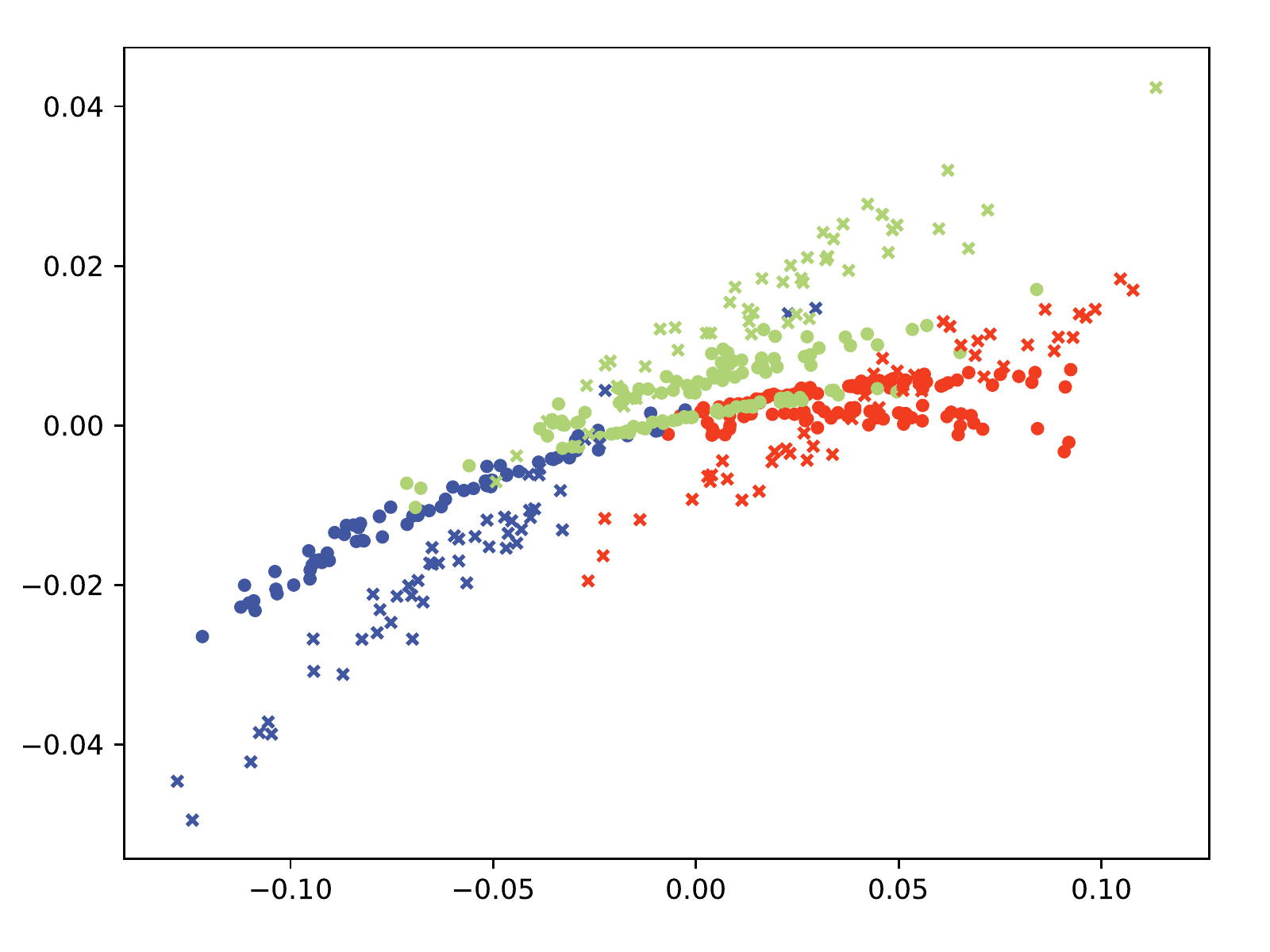}
		\caption{CIDG(86.00\%)}
	\end{subfigure}
	\begin{subfigure}{.24\linewidth}
		\centering
		\includegraphics[width=\linewidth]{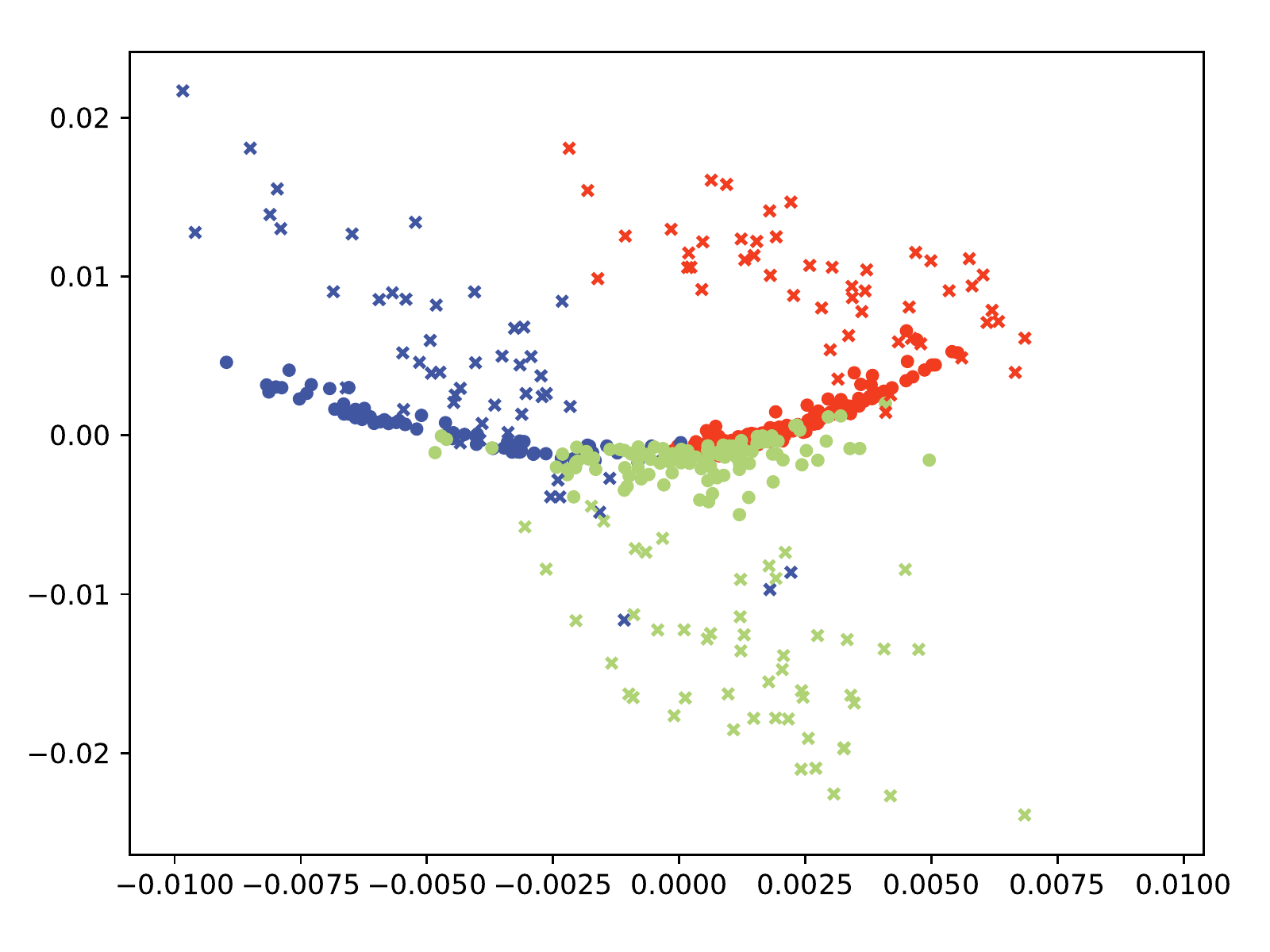}
		\caption{MDA(94.00\%)}
	\end{subfigure}

	\caption{Visualization of transformed data in $\mathbb{R}^{q}$ of cases (2a, 2b), (2a, 2c), (2a, 2d), (2a, 2e). Each row corresponds to a case of class-prior distributions. Each column corresponds to a DG methods. The first column shows the distribution of the raw data. Different colors denote different classes. Circle marker denotes the data of source domain and cross marker denotes the data of target domain.}
	\label{fig:syn_re_2}
\end{figure}

\clearpage
\section{Related Work}
Compared with domain adaptation, domain generalization is a younger line of research. \cite{blanchard2011generalizing} are the first to formalize the domain generalization of classification tasks. Motivated by automatic gating of flow cytometry data, they adopted kernel-based methods and derived the dual of a kind of cost-sensitive SVM to solve for the optimal decision function. A feature projection-based method called Domain Invariant Component Analysis (DICA; \citep{muandet2013domain}) was then proposed in 2013. DICA was the first to bring the idea of learning a shared subspace into domain generalization. It finds a transformation to a subspace in which the differences between marginal distributions $\mathbb{P}(X)$ over domains are minimized while preserving the functional relationship between $Y$ and $X$. 

Along this line, subsequent feature projection-based methods have been proposed. Scatter Component Analysis (SCA; \citep{ghifary2017scatter}) is the first unified framework for both domain adaptation and domain generalization. It combines domain scatter, kernel principal component analysis and kernel Fisher discriminant analysis into an objective and trades between them to learn the transformation. Unlike previous works, the authors of Conditional Invariant Domain Generalization (CIDG; \citep{AAAI1816595}) are the first to analyze domain generalization of classification tasks from causal perspective and thus consider more general cases where both $\mathbb{P}(Y|X)$ and $\mathbb{P}(X)$ vary across domains. They combine total scatter of class-conditional distributions, scatter of class prior-normalized marginal distributions, and kernel Fisher discriminant analysis to achieve the goal of domain generalization.

Besides the aforementioned methods in general, domain generalization problem also attracted extensive attention of computer vision community. \cite{khosla2012undoing} proposed a max-margin framework (Undo-Bias) in which each domain is assumed to be controlled by the sum of the visual world and a bias. A modified SVM-based method is adopted for solving the weights and biases in the model. Unbiased Metric Learning (UML; \citep{fang2013unbiased}), which is based on a learning-to-rank framework, first learns a set of distance metrics and then validate to select the one with best generalization ability. \cite{xu2014exploiting} adopted exemplar-SVM and introduced a nuclear norm based regularizer into the objective to learn a set of more robust examplar-SVMs for domain generalization purpose. \citet{ghifary2015domain} introduced Multi-task Autoencoder (MTAE), a feature learning algorithm that uses a multi-task strategy to learn unbiased object features, where the task is the data reconstruction. More recently, domain generalization methods based on deep neural networks \citep{Motiian_2017_ICCV, 8237853, lihl2018domain, Li_2018_ECCV} were proposed to cope with the problem induced by distribution shift.

\end{document}